\documentclass{article}
\usepackage[utf8]{inputenc}
\usepackage[margin=0.93in]{geometry}
\usepackage{mjag4}
\usepackage{url}
\usepackage{pdfpages}
\usepackage{amsfonts}       % blackboard math symbols
\usepackage[T1]{fontenc}    % use 8-bit T1 fonts
\title{Understanding Sparse JL for Feature Hashing\thanks{A version of this work appeared at NeurIPS 2019. I would like to thank Prof. Jelani Nelson for advising this project.}}

\author{%
  Meena Jagadeesan\\
  Harvard University\\
  Cambridge, MA 02138 \\
  \texttt{mjagadeesan@college.harvard.edu} \\
}

\date{}

\begin{document}

\maketitle

\begin{abstract}
Feature hashing and other random projection schemes are commonly used to reduce the dimensionality of feature vectors. The goal is to efficiently project a high-dimensional feature vector living in $\mathbb{R}^n$ into a much lower-dimensional space $\mathbb{R}^m$, while approximately preserving Euclidean norm. These schemes can be constructed using sparse random projections, for example using a sparse Johnson-Lindenstrauss (JL) transform. A line of work introduced by Weinberger et. al (ICML '09) analyzes the accuracy of sparse JL with sparsity 1 on feature vectors with small $\ell_\infty$-to-$\ell_2$ norm ratio. Recently, Freksen, Kamma, and Larsen (NeurIPS '18) closed this line of work by proving a tight tradeoff between $\ell_\infty$-to-$\ell_2$ norm ratio and accuracy for sparse JL with sparsity $1$. 

In this paper, we demonstrate the benefits of using sparsity $s$ greater than $1$ in sparse JL on feature vectors. Our main result is a tight tradeoff between $\ell_\infty$-to-$\ell_2$ norm ratio and accuracy for a general sparsity $s$, that significantly generalizes the result of Freksen et. al. Our result theoretically demonstrates that sparse JL with $s > 1$ can have significantly better norm-preservation properties on feature vectors than sparse JL with $s = 1$; we also empirically demonstrate this finding.

\end{abstract}

\section{Introduction}
Feature hashing and other random projection schemes are influential in helping manage large data \cite{Dal13}. The goal is to \textit{reduce the dimensionality} of feature vectors: more specifically, to project high-dimensional feature vectors living in $\mathbb{R}^n$ into a lower dimensional space $\mathbb{R}^m$ (where $m  \ll  n$), while approximately preserving Euclidean distances (i.e. $\ell_2$ distances) with high probability. This dimensionality reduction enables a classifier to process vectors in $\mathbb{R}^m$, instead of vectors in $\mathbb{R}^n$. In this context, feature hashing was first introduced by Weinberger et. al \cite{Weinberger} for document-based classification tasks such as email spam filtering. For such tasks, feature hashing yields a lower dimensional embedding of a high-dimensional feature vector derived from a bag-of-words model. Since then, feature hashing has become a mainstream approach \cite{Sut16}, applied to numerous domains including ranking text documents \cite{DocClass}, compressing neural networks \cite{NeuralNet}, and protein sequence classification \cite{Protein}. 

\subsection*{Random Projections}
Dimensionality reduction schemes for feature vectors fit nicely into the random projection literature. In fact, the feature hashing scheme proposed by Weinberger et al. \cite{Weinberger} boils down to uniformly drawing a random $m \times n$ matrix where each column contains \textit{one} nonzero entry, equal to $-1$ or $1$. 

The $\ell_2$-norm-preserving objective can be expressed mathematically as follows: for error $\epsilon > 0$ and failure probability $\delta$, the goal is to construct a probability distribution $\mathcal{A}$ over $m \times n$ real matrices that satisfies the following condition for vectors $x \in \mathbb{R}^n$:
\begin{equation}
\label{JLcondition}
\mathbb{P}_{A \in \mathcal{A}} [(1 - \epsilon)\norm{x}_2 \le \norm{Ax}_2 \le (1 + \epsilon)\norm{x}_2] > 1-\delta.
\end{equation}
The result underlying the random projection literature is the Johnson-Lindenstrauss lemma, which gives an upper bound on the dimension $m$ achievable by a probability distribution $\mathcal{A}$ satisfying $\eqref{JLcondition}$:
\begin{lemma}[Johnson-Lindenstrauss \cite{JL}]
\label{RegularJL}
For any $n \in \mathbb{N}$ and $\epsilon, \delta \in (0,1)$, there exists a probability distribution $\mathcal{A}$ over $m \times n$ matrices, with $m = \Theta(\epsilon^{-2} \ln(1/\delta))$, that satisfies \eqref{JLcondition}.  
\end{lemma}
\noindent The optimality of the dimension $m$ achieved by Lemma~$\ref{RegularJL}$ has been proven \cite{OptimalityJL1, OptimalityJL2}. 

To speed up projection time, it is useful to consider probability distributions over sparse matrices (i.e. matrices with a small number of nonzero entries per column). More specifically, for matrices with $s$ nonzero entries per column, the projection time for a vector $x$ goes down from $O(m\norm{x}_0)$ to $O(s\norm{x}_0)$, where $\norm{x}_0$ is the number of nonzero entries of $x$. In this context, Kane and Nelson \cite{KN12} constructed sparse JL distributions (which we define formally in \Cref{sec:sparseJLDist}), improving upon previous work \cite{Achlioptas, Li, DKS10}. Roughly speaking, a sparse JL distribution, as constructed in \cite{KN12}, boils down to drawing a random $m \times n$ matrix where each column contains exactly $s$ nonzero entries, each equal to $-1 / \sqrt{s}$ or $1 / \sqrt{s}$. Kane and Nelson show that sparse JL distributions achieve the same (optimal) dimension as Lemma~$\ref{RegularJL}$, while also satisfying a sparsity property.
\begin{theorem}[Sparse JL \cite{KN12}]
\label{SparseJL}
For any $n \in \mathbb{N}$ and $\epsilon, \delta \in (0,1)$, a sparse JL distribution $\mathcal{A}_{s,m,n}$ (defined formally in \Cref{sec:sparseJLDist}) over $m \times n$ matrices, with dimension $m = \Theta(\epsilon^{-2} \ln(1/\delta))$ and sparsity $s = \Theta(\epsilon^{-1} \ln(1/\delta))$, satisfies \eqref{JLcondition}.  
\end{theorem} 

Sparse JL distributions are state-of-the-art sparse random projections, and achieve a sparsity that is nearly optimal when the dimension $m$ is $\Theta(\epsilon^{-2} \ln(1/\delta))$.\footnote{Nelson and Nguyen \cite{NNLower} showed that \textit{any} distribution satisfying \eqref{JLcondition} requires sparsity $\Omega(\epsilon^{-1} \ln(1/\delta)/\ln(1/\epsilon))$ when the dimension $m$ is $\Theta(\epsilon^{-2} \ln(1/\delta))$. Kane and Nelson \cite{KN12} also showed that the analysis of sparse JL distributions in Theorem \ref{SparseJL} is tight at $m = \Theta(\epsilon^{-2} \ln(1/\delta))$. } However, in practice, it can be necessary to utilize a lower sparsity $s$, since the projection time is linear in $s$. Resolving this issue, Cohen \cite{Cohen} extended the upper bound in Theorem \ref{SparseJL} to show that sparse JL distributions can achieve a lower sparsity with an appropriate gain in dimension. He proved the following dimension-sparsity tradeoffs:
\begin{theorem}[Dimension-Sparsity Tradeoffs \cite{Cohen}]
\label{DimSparsity}
For any $n \in \mathbb{N}$ and $\epsilon, \delta \in (0,1)$, a uniform sparse JL distribution $\mathcal{A}_{s,m,n}$ (defined formally in \Cref{sec:sparseJLDist}), with $s \le \Theta(\epsilon^{-1} \ln(1/\delta))$ and 
\[m \ge \min\left(2 \epsilon^{-2}/\delta, \epsilon^{-2} \ln(1/\delta) e^{\Theta\left(\epsilon^{-1} \ln(1/\delta)/s\right)}\right),\] satisfies $\eqref{JLcondition}$.
\end{theorem} 

\subsection*{Connection to Feature Hashing}
Sparse JL distributions have particularly close ties to feature hashing. In particular, the feature hashing scheme proposed by Weinberger et al. \cite{Weinberger} can be viewed as a special case of sparse JL, namely with $s = 1$. Interestingly, in practice, feature hashing can do much better than theoretical results, such as Theorem~$\ref{SparseJL}$ and Theorem~$\ref{DimSparsity}$, would indicate \cite{FKL}. An explanation for this phenomenon is that the highest error terms in sparse JL stem from vectors with mass concentrated on a very small number of entries, while in practice, the mass on feature vectors may be spread out between many coordinates. This motivates studying the tradeoff space for vectors with low $\ell_{\infty}$-to-$\ell_2$ ratio.

More formally, take $S_v$ to be $\left\{x \in \mathbb{R}^n \mid \frac{\norm{x}_{\infty}}{\norm{x}_2} \le v \right\}$, so that $S_1 = \mathbb{R}^n$ and $S_v \subsetneq S_w$ for $0 \le v < w \le 1$. Let $v(m, \epsilon, \delta, s)$ be the supremum over all $0 \le v \le 1$ such that a sparse JL distribution with sparsity $s$ and dimension $m$ satisfies $\eqref{JLcondition}$ for each $x \in S_v$. (That is, $v(m, \epsilon, \delta, s)$ is the maximum $v \in [0,1]$ such that for every $x \in \mathbb{R}^n$, if $\norm{x}_{\infty} \le v \norm{x}_2$ then \eqref{JLcondition} holds.\footnote{Technically, the quantity $v(m, \epsilon, \delta, s)$, as defined here, also depends on $n$. In particular, every vector $x \in \mathbb{R}^n$ satisfies $\norm{x}_{\infty} \ge \norm{x}_2 / \sqrt{n}$, so $\ell_{\infty}$-to-$\ell_2$ norm ratios below $1 / \sqrt{n}$ are not possible in $\mathbb{R}^n$. To avoid this dependence on $n$ and thus make the bounds cleaner, the quantity $v(m, \epsilon, \delta, s)$ is actually defined to be the infimum over all $n \in \mathbb{N}$ of the supremum over all $0 \le v \le 1$ such that a sparse JL distribution with sparsity $s$ and dimension $m$ satisfies $\eqref{JLcondition}$ for each $x \in S_v$. (That is, $v(m, \epsilon, \delta, s)$ is the maximum $v \in [0,1]$ such that for every $n \in \mathbb{N}$ and every $x \in \mathbb{R}^n$, if $\norm{x}_{\infty} \le v \norm{x}_2$ then \eqref{JLcondition} holds.)}) For $s = 1$, a line of work \cite{Weinberger, DKS10, Derandomized, DKT17, KN12} improved bounds on $v(m, \epsilon, \delta, 1)$, and was recently closed by Freksen et al. \cite{FKL}. 
\begin{theorem}[\cite{FKL}]
\label{FKLresult}
For any $m \in \mathbb{N}$ and $\epsilon, \delta \in (0, 1)$, the function $v(m, \epsilon, \delta, 1)$ is equal to $f(m, \epsilon, \ln(1/\delta))$ where: 
\[f(m, \epsilon, p) = 
\begin{cases}
1 & \text{ if } m \ge 2 \epsilon^{-2} e^p \\
\Theta\left(\sqrt{\epsilon}\min\left(\frac{\ln(\frac{m \epsilon}{p})}{p}, \frac{\sqrt{\ln(\frac{m \epsilon^2}{p})}}{\sqrt{p}}\right)\right) &\text{ if } \Theta(\epsilon^{-2} p) \le m < 2 \epsilon^{-2} e^p\\
0 & \text{ if } m \le \Theta(\epsilon^{-2} p).
\end{cases}
\]
\end{theorem}

\subsection*{Generalizing to Sparse Random Projections with $s > 1$}
While Theorem~$\ref{FKLresult}$ is restricted to the case of $s = 1$, dimensionality reduction schemes constructed using sparse random projections with sparsity $s > 1$ have been used in practice for projecting feature vectors. For example, sparse JL-like methods (with $s > 1$) have been used to project feature vectors in machine learning domains including visual tracking \cite{VisTrack}, face recognition \cite{FR}, and recently in ELM \cite{ELM}. Now, a variant of sparse JL is included in the Python sklearn library.\footnote{See \url{https://scikit-learn.org/stable/modules/random_projection.html}.} 

In this context, it is natural to explore how constructions with $s > 1$ perform on feature vectors, by studying $v(m, \epsilon, \delta, s)$ for sparse JL with $s > 1$. In fact, a related question was considered by Weinberger et al. \cite{Weinberger} for ``multiple hashing,'' an alternate distribution over sparse matrices constructed by adding $s$ draws from $\mathcal{A}_{1, m, n}$ and scaling by $1 / \sqrt{s}$. More specifically, they show that $v(m, \epsilon, \delta, s) \ge \min(1, \sqrt{s} \cdot v(m, \epsilon, \delta, 1))$ for multiple hashing. However, Kane and Nelson \cite{KN12} later showed that multiple hashing has worse geometry-preserving properties than sparse JL: that is, multiple hashing requires a larger sparsity than sparse JL to satisfy \eqref{JLcondition}. 

Characterizing $v(m, \epsilon, \delta, s)$ for sparse JL distributions, which are state-of-the-art, remained an open problem. In this work, we settle how $v(m, \epsilon, \delta, s)$ behaves for sparse JL with a general sparsity $s > 1$, giving tight bounds. Our theoretical result shows that sparse JL with $s > 1$, even if $s$ is a small constant, can achieve significantly better norm-preservation properties for feature vectors than sparse JL with $s = 1$. Moreover, we empirically demonstrate this finding. 

\subsection*{Main Results}

We show the following tight bounds on $v(m, \epsilon, \delta, s)$ for a general sparsity $s$: 
\begin{theorem}
\label{thm:mainresult}
For any $s, m \in \mathbb{N}$ such that $s \le m/e$, consider a uniform sparse JL distribution (defined in \Cref{sec:sparseJLDist}) with sparsity $s$ and dimension $m$.\footnote{We prove the lower bound on $v(m, \epsilon, \delta, s)$ in Theorem \ref{thm:mainresult} for \textit{any} sparse JL distribution.} If $\epsilon$ and $\delta$ are small enough\footnote{By ``small enough'', we mean the condition that $\epsilon, \delta \in (0, C')$ for some positive constant $C'$.}, the function $v(m, \epsilon, \delta, s)$ is equal to $f'(m, \epsilon, \ln(1/\delta), s)$, where $f'(m, \epsilon, p, s)$ is\footnote{Notice that the function $f'(m, \epsilon, p, s)$ is not defined for certain ``constant-factor'' intervals between the boundaries of regimes (e.g. $C_1 \epsilon^{-2} p \le m \le C_2 \epsilon^{-2} p$). See Appendix \ref{sec:mainresultsstatement} for a discussion.}: 
\small{
\[
\begin{cases}
1 & \text{if } m \ge \min\left(2 \epsilon^{-2} e^{p}, \epsilon^{-2} p e^{\Theta\left(\max\left(1, \frac{p \epsilon^{-1}}{s}\right)\right)}\right) \\
\Theta\left(\sqrt{\epsilon s}  \frac{\sqrt{\ln(\frac{m \epsilon^2}{p})}}{\sqrt{p}}\right) & \text{else, if } \max\left(\Theta(\epsilon^{-2} p), s \cdot e^{\Theta\left(\max\left(1, \frac{p \epsilon^{-1}}{s}\right)\right)}\right) \le  m \le \epsilon^{-2} e^{\Theta(p)} \\
\Theta\left(\sqrt{\epsilon s} \min\left(\frac{\ln(\frac{m \epsilon}{p})}{p}, \frac{\sqrt{\ln(\frac{m \epsilon^2}{p})}}{\sqrt{p}}\right)\right) & \text {else, if } \Theta(\epsilon^{-2} p) \le m \le \min\left(  \epsilon^{-2} e^{\Theta(p)}, s \cdot e^{\Theta\left(\max\left(1, \frac{p \epsilon^{-1}}{s}\right)\right)}\right) \\
0 & \text{if } m \le \Theta(\epsilon^{-2} p).
\end{cases}
\]}
\end{theorem}

Our main result, Theorem \ref{thm:mainresult}, significantly generalizes Theorem \ref{SparseJL}, Theorem \ref{DimSparsity}, and Theorem \ref{FKLresult}. Notice our bound in Theorem \ref{thm:mainresult} has up to four regimes. In the first regime, which occurs when $m \ge \min(2 \epsilon^{-2}/\delta, \epsilon^{-2} \ln(1/\delta) e^{\Theta(\max(1, \ln(1/\delta) \epsilon^{-1}/s))})$, Theorem \ref{thm:mainresult} shows $v(m, \epsilon, \delta, s) = 1$, so $\eqref{JLcondition}$ holds on the full space $\mathbb{R}^n$. Notice this boundary on $m$ occurs at the dimensionality-sparsity tradeoff in Theorem \ref{DimSparsity}. In the last regime, which occurs when $m \le \Theta(\epsilon^{-2} \ln (1/\delta))$, Theorem \ref{thm:mainresult} shows that $v(m, \epsilon, \delta, s) = 0$, so there are vectors with arbitrarily small $\ell_{\infty}$-to-$\ell_2$ norm ratio where $\eqref{JLcondition}$ does not hold. When $s \le \Theta(\epsilon^{-1}\ln(1/\delta))$, Theorem \ref{thm:mainresult} shows that up to two intermediate regimes exist. One of the regimes, $\Theta(\sqrt{\epsilon s}\min(\ln(\frac{m \epsilon}{p})/p, \sqrt{\ln(\frac{m \epsilon^2}{p})}/\sqrt{p}))$, matches the middle regime of $v(m, \epsilon, \delta, 1)$ in Theorem \ref{FKLresult} with an extra factor of $\sqrt{s}$, much like the bound for multiple hashing in \cite{Weinberger} that we mentioned previously. However, unlike the multiple hashing bound, Theorem \ref{thm:mainresult} sometimes has another regime, $\Theta(\sqrt{\epsilon s} \sqrt{\ln(\frac{m \epsilon^2}{p})}/\sqrt{p})$, which does not arise for $s = 1$ (i.e. in Theorem \ref{FKLresult}).\footnote{This regime does not arise for $s = 1$, since $e^{\Theta\left(p \epsilon^{-1}\right)} \geq \epsilon^{-2} e^{\Theta\left(p\right)}$ for sufficiently small $\epsilon$.} Intuitively, we expect this additional regime for sparse JL with $s$ close to $\Theta(\epsilon^{-1}\ln(1/\delta))$: at $s = \Theta(\epsilon^{-1} \ln(1/\delta))$ and $m = \Theta(\epsilon^{-2} \ln(1/\delta))$, Theorem~$\ref{SparseJL}$ tells us $v(m, \epsilon, \delta, s) = 1$, but if $\epsilon$ is a constant, then the branch $\Theta(\sqrt{\epsilon s} \ln\left(\frac{m \epsilon}{p}\right)/p)$ yields $\Theta(1/\sqrt{\ln(1/\delta)})$, while the branch $\Theta(\sqrt{\epsilon s} \sqrt{\ln(\frac{m \epsilon^2}{p})}/\sqrt{p})$ yields $\Theta(1)$. Thus, it is natural that the first branch disappears for large $m$. 

Our result elucidates that $v(m, \epsilon, \delta, s)$ increases approximately as $\sqrt{s}$, thus providing insight into how even small constant increases in sparsity can be useful in practice. Another consequence of our result is a lower bound on dimension-sparsity tradeoffs (Corollary \ref{dimsparsitylower} in Appendix \ref{sec:mainresultsstatement}) that essentially matches the upper bound in Theorem \ref{DimSparsity}. Moreover, we require new techniques to prove Theorem \ref{thm:mainresult}, for reasons that we discuss further in Section \ref{subsec:prooftechniques}.

We also empirically support our theoretical findings in Theorem \ref{thm:mainresult}. First, we illustrate with real-world datasets the potential benefits of using small constants $s > 1$ for sparse JL on feature vectors. We specifically show that $s = \left\{4, 8, 16\right\}$ consistently outperforms $s = 1$ in preserving the $\ell_2$ norm of each vector, and that there can be up to a \textit{factor of ten} decrease in failure probability for $s = 8, 16$ in comparison to $s = 1$. Second, we use synthetic data to illustrate phase transitions and other trends in Theorem \ref{thm:mainresult}. More specifically, we empirically show that $v(m, \epsilon, \delta, s)$ is not smooth, and that the middle regime(s) of $v(m, \epsilon, \delta, s)$ increases with $s$.  

% Previous work on studying $v(m, \epsilon, \delta, s)$ for $s > 1$ for the related construction of ``multiple hashing'' gives some intuition for the functions in our bound in Theorem \ref{thm:mainresult}. In \cite{Weinberger}, a reduction is given that demonstrates that $v(m, \epsilon, \delta, s) \ge \min(1, \sqrt{s} \cdot v(m, \epsilon, \delta, 1))$ for multiple hashing. This bound provides intuition for the $\sqrt{s}$ term arising in Theorem \ref{thm:mainresult}. However, our result differs in an important way. Multiple hashing is a different distribution than uniform sparse JL, and it is shown in \cite{KN12} that multiple hashing actually requires roughly an extra $\ln(1/\delta)$ factor on the sparsity to satisfy \eqref{JLcondition}. Moreover, the bound for multiple hashing in \cite{Weinberger} is different than our bound in Theorem \ref{thm:mainresult}: more specifically, it does not include regime (3) in Theorem \ref{thm:mainresult}, which we have shown is necessary for characterizing $v(m, \epsilon, \delta, s)$ for sparse JL.

\subsection{Preliminaries}
\label{sec:sparseJLDist}
Let $\mathcal{A}_{s,m,n}$ be a \textbf{sparse JL distribution} if the entries of a matrix $A \in \mathcal{A}_{s,m,n}$ are generated as follows. Let $A_{r,i} = \eta_{r,i} \sigma_{r,i} / \sqrt{s}$ where $\left\{\sigma_{r,i} \right\}_{r \in [m], i \in [n]}$ and $\left\{\eta_{r,i} \right\}_{r \in [m], i \in [n]}$ are defined as follows:
\begin{itemize}
\item The families $\left\{\sigma_{r,i} \right\}_{r \in [m], i \in [n]}$ and $\left\{ \eta_{r,i} \right\}_{r \in [m], i \in [n]}$ are independent from each other.
\item The variables $\left\{\sigma_{r,i} \right\}_{r \in [m], i \in [n]}$ are i.i.d Rademachers ($\pm 1$ coin flips).
\item The variables $\left\{ \eta_{r,i} \right\}_{r \in [m], i \in [n]}$ are identically distributed Bernoullis ($\left\{0,1\right\}$ random variables) with expectation $s/m$.
\item The $\left\{ \eta_{r,i} \right\}_{r \in [m], i \in [n]}$ are independent across columns but not independent within each column. For every column $1 \le i \le n$, it holds that $\sum_{r=1}^m \eta_{r,i} = s$. Moreover, the random variables are \textit{negatively correlated}: for every subset $S \subseteq [m]$ and every column $1 \le i \le n$, it holds that $\mathbb{E} \left[\prod_{r \in S} \eta_{r, i}\right] \le \prod_{r \in S} \mathbb{E} [\eta_{r, i}]$. 
\end{itemize}
A common special case is a \textbf{uniform sparse JL distribution}, generated as follows: for every $1 \le i \le n$, we \textit{uniformly} choose exactly $s$ of these variables in $\left\{\eta_{r,i}\right\}_{r \in [m]}$ to be $1$. When $s = 1$, every sparse JL distribution is a uniform sparse JL distribution, but for $s > 1$, this is not the case.

Another common special case is a \textbf{block sparse JL distribution}. This produces a different construction for $s > 1$. In this distribution, each column $1 \le i \le n$ is partitioned into $s$ blocks of $\floor{\frac{m}{s}}$ consecutive rows. In each block in each column, the distribution of the variables $\left\{\eta_{r,i}\right\}$ is defined by uniformly choosing \textit{exactly one} of these variables to be $1$.\footnote{Our lower bound in Theorem \ref{thm:mainresult} applies to this distribution, though our upper bound does not. An interesting direction for future work would be to generalize the upper bound to this distribution.}

\subsection{Proof Techniques}\label{subsec:prooftechniques}
We use the following notation. For any random variable $X$ and value $q \ge 1$, we call $\mathbb{E}[|X|^q]$ the $q$th \textit{moment} of $X$, where $\mathbb{E}$ denotes the expectation. We use $\norm{X}_q$ to denote the $q$-norm $\left(\mathbb{E}[|X|^q]\right)^{1/q}$.

For every $[x_1, \ldots, x_n] \in \mathbb{R}^n$ such that $\norm{x}_2 = 1$, we need to analyze tail bounds of an error term, which for the sparse JL construction is the following random variable:
\[\norm{Ax}^2_2 - 1 = \frac{1}{s} \sum_{i \neq j} \sum_{r=1}^m \eta_{r,i} \eta_{r,j} \sigma_{r,i} \sigma_{r,j} x_i x_j =: R(x_1, \ldots, x_n).\] 
An upper bound on the tail probability of $R(x_1, \ldots, x_n)$ is needed to prove the lower bound on $v(m, \epsilon, \delta, s)$ in Theorem \ref{thm:mainresult}, and a lower bound is needed to prove the upper bound on $v(m, \epsilon, \delta, s)$ in Theorem \ref{thm:mainresult}. It turns out that it suffices to tightly analyze the random variable moments $\mathbb{E}[(R(x_1, \ldots, x_n))^q]$. For the upper bound, we use Markov's inequality like in \cite{FKL, KN12, Original, NN13}, and for the lower bound, we use the Paley-Zygmund inequality like in \cite{FKL}: Markov's inequality gives a tail upper bound from upper bounds on moments, and the Paley-Zygmund inequality gives a tail lower bound from upper and lower bounds on moments. Thus, the key ingredient of our analysis is a \textit{tight bound} for $\norm{R(x_1, \ldots, x_n)}_q$ on $S_v = \left\{x \in \mathbb{R}^n \mid \frac{\norm{x}_{\infty}}{\norm{x}_2} \le v\right\}$ at \textit{each} threshold $v$ value. 

% We use the following notation in our analysis. For any random variable $X$ and value $q \ge 1$, we use $\norm{X}_q$ to denote the $q$-norm $\left(\mathbb{E}[|X|^q]\right)^{1/q}$, where $\mathbb{E}$ denotes the expectation. Given two scalar quantities $Q_1$ and $Q_2$ that are functions of some parameters, we use $Q_1 \simeq Q_2$ to denote that there exist positive universal constants $C_1 \le C_2$ such that $C_1Q_2 \le Q_1 \le C_2Q_2$, and we use $Q_1 \lesssim Q_2$ (resp. $Q_1 \gtrsim Q_2$) to denote that there exists a positive universal constant $C$ such that $Q_1 \le CQ_2$ (resp. $Q_1 \geq CQ_2)$.

While the moments of $R(x_1, \ldots, x_n)$ have been studied in previous analyses of sparse JL, we emphasize that it is not clear how to adapt these existing approaches to obtain a tight bound on every $S_v$. The moment bound that we require and obtain is far more general: the bounds in \cite{KN12, NelsonNotes} are limited to $\mathbb{R}^n = S_1$ and the bound in \cite{FKL} is limited to $s = 1$.\footnote{As described in \cite{FKL}, even for the case for $s = 1$, the approach in \cite{KN12} cannot be directly generalized to recover Theorem \ref{FKLresult}. Moreover, the approach in \cite{FKL}, though more precise for $s = 1$, is highly tailored to $s = 1$, and it is not clear how to generalize it to $s > 1$.} The non-combinatorial approach in \cite{NelsonNotes} for bounding $\norm{R(x_1, \ldots, x_n)}_q$ on $\mathbb{R}^n = S_1$ also turns out to not be sufficiently precise on $S_v$, for reasons we discuss in \Cref{sec:momentbounding}.\footnote{In predecessor work \cite{Me}, we give a non-combinatorial approach similar to \cite{NelsonNotes} for a sign-consistent variant of the JL distribution. Moreover, a different non-combinatorial approach for subspace embeddings is given in \cite{Cohen}. However, these approaches both suffer from issues in this setting that are similar to \cite{NelsonNotes}.} 

Thus, we require new tools for our moment bound. Our analysis provides a new perspective, inspired by the probability theory literature, that differs from the existing approaches in the JL literature. We believe our style of analysis is less brittle than combinatorial approaches \cite{FKL, KN12, Original, NN13}: in this setting, once the sparsity $s = 1$ case is recovered, it becomes straightforward to generalize to other $s$ values. Moreover, our approach can yield greater precision than the existing non-combinatorial approaches \cite{NelsonNotes, Cohen, Me}, which is necessary for this setting. Thus, we believe that our \textit{structural} approach to analyzing JL distributions could be of use in other settings. 

In \Cref{sec:momentbounding}, we present an overview of our methods and the key technical lemmas to analyze $\norm{R(x_1, \ldots, x_n)}_q$. We defer the proofs to the Appendix. In \Cref{sec:mainresultsproof}, we prove the tail bounds in Theorem~$\ref{thm:mainresult}$ from these moment bounds. In \Cref{sec:empirical}, we empirically evaluate sparse JL.
% we show on real-world datasets that sparse JL with $s \ge 4$ can achieve better norm-preservation properties than $s = 1$, and we empirically show the theoretical trends in Theorem \ref{thm:mainresult} on synthetic data for the block sparse JL distribution. 

\section{Sketch of Bounding the Moments of $R(x_1, \ldots, x_n)$}
\label{sec:momentbounding}

Our approach takes advantage of the structure of $R(x_1, \ldots, x_n)$ as a quadratic form of Rademachers (i.e. $\sum_{t_1,t_2} a_{t_1,t_2} \sigma_{t_1} \sigma_{t_2}$) with random variable coefficients (i.e. where $a_{t_1, t_2}$ is itself a random variable). For the upper bound, we need to analyze $\norm{R(x_1, \ldots, x_n)}_q$ for general vectors $[x_1, \ldots, x_n]$. For the lower bound, we only need to show $\norm{R(x_1, \ldots, x_n)}_q$ is large for single vector in each $S_v$, and we show we can select the vector in the $\ell_2$-unit ball with $1/v^2$ nonzero entries, all equal to $v$. For ease of notation, we denote this vector by $[v, \ldots, v, 0, \ldots, 0]$ for the remainder of the paper.

We analyze $\norm{R(x_1, \ldots, x_n)}_q$ using general moment bounds for Rademacher linear and quadratic forms. Though Cohen, Jayram, and Nelson \cite{NelsonNotes} also view $R(x_1, \ldots, x_n)$ as a quadratic form, we show in Appendix \ref{sec:looseHW} that their approach of bounding the Rademachers by gaussians is not sufficiently precise for our setting.\footnote{We actually made a similar conceptual point for a different JL distribution in our predecessor work \cite{Me}, but the alternate bound that we produce there also suffers from precision issues in this setting.}

In our approach, we make use of stronger moment bounds for Rademacher linear and quadratic forms, some of which are known to the probability theory community through Lata{\l}a's work in \cite{LatalaChaos, LatalaMoments} and some of which are new adaptions tailored to the constraints arising in our setting. More specifically, Lata{\l}a's bounds \cite{LatalaChaos, LatalaMoments} target the setting where the coefficients are scalars. In our setting, however, the coefficients are themselves random variables, and we need bounds that are \textit{tractable} to analyze in this setting, which involves creating new bounds to handle some cases. 

Our strategy for bounding $\norm{R(x_1, \ldots, x_n)}_q$ is to break down into rows. We define
\[Z_r(x_1, \ldots, x_n) := \sum_{1 \le i \neq j \le n} \eta_{r,i} \eta_{r,j} \sigma_{r,i} \sigma_{r,j} x_i x_j\] so that $R(x_1, \ldots, x_n) = \frac{1}{s} \sum_{r=1}^m Z_r(x_1, \ldots, x_n)$. We analyze the moments of $Z_r(x_1, \ldots, x_n)$, and then combine these bounds to obtain moment bounds for $R(x_1, \ldots, x_n)$. In our bounds, we use the notation $f \lesssim g$ (resp. $f \gtrsim g$) to denote $f \leq Cg$ (resp. $f \geq C g$) for some constant $C > 0$. 

% First, we describe our methods to bound the moments of each row (i.e. $Z_r(x_1, \ldots, x_n)$). To obtain a lower bound on the moments of $Z_r(v, \ldots, v, 0, \ldots, 0)$, we make use of the following tight bound on moments of quadratic forms of Rademacher random variables with random variable coefficients. We derive this bound from Lata{\l}a's bound\footnote{In fact, Lata{\l}a shows moment bounds for much more general quadratic forms, but for the application to JL, we only need the bound in the special case of Rademachers.} on Rademacher quadratic forms \cite{LatalaChaos}, and we defer the proof of this lemma from Lata{\l}a's result to \Cref{sec:usefulmomentboundsproofs}. 
% \begin{lemma}
% \label{usefulquadformbound}
% Let $T$ be an even integer, $\left\{\sigma_i\right\}_{1 \le i \le n}$ be independent Rademachers, and $(Y_{i,j})_{1 \le i, j \le n}$ be a $n \times n$ symmetric, nonnegative random matrix with zero diagonal (i.e. $Y_{i,i} = 0$) such that $\left\{Y_{i,j}\right\}_{1 \le i,j \le n}$ is independent from $\left\{\sigma_i\right\}_{1 \le i \le n}$. If $W_i = \sqrt{\sum_{1 \le j \le n} Y_{i,j}^2}$, then: 
% \[\norm{\sum_{1 \le i, j \le n} Y_{i,j} \sigma_i \sigma_j}_T \simeq \norm{\sup_{\norm{b}_2, \norm{c}_2 \le \sqrt{T}, \norm{b}_{\infty}, \norm{c}_{\infty} \le 1} \sum_{1 \le i, j \le n} Y_{i,j} b_i c_j}_T  + \norm{\sum_{1 \le i \le T} W_{(i)} + \sqrt{T} \sqrt{\sum_{T < i \le n} W_{(i)}^2}}_T \] where $W_{(1)} \ge W_{(2)} \ge \ldots \ge \ldots W_{(n)}$ is a permutation of $W_1, \ldots, W_n$. 
% \end{lemma}

\subsection{Bounding $\norm{Z_r(x_1, \ldots, x_n)}_q$}
We show the following bounds on $\norm{Z_r(x_1, \ldots, x_n)}_q$. For the lower bound, as we discussed before, it suffices to bound $\norm{Z_r(v, \ldots, v, 0, \ldots, 0)}_q$. For the upper bound, we need to bound $\norm{Z_r(x_1, \ldots, x_n)}_q$ for general vectors as a function of the $\ell_{\infty}$-to-$\ell_2$ norm ratio. 
\begin{lemma}
\label{rowbound}
Let $\mathcal{A}_{s,m, n}$ be a sparse JL distribution such that $s \le m/e$. Suppose that $x = [x_1, \ldots, x_n]$ satisfies $\norm{x}_{\infty} \le v$ and $\norm{x}_2 = 1$. If $T$ is even, then:
\[\norm{Z_r(x_1, \ldots, x_n)}_T \lesssim 
 \begin{cases}
    \frac{Ts}{m}, & \text{for } T=2, 3 \le T \le \frac{se}{mv^2} \\
    \min\left(\frac{T^2v^2}{\ln(mTv^2/s)^2}, \frac{T}{\ln(m/s)}\right)& \text{for }  T \ge 3, T \ge \frac{se}{mv^2}, \ln(Tmv^2/s) \le T\\
  v^2 \left(\frac{s}{mTv^2}\right)^{2/T}, & \text{for } T \ge 3, T \ge \frac{se}{mv^2}, \ln(Tmv^2/s) > T.
  \end{cases}
 \]
\end{lemma}
\begin{lemma}
\label{lowerboundrow}
Let $\mathcal{A}_{s,m, n}$ be a sparse JL distribution. Suppose $\frac{1}{v^2}$ and $T$ are even integers. 

Then, $\norm{Z_r(v, \ldots, v, 0, \ldots, 0)}_2 \gtrsim \frac{s}{m}$. Moreover, if $s \le m/e$ and $T \ge \frac{se}{mv^2}$, then 
\[\norm{Z_r(v, \ldots, v, 0, \ldots, 0)I_{\sum_{i=1}^{1/v^2} \eta_{1,i} = 2}}_T \gtrsim v^2\left(\frac{s}{mTv^2}\right)^{2/T}\] and 
\[
 \norm{Z_r(v, \ldots, v, 0, \ldots, 0)}_T \gtrsim
    \begin{cases}
    \frac{T^2v^2}{\ln^2(mv^2T/s)} & \text{ for } 1\le \ln(mv^2T/s) \le T, v \le \frac{\sqrt{\ln(m/s)}}{\sqrt{T}}\\
  v^2\left(\frac{s}{mTv^2}\right)^{2/T}  & \text{ for } \ln(mv^2T/s) > T.
  \end{cases}
\]
\end{lemma}

We now sketch our methods to prove Lemma \ref{rowbound} and Lemma \ref{lowerboundrow}. For the lower bound (Lemma \ref{lowerboundrow}), we can view $Z_r(v, \ldots, v, 0, \ldots, 0)$ as a quadratic form $\sum_{t_1, t_2} a_{t_1, t_2} \sigma_{t_1} \sigma_{t_2}$ where $(a_{t_1, t_2})_{t_1, t_2 \in [mn]}$ is an appropriately defined block-diagonal $mn$ dimensional matrix. We can write $\mathbb{E}_{\sigma, \eta}[(Z_r(v, \ldots, v, 0, \ldots, 0))^q]$ as $\mathbb{E}_{\eta} \left[\mathbb{E}_{\sigma}[(Z_r(v, \ldots, v, 0, \ldots, 0))^q]\right]$: for \textit{fixed} $\eta_{r,i}$ values, the coefficients are scalars. We make use of Lata{\l}a's tight bound on Rademacher quadratic forms with scalar coefficients \cite{LatalaChaos} to analyze $\mathbb{E}_{\sigma}[(Z_r(v, \ldots, v, 0, \ldots, 0))^q]$ as a function of the $\eta_{r,i}$. Then, we handle the randomness of the $\eta_{r,i}$ by taking an expectation of the resulting bound on $\mathbb{E}_{\sigma}[(Z_r(v, \ldots, v, 0, \ldots, 0))^q]$ over the $\eta_{r,i}$ values to obtain a bound on $\norm{Z_r(v, \ldots, v, 0, \ldots, 0)}_q$.  

For the upper bound (Lemma \ref{rowbound}), since Lata{\l}a's bound \cite{LatalaChaos} is tight for scalar quadratic forms, the natural approach would be to use it to upper bound $\mathbb{E}_{\sigma}[(Z_r(x_1, \ldots, x_n))^q]$ for general vectors. However, when the vector is not of the form $[v, \ldots, v, 0, \ldots, 0]$, the asymmetry makes the resulting bound intractable to simplify. Specifically, there is a term, which can be viewed as a generalization of an operator norm to an $\ell_2$ ball cut out by $\ell_{\infty}$ hyperplanes, that  becomes problematic when taking an expectation over the $\eta_{r,i}$ to obtain a bound on $\mathbb{E}_{\sigma, \eta}[(Z_r(x_1, \ldots, x_n))^q]$. Thus, we construct simpler estimates that avoid these complications while remaining sufficiently precise for our setting. These estimates take advantage of the structure of $Z_r(x_1, \ldots, x_n)$ and enable us to show Lemma \ref{rowbound}.  

\subsection{Obtaining bounds on $\norm{R(x_1, \ldots, x_n)}_q$}
Now, we use Lemma \ref{rowbound} and Lemma \ref{lowerboundrow} to show the following bounds on $\norm{R(x_1, \ldots, x_n)}_q$:
\begin{lemma}
\label{upper}
Suppose $\mathcal{A}_{s,m,n}$ is a sparse JL distribution such that $s \le m/e$, and let $x = [x_1, \ldots, x_n]$ be such that $\norm{x}_2 = 1$.
Then, $\norm{R(x_1, \ldots, x_n)}_2 \le \frac{\sqrt{2}}{\sqrt{m}}$. Now, suppose that $2 < q \le m$ is an even integer and $\norm{x}_{\infty} \le v$. If $\frac{se}{mv^2} \ge q$, then $\norm{R(x_1, \ldots, x_n)}_q \lesssim \frac{\sqrt{q}}{\sqrt{m}}$. If $\frac{se}{mv^2} < q$ and if there exists a constant $C_2 \ge 1$ such that $C_2q^3mv^4 \ge s^2$, then $\norm{R(x_1, \ldots, x_n)}_q \lesssim g$ where $g$ is:
\small{
\[\begin{cases} 
\max\left(\frac{\sqrt{q}}{\sqrt{m}}, \frac{C_2^{1/3}q^2v^2}{s \ln^2(qmv^2/s)}\right) & \text{ if } \ln(\frac{qmv^4}{s^2}) \le 2, \ln(\frac{qmv^2}{s}) \le q\\ 
\frac{\sqrt{q}}{\sqrt{m}} & \text{ if } \ln(\frac{qmv^4}{s^2}) \le 2, \ln(\frac{qmv^2}{s})> q\\ 
\max\left(\frac{\sqrt{q}}{\sqrt{m}}, \frac{qv^2}{s\ln(qmv^4/s^2)}, \min\left( \frac{C_2^{1/3}q^2v^2}{s \ln^2(qmv^2/s)}, \frac{q}{s\ln(m/s)}\right)\right) & \text{ if } \ln(\frac{qmv^4}{s^2}) > 2, \ln(\frac{qmv^2}{s}) \le q\\
\max\left(\frac{\sqrt{q}}{\sqrt{m}}, \frac{qv^2}{s\ln(qmv^4/s^2)}\right) & \text{ if } \ln(\frac{qmv^4}{s^2}) > 2, \ln(\frac{qmv^2}{s})> q. \\
\end{cases}
\]}
\end{lemma}
\normalsize{}
\begin{lemma}
\label{lower}
Suppose $\mathcal{A}_{s,m,n}$ is a uniform sparse JL distribution. Let $q$ be a power of $2$, and suppose that $0 < v \le 0.5$ and $\frac{1}{v^2}$ is an even integer. If $qv^2 \le s$, then $\norm{R(v, \ldots, v, 0, \ldots, 0)}_q \gtrsim \frac{\sqrt{q}}{\sqrt{m}}$. If $m \ge q$,  $2 \le \ln(qmv^4/s^2) \le q$, $2qv^2 \le 0.5 s \ln(qmv^4/s^2)$, and $s \le m/e$, then $\norm{R(v, \ldots, v, 0, \ldots, 0)}_q \gtrsim \frac{qv^2}{s \ln(qmv^4/s^2)}$. If $s \le m/e$, $v \le \frac{\sqrt{\ln(m/s)}}{\sqrt{q}}$, and $1 \le \ln(qmv^2/s) \le q$, then $\norm{R(v, \ldots, v, 0 \ldots, 0)}_q \gtrsim \frac{q^2v^2}{s \ln^2(qmv^2/s)}$. 
\end{lemma}
We now sketch how to prove bounds on $\norm{R(x_1, \ldots, x_n)}_q$ using bounds on $\norm{Z_r(x_1, \ldots, x_n)}_T$. To show Lemma \ref{upper}, we show that making the row terms $Z_r(x_1, \ldots, x_n)$ independent does not decrease $\norm{R(x_1, \ldots, x_n)}_q$, and then we apply a general result from \cite{LatalaMoments} for moments of sums of i.i.d symmetric random variables. For Lemma \ref{lower}, handling the correlations between the row terms $Z_r(x_1, \ldots, x_n)$ requires more care. We show that the negative correlations induced by having exactly $s$ nonzero entries per column do not lead to significant loss, and then stitch together $\norm{R(v, \ldots, v, 0, \ldots, 0)}_q$ using the moments of $Z_r(v, \ldots, v, 0, \ldots, 0)$ that contribute the most.

\section{Proof of Main Result from Moment Bounds}
\label{sec:mainresultsproof}
We now sketch how to prove Theorem \ref{thm:mainresult}, using Lemma \ref{upper} and Lemma \ref{lower}. First, we simplify these bounds at the target parameters to obtain the following: 
\begin{lemma}
\label{upperfinal}
Let $\mathcal{A}_{s,m,n}$ be a sparse JL distribution, and suppose $\epsilon$ and $\delta$ are small enough, $s \le m /e$, $\Theta(\epsilon^{-2} \ln(1/\delta)) \le m < 2 \epsilon^{-2}/\delta$,
$v \le f'(m, \epsilon, \ln(1/\delta), s)$, and $p = \Theta(\ln(1/\delta))$ is even. If $x = [x_1, \ldots, x_n]$ satisfies $\norm{x}_{\infty} \le v$ and $\norm{x}_2 = 1$, then $\norm{R(x_1, \ldots, x_n)}_p \leq \frac{\epsilon}{2}$. 
\end{lemma}
\begin{lemma}
\label{lowerupperfinal}
There is a universal constant $D$ satisfying the following property. Let $\mathcal{A}_{s,m,n}$ be a uniform sparse JL distribution, and suppose $\epsilon, \delta$ are small enough, $s \le m/e$, $f'(m, \epsilon, \ln(1/\delta), s) \le 0.5$, and $q$ is an even integer such that $q = \min(m /2, \Theta(\ln(1/\delta)))$. For each $\psi > 0$, there exists $v \le f'(m, \epsilon, \ln(1/\delta), s) + \psi$, such that $\norm{R(v, \ldots, v, 0, \ldots, 0)}_q \ge 2 \epsilon$ and $\frac{\norm{R(v, \ldots, v, 0, \ldots, 0)}_q}{\norm{R(v, \ldots, v, 0, \ldots, 0)}_{2q}} \geq D$. 
\end{lemma}

Now, we use Lemma \ref{upperfinal} and Lemma \ref{lowerupperfinal} to prove Theorem \ref{thm:mainresult}. 
\begin{proof}[Proof of Theorem~$\ref{thm:mainresult}$]
Since the maps in $\mathcal{A}_{s,m,n}$ are linear, it suffices to consider unit vectors $x$. 
First, we prove the lower bound on $v(m, \epsilon, \delta, s)$. To handle $m \ge 2 \epsilon^{-2} / \delta$, we take $q=2$ in Lemma \ref{upperfinal} and apply Chebyshev's inequality. Otherwise, we take $p = \ln(1/\delta)$ (approximately) and apply Lemma \ref{upperfinal} and Markov's inequality. We see that $\mathbb{P}[|\norm{Ax}_2^2 - 1| \ge \epsilon]$ can be expressed as:
\[\mathbb{P}[|R(x_1, \ldots, x_n)| \ge \epsilon] = \mathbb{P}[R(x_1, \ldots, x_n)^p \ge \epsilon^p] \le \epsilon^{-p} \mathbb{E}[R(x_1, \ldots, x_n)]^p \le \delta.\]
Thus, condition \eqref{JLcondition} is satisfied for $x \in S_v$ when $v \le f'(m, \epsilon, \ln(1/\delta), s)$ as desired. 
% When $m \ge \epsilon^{-2} p e^{\frac{p \epsilon^{-1}}{s}}$, we have that $C_v \sqrt{\epsilon s} \frac{\sqrt{\ln(m \epsilon^2)}}{\sqrt{p}}\ge 1$ for some constant $C_v > 0$, so we know that $v(m, \epsilon, \delta, s) = 1$. Moreover, we need to take a $\min$ with $1$ since $S_v = S_1$ for $v \ge 1$.

Now, we prove the upper bound on $v(m, \epsilon, \delta, s)$. We need to lower bound the tail probability of \\
$R(v,\ldots, v, 0, \ldots,0)$, and to do this, we use the Paley-Zygmund inequality applied to $q$th moments. Let $D$ be defined as in Lemma \ref{lowerupperfinal}, and take $q = \min(m/2, \max(2, \frac{\ln(1/\delta)-2}{-2\ln(D)}))$. By the Paley-Zygmund inequality and Lemma \ref{lowerupperfinal}, there exists $v \le f'(m, \epsilon, \ln(1/\delta), s) + \psi$ such that:
\[\mathbb{P}[|R(v, \ldots, v, 0, \ldots, 0)| > \epsilon] \ge 0.25 \left(\frac{\norm{R(v, v, \ldots, v, 0, \ldots, 0)}_q}{\norm{R(v, v, \ldots, v, 0, \ldots, 0)}_{2q}}\right)^{2q} \ge 0.25 D^{2q} > \delta.\]
Thus, it follows that $\sup_{x \in S_{f'(m, \epsilon, \ln(1/\delta), s) + \psi}, \norm{x}_2 = 1} \mathbb{P}[|\norm{Ax}_2^2 - 1| > \epsilon]  > \delta$ as desired. 
\end{proof}

\section{Empirical Evaluation}\label{sec:empirical}
Recall that for sparse JL distributions with sparsity $s$, the projection time for an input vector $x$ is $O(s \norm{x}_0)$, where $\norm{x}_0$ is the number of nonzero entries in $x$. Since this grows linearly in $s$, in order to minimize the impact on projection time, we restrict to small constant $s$ values (i.e. $1 \le s \le 16$). In Section \ref{subsec:realworld}, we demonstrate on real-world data the benefits of using $s > 1$. In Section \ref{subsec:synthetic}, we illustrate trends in our theoretical bounds on synthetic data. Additional graphs can be found in Appendix \ref{sec:additionalexperiments}. For all experiments, we use a block sparse JL distribution to demonstrate that our theoretical upper bounds also empirically generalize to non-uniform sparse JL distributions. 
\subsection{Real-World Datasets}\label{subsec:realworld}
We considered two bag-of-words datasets: the News20 dataset \cite{News20} (based on newsgroup documents), and the Enron email dataset \cite{Dataset} (based on e-mails from the senior management of Enron).\footnote{Note that the News20 dataset is used in \cite{DKT17}, and the Enron dataset is from the same collection as the dataset used in \cite{FKL}, but contains a larger number of documents.} Both datasets were pre-processed with the standard \texttt{tf-idf} preprocessing. In this experiment, we evaluated how well sparse JL preserves the $\ell_2$ norms of the vectors in the dataset. An interesting direction for future work would be to empirically evaluate how well sparse JL preserves other aspects of the geometry of real-world data sets, such as the $\ell_2$ distances between pairs of vectors. 

In our experiment, we estimated the failure probability $\hat{\delta}(s,m, \epsilon)$ for each dataset as follows. Let $D$ be the number of vectors in the dataset, and let $n$ be the dimension ($n = 101631$, $D = 11314$ for News20; $n = 28102$, $D = 39861$ for Enron). We drew a matrix $M \sim \mathcal{A}_{s,m,n}$  from a block sparse JL distribution. Then, we computed $\frac{\norm{Mx}_2}{\norm{x}_2}$ for each vector $x$ in the dataset, and used these values to compute an estimate $\hat{\delta}(s, m, \epsilon) = \frac{\text{ number of vectors $x$ such that } \frac{\norm{Mx}_2}{\norm{x}_2} \not\in 1 \pm \epsilon}{D}$. We ran 100 trials to produce 100 estimates $\hat{\delta}(s, m, \epsilon)$. 

\begin{figure}[!htb]
\centering
\begin{minipage}[b]{.5\textwidth}
  \centering
  \includegraphics[scale=0.45]{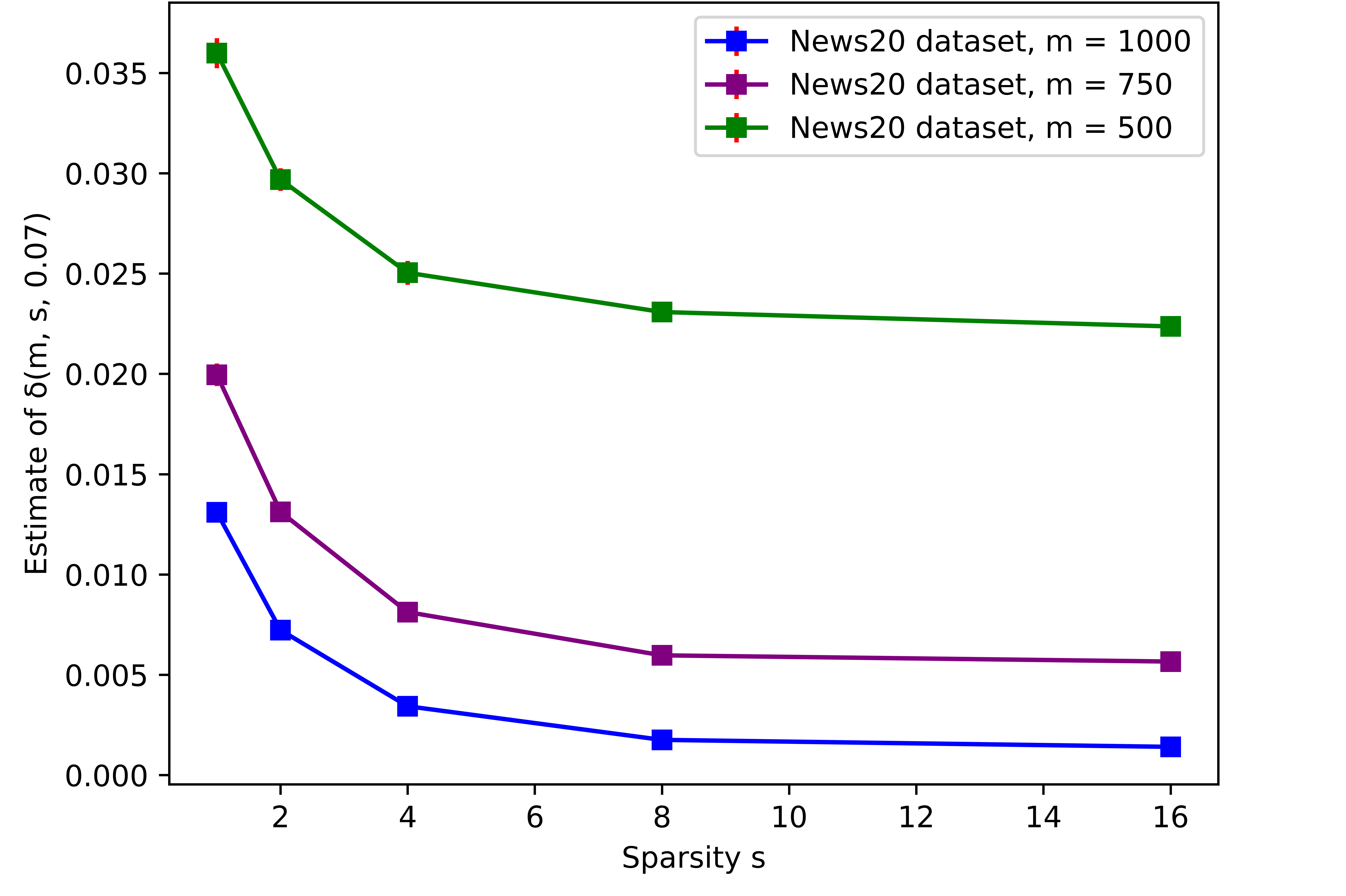}
  \caption{News20: $\hat{\delta}(m, s, 0.07)$ v. $s$}
  \label{fig:news20007}
\end{minipage}%
\begin{minipage}[b]{.5\textwidth}
  \centering
  \includegraphics[scale=0.45]{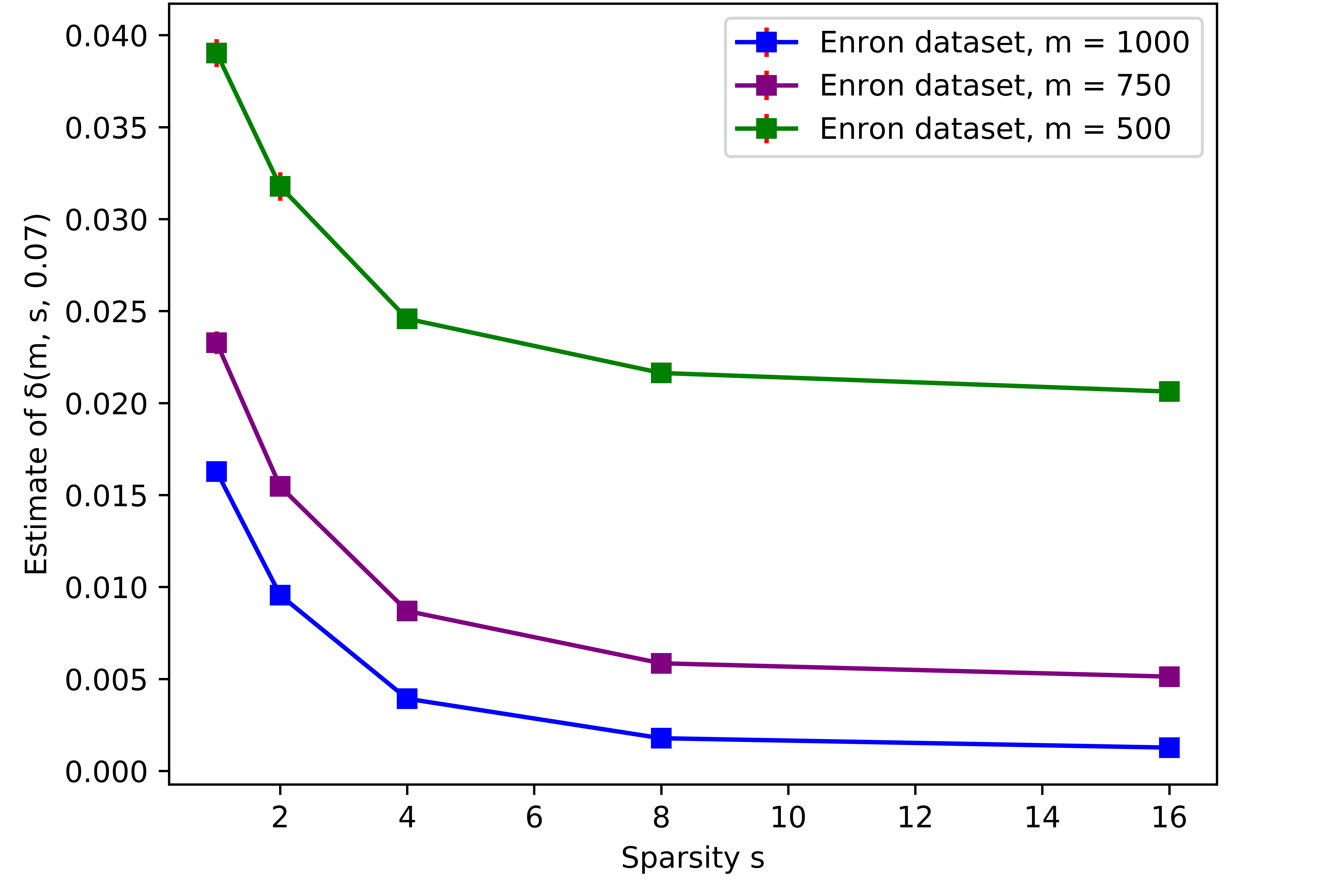}
  \caption{Enron: $\hat{\delta}(m, s, 0.07)$ vs. $s$}
  \label{fig:enron007}
\end{minipage}
% \begin{minipage}[b]{.5\linewidth}
%   \centering
%   \includegraphics[scale=0.48]{news2001.png}
%   \caption{News20: $\hat{\delta}(m, s, 0.1)$ vs. $s$}
%   \label{fig:news2001}
% \end{minipage}%
% \begin{minipage}[b]{.5\linewidth}
%   \centering
%   \includegraphics[scale=0.48]{enron001.png}
%   \caption{Enron: $\hat{\delta}(m, s, 0.1)$ vs. $s$}
%   \label{fig:enron01}
% \end{minipage}
\end{figure}

Figure \ref{fig:news20007} and Figure \ref{fig:enron007} show the mean and error bars (3 standard errors of the mean) of $\hat{\delta}(s, m, \epsilon)$ at $\epsilon = 0.07$. We consider $s \in \left\{1, 2, 4, 8, 16\right\}$, and choose $m$ values so that $0.01 \le \hat{\delta}(1,m,\epsilon) \le 0.04$.

% In the supplementary material, we show plots for other $\epsilon$ and $\delta$ values -- the shape looks quite similar. 

All of the plots show that $s \in \left\{2, 4, 8, 16\right\}$ achieves a lower failure probability than $s = 1$, with the differences most pronounced when $m$ is larger. In fact, at $m = 1000$, there is a \textit{factor of four} decrease in $\delta$ between $s = 1$ and $s = 4$, and a \textit{factor of ten} decrease between $s = 1$ and $s = 8, 16$. We note that in plots in the Appendix, there is a slight increase between $s = 8$ and $s = 16$ at some $\epsilon, \delta, m$ values (see Appendix \ref{sec:additionalexperiments} for a discussion of this non-monotonicity in $s$); however $s > 1$ still consistently beats $s = 1$. Thus, these findings demonstrate the potential benefits of using small constants $s > 1$ in sparse JL in practice, which aligns with our theoretical results. 

% Furthermore, Figure \ref{fig:news20007} and Figure \ref{fig:enron007} show a consistent decrease in $\hat{\delta}(s, m, \epsilon)$ as a function of $s$. 

% Lastly, the plots show that sparse JL yields similar $\hat{\delta}(s, m, \epsilon)$ values (and thus behaves similarly) on the two datasets. 
% either delta v. sparsity for a fixed epsilon at many different m values, or for a fixed m at many different epsilon values? (another option is a 3d plot)

\subsection{Synthetic Datasets}\label{subsec:synthetic}
We used synthetic data to illustrate the phase transitions in our bounds on $v(m, \epsilon, \delta, s)$ in Theorem \ref{thm:mainresult} for a block sparse JL distribution. For several choices of $s, m, \epsilon, \delta$, we computed an estimate $\hat{v}(m, \epsilon, \delta, s)$ of $v(m, \epsilon, \delta, s)$  as follows. Our experiment borrowed aspects of the experimental design in \cite{FKL}. Our synthetic data consisted of binary vectors  (i.e. vectors whose entries are in $\left\{0,1\right\}$). The binary vectors were defined by a set $W$ of values exponentially spread between $0.03$ and $1$\footnote{We took $W = \left\{w \mid w^{-2} \in \left\{986, 657, 438, 292, 195, 130, 87, 58, 39, 26, 18, 12, 9, 8, 7, 6, 5, 4, 3, 2, 1\right\}\right\}$.}: for each $w \in W$, we constructed a binary vector $x^w$ where the first $1/w^2$ entries are nonzero, and computed an estimate $\hat{\delta}(s, m, \epsilon, w)$ of the failure probability  of the block sparse JL distribution on the specific vector $x^w$ (i.e. $\mathbb{P}_{A \in \mathcal{A}_{s,m,1/w^2}}[\norm{Ax^w}_2 \not\in (1\pm\epsilon) \norm{x^w}_2]$). We computed each $\hat{\delta}(s, m, \epsilon, w)$ using 100,000 samples from a block sparse JL distribution, as follows. In each sample, we independently drew a matrix $M \sim A_{s,m,1/w^2}$ and computed the ratio $\frac{\norm{Mx^w}_2}{\norm{x^w}_2}$. Then, we took $\hat{\delta}(s, m, \epsilon, w) := (\text{number of samples where } \frac{\norm{Mx^w}_2}{\norm{x^w}_2} \not\in 1 \pm \epsilon)/T$. Finally, we used the estimates $\hat{\delta}(s, m, \epsilon, w)$ to obtain the estimate  $\hat{v}(m, \epsilon, \delta, s) = \max \left\{v \in W \mid \hat{\delta}(s, m, \epsilon, w) < \delta \text{ for all } w \in W \text{ where } w \le v \right\}$.

Why does this procedure estimate $v(m, \epsilon, \delta, s)$? With enough samples, $\hat{\delta}(s, m, \epsilon, w) \rightarrow \mathbb{P}_{A \in \mathcal{A}_{s,m,1/w^2}}[\norm{Ax^w}_2 \not\in (1\pm\epsilon) \norm{x^w}_2]$.\footnote{With 100,000 samples, running our procedure twice yielded the same $\hat{v}(m, \epsilon, \delta, s)$ values both times.} As a result, if $x^w$ is a ``violating'' vector, i.e. $\hat{\delta}(s, m, \epsilon, w) \ge \delta$, then likely $\mathbb{P}_{A \in \mathcal{A}_{s,m,n}}[\norm{Ax^w}_2 \not\in (1\pm\epsilon) \norm{x^w}_2] \ge \delta$, and so $\hat{v}(m, \epsilon, \delta, s) \ge v(m, \epsilon, \delta, s)$. For the other direction, we use that in the proof of Theorem 1.5, we show that asymptotically, if a ``violating'' vector (i.e. $x$ s.t. $\mathbb{P}_{A \in \mathcal{A}_{s,m,n}}[\norm{Ax}_2 \not\in (1\pm\epsilon) \norm{x}_2] \ge \delta$) exists in $S_v$, then there's a ``violating'' vector of the form $x^w$ for some $w \le \Theta(v)$. Thus, the estimate $\hat{v}(m, \epsilon, \delta, s) = \Theta(v(m, \epsilon, \delta, s))$ as $T \rightarrow \infty$ and as precision in $W$ goes to $\infty$.

% could repeat all of the experiments 5 times and then say that the results looked the same 
\begin{figure}[!htb]
\centering
\begin{minipage}[b]{.5\textwidth}
  \centering
  \includegraphics[scale=0.45]{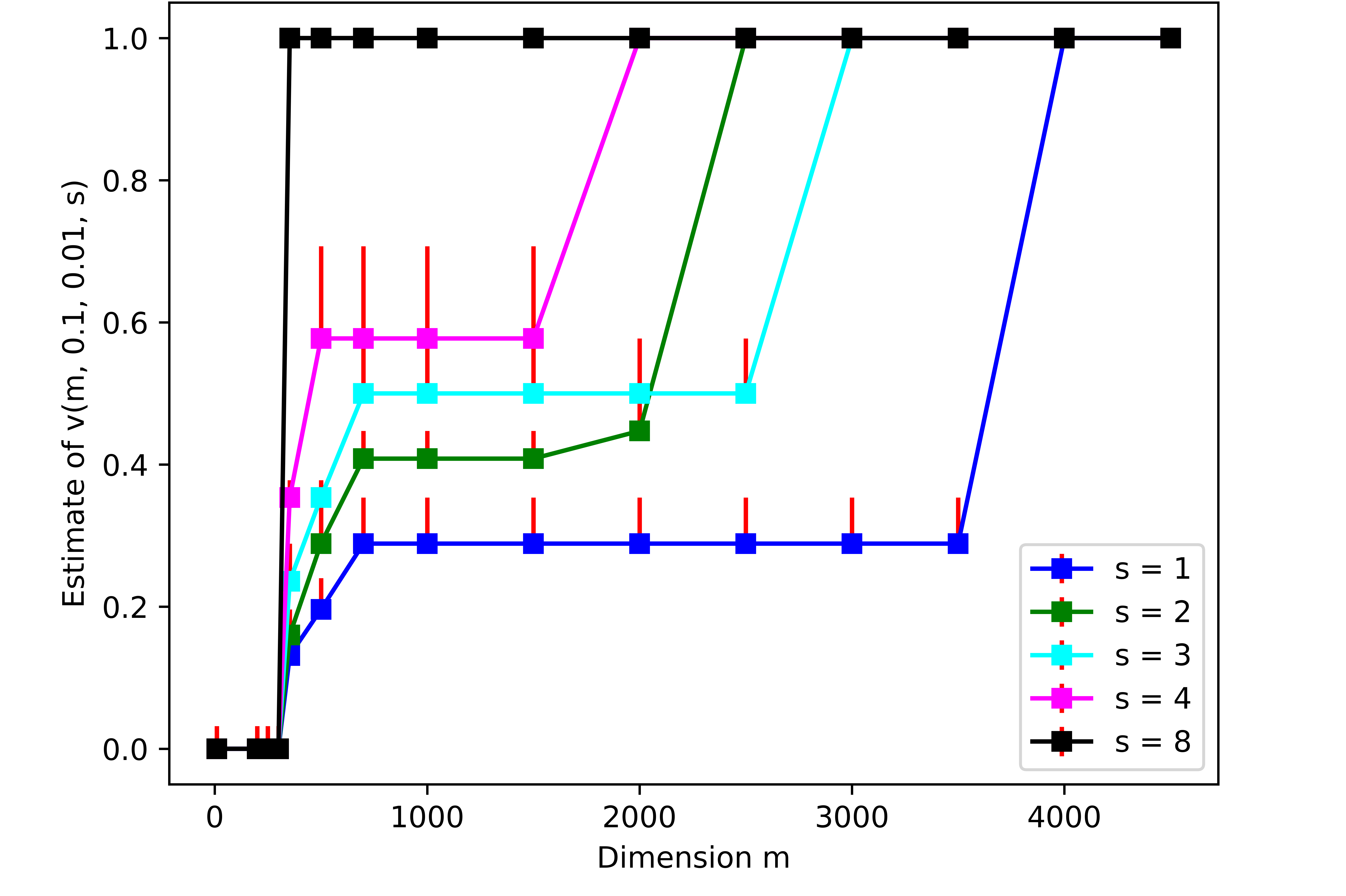}
  \caption{Phase transitions of $\hat{v}(m, 0.1, 0.01,s)$}
  \label{fig:vphase1}
\end{minipage}%
\begin{minipage}[b]{.5\textwidth}
  \centering
    \includegraphics[scale=0.45]{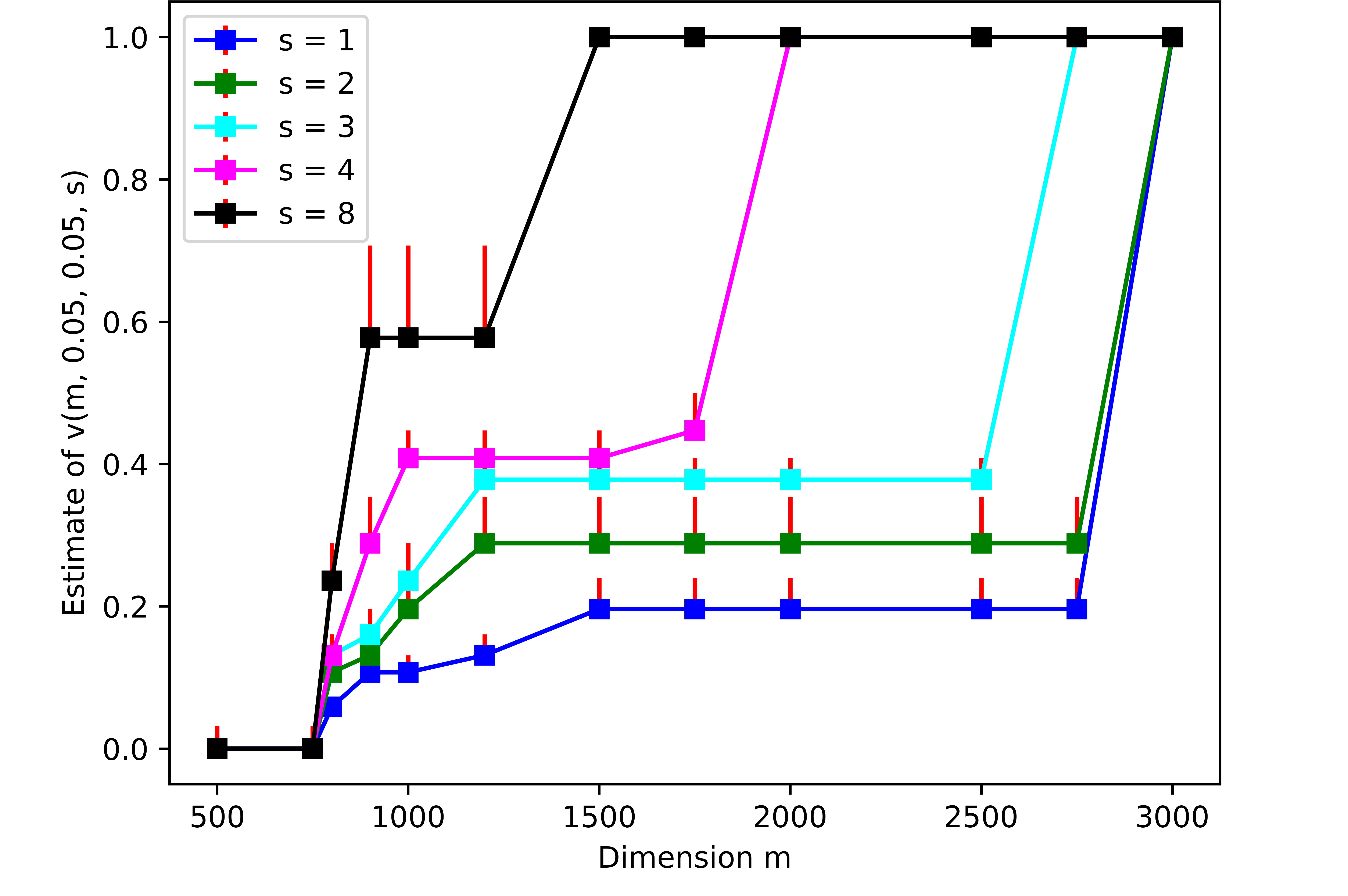}
  \caption{Phase transitions of $\hat{v}(m, 0.05, 0.05,s)$}
  \label{fig:vphase2}
\end{minipage}
\end{figure}

Figure \ref{fig:vphase1} and Figure \ref{fig:vphase2} show $\hat{v}(m, \epsilon, \delta, s)$ as a function of dimension $m$ for $s \in \left\{1, 2, 3, 4, 8\right\}$ for two settings of $\epsilon$ and $\delta$. The error-bars are based on the distance to the next highest $v$ value in $W$. 

Our first observation is that for each set of $s, \epsilon, \delta$ values considered, the curve $\hat{v}(m, \epsilon, \delta, s)$ has ``sharp'' changes as a function of $m$. More specifically, $\hat{v}(m, \epsilon, \delta, s)$ is $0$ at small $m$, then there is a phase transition to a nonzero value, then an increase to a higher value, then an interval where the value appears ``flat'', and lastly a second phase transition to $1$. The first phase transition is shared between $s$ values, but the second phase transition occurs at different dimensions $m$ (but is within a factor of $3$ between $s$ values). Here, the first phase transition likely corresponds to $\Theta(\epsilon^{-2} \ln(1/\delta))$ and the second phase transition likely corresponds to $\min\left(\epsilon^{-2} e^{\Theta(\ln(1/\delta))}, \epsilon^{-2} \ln(1/\delta) e^{\Theta\left(\ln(1/\delta) \epsilon^{-1}/s\right)}\right)$. 

Our second observation is that as $s$ increases, the ``flat'' part occurs at a higher y-coordinate. Here, the increase in the ``flat'' y-coordinate as a function of $s$ corresponds to the $\sqrt{s}$ term in $v(m, \epsilon, \delta, s)$. Technically, according to Theorem \ref{thm:mainresult}, the ``flat'' parts should be increasing in $m$ at a slow rate: the empirical ``flatness'' likely arises since $W$ is a finite set in the experiments.

Our third observation is that $s > 1$ generally outperforms $s = 1$ as Theorem \ref{thm:mainresult} suggests: that is, $s> 1$ generally attains a higher $\hat{v}(m, \epsilon, \delta, s)$ value than $s = 1$. We note at large $m$ values (where $\hat{v}(m, \epsilon, \delta, s)$ is close to $1$), lower $s$ settings sometimes attain a higher $\hat{v}(m, \epsilon, \delta, s)$ than higher $s$ settings (e.g. the second phase transition doesn't quite occur in decreasing order of $s$ in Figure \ref{fig:vphase1}): see Appendix \ref{sec:additionalexperiments} for a discussion of this non-monotonicity in $s$.\footnote{In Appendix \ref{sec:additionalexperiments}, we also show more examples where at large $m$ values, lower $s$ settings attain a higher $\hat{v}(m, \epsilon, \delta, s)$ than higher $s$ settings.} Nonetheless, in practice, it's unlikely to select such a large dimension $m$, since the $\ell_{\infty}$-to-$\ell_2$ guarantees of smaller $m$ are likely sufficient. Hence, a greater sparsity generally leads to a better $\hat{v}(m, \epsilon, \delta, s)$ value, thus aligning with our theoretical findings. 

\bibliographystyle{plain} 
\bibliography{bib.bib}

\newpage 
\appendix

\begin{center}
\Large{\textbf{APPENDIX}}
\end{center}

\normalsize{}

In Appendix \ref{sec:mainresultsstatement}, we prove our corollary regarding dimension-sparsity tradeoffs and discuss some of the subtleties of Theorem \ref{thm:mainresult}. In Appendix \ref{sec:looseHW}, we show that the Hanson-Wright bound is too loose to prove Theorem \ref{thm:mainresult}. In Appendix \ref{sec:usefulmomentboundsproofs}, we state and prove useful moment bounds that we use throughout the analysis. In Appendix \ref{sec:analyzerow}, we prove our moment bounds for $Z_r(x_1, \ldots, x_n)$ in Lemma \ref{rowbound} and Lemma \ref{lowerboundrow}. In Appendix \ref{sec:combinerows}, we prove our moment bounds for $R(x_1, \ldots, x_n)$ in Lemma \ref{upper} and Lemma \ref{lower}. In Appendix \ref{sec:auxupper}, we prove auxiliary lemmas needed in the proof of Lemma \ref{upper}. In Appendix \ref{sec:auxlower}, we prove auxiliary lemmas needed in the proof of Lemma \ref{lower}. In Appendix \ref{sec:functionbounding}, we prove our simplified moment bounds for $R(x_1, \ldots, x_n)$ in Lemma \ref{upperfinal} and Lemma \ref{lowerupperfinal}. In Appendix \ref{sec:additionalexperiments}, we provide additional experimental results on real-world and synthetic datasets as well as additional discussion. 

\section{Discussion of Theoretical Results}\label{sec:mainresultsstatement}

We discuss some of the subtleties of Theorem \ref{thm:mainresult}. When $m \ge \min(2 \epsilon^{-2} e^{p}, \epsilon^{-2} p e^{\Theta(\max(1, p \epsilon^{-1}/s))})$, where $p = \ln(1/\delta)$, we show that $v(m, \epsilon, \delta, s) = 1$, which means that the norm-preserving condition holds on the full space. This generalizes Cohen's bound \cite{Cohen} to a slightly more general family of sparse JL distributions, as we discuss below. When $m \le \Theta(\epsilon^{-2} \ln (1/\delta))$, we show that $v(m, \epsilon, \delta, s) = 0$. For the  remaining regimes, $\sqrt{\epsilon s} \sqrt{\ln(\frac{m \epsilon^2}{p})} /\sqrt{p}$ and
$\sqrt{\epsilon s}\min\left( \ln(\frac{m \epsilon}{p})/p, \sqrt{\ln(\frac{m\epsilon^2}{p})}/\sqrt{p}\right)$, our upper and lower bounds on $v(m, \epsilon, \delta, s)$ match up to constant factors. 

In terms of the boundaries between regimes, we emphasize that in Theorem \ref{thm:mainresult}, the function $f'(m, \epsilon, \delta, s)$ may not be defined for certain intervals between the boundaries of regimes, since there may be different absolute constants in different boundaries. More specifically, these intervals are $C_1 \epsilon^{-2} p \le m \le C_2 \epsilon^{-2}p$, $\epsilon^{-2} e^{C_1p} \le m \le 2 \epsilon^{-2} e^p$, and $s \cdot e^{C_1 \max(1, p\epsilon^{-1}/s)} \le m \le s \cdot e^{C_2 \max(1, p\epsilon^{-1}/s)}$. These gaps arise because the boundaries between the regimes on our upper and lower bounds on $v(m, \epsilon, \delta, s)$ can have different absolute constants, so we don't have precise control on $v(m, \epsilon, \delta, s)$ in these gaps. Nonetheless, the gaps only span a constant factor range on the exponent in the dimension $m$. 

We now state the dimension-sparsity tradeoffs that follow from our bounds:
\begin{corollary}
\label{dimsparsitylower}
Suppose that $\epsilon$ and $\delta$ are sufficiently small and $s \le m/e$. 
If $\mathcal{A}_{s,m,n}$ is any sparse JL distribution, then $v(m, \epsilon, \delta, s) =  1$ when $m \ge \min\left(2 \epsilon^{-2} / \delta, \epsilon^{-2} \ln(1/\delta) e^{\Theta(\max(1, \ln(1/\delta) \epsilon^{-1}/s))} \right)$. If $\mathcal{A}_{s,m,n}$ is a uniform sparse JL distribution, then $v(m, \epsilon, \delta, s) \le 1/2$ when $m \le \min\left(\epsilon^{-2} e^{\Theta(\ln(1/\delta))}, \epsilon^{-2} \ln(1/\delta) e^{\Theta(\max(1, \ln(1/\delta) \epsilon^{-1}/s))} \right)$, apart from a constant-factor interval $C_1 \epsilon^{-2} \ln(1/\delta) \le m \le C_2 \epsilon^{-2} \ln(1/\delta)$ where we do not have a bound on the behavior of sparse JL. 
\end{corollary}
\begin{proof}[Proof of Corollary~$\ref{dimsparsitylower}$]
The first statement follows from the fact the lower bound in Theorem \ref{thm:mainresult} holds for any sparse JL distribution. For the upper bound, we also use Theorem \ref{thm:mainresult}. Let's set $C_v \sqrt{\epsilon s} \frac{\sqrt{\ln(m \epsilon^2}}{\sqrt{p}} = \frac{1}{2}$, where $C_v$ is the implicit constant in the upper bound. This solves to 
$m = \epsilon^{-2} p e^{\frac{C_L p \epsilon^{-1}}{s}}$ for some constant $C_L$ as desired. We also have the condition that $m \le \epsilon^{-2} e^{\Theta(\ln(1/\delta))}$ for this regime to be reached. We can obtain the $\max$ with $1$ on the exponent, by using that $v(m, \epsilon, \delta, s) = 0$ when $m \le \Theta(\epsilon^{-2}\ln(1/\delta))$. To avoid having a gap when $m = s \cdot e^{\Theta(\max(1, \ln(1/\delta) \epsilon^{-1}/s))}$, we implicitly use that our lower bound actually doesn't have a gap between these regimes (though there may be a gap in the boundary between the lower bound and upper bound). Thus, we only have to keep the gap $C_1 \epsilon^{-2} \ln(1/\delta) \le m \le C_2 \epsilon^{-2} \ln(1/\delta)$ where we do not have a lower bound.
\end{proof}
\noindent Notice that the  upper and lower bounds in Corollary~$\ref{dimsparsitylower}$ also match up to constant factors on the exponent in the dimension $m$. 

\section{Hanson-Wright is too loose even for $s = 1$}
\label{sec:looseHW}

Though Cohen, Jayram, and Nelson \cite{NelsonNotes} also view $R(x_1, \ldots, x_n)$ as a quadratic form, we show that their approach is not sufficiently precise for our setting. They upper bound the moments of $R(x_1, \ldots, x_n)$ by the gaussian case through considering:
\[\tilde{R}(x_1, \ldots, x_n) = \frac{1}{s} \sum_{r=1}^m \sum_{i \neq j} \eta_{r,i} \eta_{r,j} g_{r,i} g_{r,j} x_i x_j\] where the $g_{r,i}$ are i.i.d standard gaussians. They use the fact Rademachers are subgaussian to conclude that $\norm{R(x_1, \ldots, x_n)}_q \le \norm{\tilde{R}(x_1, \ldots, x_n)}_q$. In order to obtain upper bounds on $\norm{\tilde{R}(x_1, \ldots, x_n)}_q$, they use the Hanson-Wright bound, a tight bound on moments of gaussian quadratic forms. However, we need different technical tools for two reasons.
\begin{enumerate}
\item First, in order to upper bound $v(m, \epsilon, \delta, s)$, we need to \textit{lower bound} $\norm{|R(x_1, \ldots, x_n)}_q$, and thus cannot simply consider $\norm{\tilde{R}(x_1, \ldots, x_n)}_q$. 
\item Second, even to lower bound $v(m, \epsilon, \delta, s)$, using $\norm{\tilde{R}(x_1, \ldots, x_n)}_q$ as a upper bound for $\norm{R(x_1, \ldots, x_n)}_q$ is not sufficiently strong. Below, we give a counter-example, i.e. a vector $x$, where $\norm{\tilde{R}(x_1, \ldots, x_n)}_q$ is too large to recover a tight lower bound.
\end{enumerate}
Thus, we cannot use the Hanson-Wright bound in this setting, and need to come up with a better bound on $\norm{R(x_1, \ldots, x_n)}_q$ that does not implicitly replace Rademachers by gaussians. The second point is similar in flavor to the conceptual point made in \cite{Me}, where a sign-consistent variant of sparse JL was analyzed using an \textit{upper} bound for Rademacher quadratic forms. However, the bound in \cite{Me} also turns out to be loose in this setting and also can't be used to obtain either a sufficiently tight upper bound or a lower bound for $R(x_1, \ldots, x_n)$.

We now show point (2): that the Hanson-Wright bound is not sufficiently strong to obtain a lower bound $v(m, \epsilon, \delta, s)$. 
We consider $\tilde{R}(x_1, \ldots, x_n) = \frac{1}{s} \sum_{r=1}^m \sum_{i \neq j} \eta_{r,i} \eta_{r,j} g_{r,i} g_{r,j} x_i x_j$, as above, where the $g_{r,i}$ are i.i.d standard gaussians. We consider $p$ equal to $\ln(1/\delta)$ rounded up to the nearest even integer, and we consider a vector of the form $[v, \ldots, v, 0, \ldots, 0]$ where $\frac{1}{v^2}$ is an integer and $v \ge 0$. We show $\norm{\tilde{R}(v, \ldots, v, 0, \ldots, 0)}_p \gtrsim \omega(\epsilon)$ for a certain $v$ value, where we know it to be true that $\norm{R(v, \ldots, v, 0, \ldots, 0)}_p \lesssim \epsilon$. 

Let's consider a vector $[v, \ldots, v, 0, \ldots, 0]$ where $\frac{1}{v^2}$ is an integer and $v \ge 0$. We apply the Hanson-Wright bound (which is tight for gaussians) to obtain:
\begin{align*}
    \norm{\tilde{R}(v, \ldots, v, 0, \ldots, 0)}_p 
    &\gtrsim  pv^2 \norm{\sup_{\norm{x}_2, \norm{y}_2 \le 1} \sum_{r=1}^m \sum_{1 \le i \neq j \le N} \eta_{r,i} \eta_{r,j} x_{r,i} y_{r,j}}_p  \\
    &\geq pv^2 \norm{\sup_{\norm{x}_2, \norm{y}_2 \le 1} \sum_{1 \le i \neq j \le N} \eta_{1,i} \eta_{1,j} x_{i} y_{j}}_p.
\end{align*}
Let $M = \sum_{i=1}^N \eta_{1,i}$. Let $S \subseteq [N]$ be the set of indices where $\eta_{1,i} = 1$. We can set the vector to $x_i = y_i = \frac{1}{\sqrt{M}}$ for all $i \in S$ and $0$ elsewhere. This gives us: 
\[ \norm{\sup_{\norm{x}_2, \norm{y}_2 \le 1} \sum_{1 \le i \neq j \le N} \eta_{1,i} \eta_{1,j} x_{i} y_{j}}_p \geq  \norm{M - 1}_p = \norm{\sum_{i=1}^N \eta_{1,i} - 1}_p \gtrsim \norm{I_{\sum_{i=1}^N \eta_{1,i} \ge 2} \sum_{i=1}^N \eta_{1,i}}_p.  \]
We can expand out this moment to obtain:
\begin{align*}
\mathbb{E}\left[\left(I_{\sum_{i=1}^N \eta_{1,i} \ge 2} \sum_{i=1}^N \eta_{1,i}\right)^p\right] &\geq C^p \sum_{M=2}^p {N \choose M} M^p \left(\frac{s}{m}\right)^M \left(1 - \frac{s}{m}\right)^M \\
&\geq C^p \sum_{M=2}^p \left(\frac{N}{p}\right)^M \left(\frac{p}{M}\right)^M M^p \left(\frac{s}{m}\right)^M \left(1 - \frac{s}{m}\right)^M \\
&= C^p \sum_{M=2}^p \left(\frac{s}{pmv^2}\right)^M M^p \left(\frac{p}{M}\right)^M \left(1 - \frac{s}{m}\right)^M.
\end{align*}
Since $M \le p$, we know that $\left(\frac{p}{M}\right)^M \ge 1$. Moreover, as long as $p \ge \frac{se}{mv^2}$, we know that $\left(1 - \frac{s}{m}\right)^{M/p} \ge \left(1 - \frac{s}{m}\right)^{N/p} \ge \left(1 - \frac{s}{m}\right)^{\frac{m}{s}} \ge 0.3$. Thus we obtain a bound of
\[D^p \sum_{M=2}^p M^p \left(\frac{s}{pmv^2}\right)^M. \]
If $2 \le \frac{p}{\ln(pmv^2/s)} \le p$ (which can be written as $1 \le \ln(pmv^2/s) \le \frac{p}{2}$), then we know that:
\[pv^2 \norm{\sup_{\norm{x}_2, \norm{y}_2 \le 1} \sum_{1 \le i \neq j \le N} \eta_{1,i} \eta_{1,j} x_{i} y_{j}}_p \gtrsim \frac{p^2v^2}{\ln(pmv^2/s)}.\] 
We show that when $s = 1$, the bound $v = \sqrt{\epsilon} \frac{\ln(\frac{m \epsilon}{p})}{p}$ will produce $\norm{R(v, \ldots, v, 0, \ldots, 0)}_p \gtrsim \omega(\epsilon)$. At this $v$ value, we know that: 
\[1 \le \frac{pmv^2}{e} = \ln^2(\frac{m \epsilon}{p}) \frac{m \epsilon}{pe}. \] If we have that $\ln(\frac{m \epsilon}{p}) \le \sqrt{p}$, then we know that $v \le \frac{1}{\sqrt{p}}$ and $\ln(pmv^2) \le \frac{p}{2}$. However, the bound 
\[\frac{p^2v^2}{\ln(pmv^2)} \gtrsim\frac{\epsilon \ln^2(m \epsilon/p)}{\ln(\frac{m \epsilon}{p})} \geq \epsilon \ln(\frac{m \epsilon}{p}) = \omega(\epsilon). \]   

\section{Useful Moment Bounds}
\label{sec:usefulmomentboundsproofs}
The key quadratic form bound for Rademachers that we use is:
\begin{lemma}
\label{usefulquadformbound}
Let $T$ be an even integer, $\left\{\sigma_i\right\}_{1 \le i \le n}$ be independent Rademachers, and $(Y_{i,j})_{1 \le i, j \le n}$ be a $n \times n$ symmetric, nonnegative random matrix with zero diagonal (i.e. $Y_{i,i} = 0$) such that $\left\{Y_{i,j}\right\}_{1 \le i,j \le n}$ is independent from $\left\{\sigma_i\right\}_{1 \le i \le n}$. If $W_i = \sqrt{\sum_{1 \le j \le n} Y_{i,j}^2}$, then: 
\[\norm{\sum_{1 \le i, j \le n} Y_{i,j} \sigma_i \sigma_j}_T \simeq \norm{\sup_{\norm{b}_2, \norm{c}_2 \le \sqrt{T}, \norm{b}_{\infty}, \norm{c}_{\infty} \le 1} \sum_{1 \le i, j \le n} Y_{i,j} b_i c_j}_T  + \norm{\sum_{1 \le i \le T} W_{(i)} + \sqrt{T} \sqrt{\sum_{T < i \le n} W_{(i)}^2}}_T \] where $W_{(1)} \ge W_{(2)} \ge \ldots \ge \ldots W_{(n)}$ is a permutation of $W_1, \ldots, W_n$. 
\end{lemma}
We derive Lemma \ref{usefulquadformbound} from Lata{\l}a's bound on Rademacher quadratic forms \cite{LatalaChaos}. In fact, Lata{\l}a shows moment bounds for much more general quadratic forms, but for the application to JL, we only need the following bound in the special case of Rademachers: 
\begin{lemma}[\cite{LatalaChaos}]
\label{quadformbound}
Let $T$ be an even natural number. Let $\sigma_1, \ldots, \sigma_n$ be independent Rademachers and let $(a_{i,j})$ a symmetric matrix with zero diagonal. Then:
\[\norm{\sum_{1 \le i, j\le n} a_{i,j} \sigma_i \sigma_j}_T \simeq \left(\sup_{\norm{b}_2, \norm{c}_2 \le \sqrt{T}, \norm{b}_{\infty}, \norm{c}_{\infty} \le 1} \sum_{1 \le i, j \le n} a_{i,j} b_i c_j\right)  + \sum_{1 \le i \le T} A_{(i)} + \sqrt{T} \sqrt{\sum_{T < i \le n} (A_{(i)})^2} \] where $A_i = \sqrt{\sum_{1 \le j \le n} a_{i,j}^2}$ and $A_{(1)} \ge A_{(2)} \ldots \ge \ldots A_{(n)}$ is a permutation of $A_1, \ldots, A_n$.  
\end{lemma}
To prove Lemma~$\ref{usefulquadformbound}$, we apply Lemma~$\ref{quadformbound}$ to the case where the $a_{i,j}$ are themselves random variables:
\begin{proof}[Proof of Lemma~$\ref{usefulquadformbound}$]
Let $Q = \sum_{1 \le i \neq j \le n} Y_{i,j} \sigma_i \sigma_j$. Applying Lemma~$\ref{quadformbound}$, we have that:
\begin{align*}
\left(\mathbb{E}_{Y, \sigma} [Q^T]\right)^{1/T} &=\ \left(\mathbb{E}_Y \mathbb{E}_{\sigma} [Q^T]\right)^{1/T} \\
&= \left(\mathbb{E}_Y \left[\norm{\sum_{1 \le i \neq j \le n} Y_{i,j} \sigma_i \sigma_j}^T_T\right]\right)^{1/T} \\
&\simeq \norm{\sup_{\norm{b}_2, \norm{c}_2 \le \sqrt{T}, \norm{b}_{\infty}, \norm{c}_{\infty} \le 1} \sum_{1 \le i \neq j \le n} Y_{i,j} b_i c_j  +\sum_{1 \le i \le T} W_{(i)} + \sqrt{T} \sqrt{\sum_{T < i \le n} W_{(i)}^2}}_T \\
&\simeq \norm{\sup_{\norm{b}_2, \norm{c}_2 \le \sqrt{T}, \norm{b}_{\infty}, \norm{c}_{\infty} \le 1} \sum_{1 \le i \neq j \le n} Y_{i,j} b_i c_j}_T  + \norm{\sum_{1 \le i \le T} W_{(i)} + \sqrt{T} \sqrt{\sum_{T < i \le n} W_{(i)}^2}}_T
\end{align*}
where the last line follows from the fact that the the $Y_{i,j}$ are nonnegative, so each term is nonnegative, so the triangle inequality results in at most a factor of $2$ of gain. 
\end{proof}

Now, we consider linear forms of symmetric random variables. Theoretically, moments of these forms can be derived from Theorem 2 in \cite{LatalaMoments} (a tight bound on moments of weighted sums of symmetric random variables). However, reducing the tight bound to the form that we want would require some simplifications. Instead, we give a direct proof of a weaker bound that is sufficiently tight for our setting.
\begin{proposition}
\label{simpler}
Suppose that $T \ge 1$ is an integer. Suppose that $Y_1, Y_2, \ldots, Y_n$ are i.i.d symmetric random variables and suppose that $x = [x_1, \ldots, x_n]$ satisfies $\norm{x}_2 \le 1$ and $\norm{x}_{\infty} \le v$. Then, we have that 
\[\norm{\sum_{i} Y_i x_i}_{2T} \lesssim v \left(\sup_{1 \le t \le T} \frac{T}{t} \left(\frac{1}{Tv^2}\right)^{\frac{1}{2t}} \norm{Y_i}_{2t}\right)  .\] 
\end{proposition}
\begin{proof}[Proof of Proposition~$\ref{simpler}$]
Let $k = 2 v \left(\sup_{1 \le t \le T} \frac{T}{t} \left(\frac{1}{Tv^2}\right)^{1/(2t)} \norm{Y_i}_{2t}\right)  $. Observe that 
\begin{align*}
 \mathbb{E}[\left(\frac{\sum_{i} Y_i x_i}{k^2}\right)^{2T}]
&= \sum_{d_1 + d_2 + \ldots + d_n = T, d_i \le T} \frac{2T!}{2d_1!\ldots 2d_n!} \prod_{i=1}^n  \mathbb{E}\left[\left(\frac{Y_ix_i}{k}\right)^{2d_i}\right]\\
&\le C^T \sum_{d_1 + d_2 + \ldots + d_n = T, d_i \le T} \frac{(2T)^{2T}}{(2d_1)^{2d_1}\ldots (2d_n)^{2d_n}} \prod_{i=1}^n  \mathbb{E}\left[\left(\frac{Y_ix_i}{k}\right)^{2d_i}\right]\\
&\le C^T \prod_{i=1}^n \sum_{0 \le d_i \le T} \frac{(2T)^{2d_i}}{(2d_i)^{2d_i}} \mathbb{E}\left[\left(\frac{Y_ix_i}{k}\right)^{2d_i}\right]\\
&= C^T \prod_{i=1}^n \left(1 +\sum_{1 \le d_i \le T}  \left(\frac{T x_i \norm{Y_i}_{2d_i} v}{vd_ik}\right)^{2d_i}\right) 
\end{align*}
Now, we use the fact that $|x_i| \le v$ and the condition on $k$ to obtain that this is bounded by 
\[C^T \prod_{i=1}^n \left(1 + \frac{x_i^2}{v^2} \sum_{1 \le d_i \le T}  \left(\frac{T v \norm{Y_i}_{2d_i}}{d_i k}\right)^{2d_i}\right) \le C^T \prod_{i=1}^n \left(1 + Tx_i^2 \right) \le C^T \prod_{i=1}^n e^{Tx_i^2} \le C^T e^T.\]
\end{proof}

We now bound moments of squares of linear forms with a zero diagonal, i.e. $\sum_{i \neq j} Y_i Y_j  x_i x_j$. This structure of random variable theoretically falls under the scope of Lemma \ref{usefulquadformbound}. However, as mentioned in Section 2.1, the first term of \ref{usefulquadformbound}, which is an operator-norm-like term for an asymmetric random matrix in this setting, becomes intractable to manage. We give an alternate (weaker) upper bound that is both tractable to analyze and sufficiently tight for our setting. Our proof of this bound is similar to our proof of Proposition \ref{simpler} presented above. Since random variables with a zero diagonal are common in the JL literature \cite{KN12, Original, NN13}, we believe this moment bound could be of broader use.
\begin{lemma}
\label{generalsymmetric}
Suppose that $Y_1, Y_2, \ldots, Y_n$ are i.i.d symmetric random variables and suppose that $x = [x_1, \ldots, x_n]$ satisfies $\norm{x}_2 = 1$ and $\norm{x}_{\infty} \le v$. Let $T$ be an even natural number. Then, we have that 
\[\norm{\sum_{i \neq j} Y_i Y_j  x_i x_j}_T \lesssim v^2 \left(\sup_{1 \le t \le T/2} \frac{T^2}{t^2} \left(\frac{1}{Tv^2}\right)^{1/t} \norm{Y_i}^2_{2t}\right)  .\] 
\end{lemma}
\allowdisplaybreaks
\begin{proof}[Proof of Lemma \ref{generalsymmetric}]
Let $k =2  v \left(\sup_{1 \le t \le T/2} \frac{T}{t} \left(\frac{1}{Tv^2}\right)^{1/(2t)} \norm{Y_i}_{2t}\right)  $. 
Observe that 
\begin{align*}
 \mathbb{E}\left[\left(\frac{\sum_{i \neq j} Y_i Y_j x_i x_j}{k^2}\right)^T\right]
&\le \sum_{d_1 + d_2 + \ldots + d_n = T, d_i \le T/2} \frac{2T!}{2d_1!\ldots 2d_n!} \prod_{i=1}^n  \mathbb{E}\left[\left(\frac{Y_ix_i}{k}\right)^{2d_i}\right]\\
&\le C^T \sum_{d_1 + d_2 + \ldots + d_n = T, d_i \le T/2} \frac{(2T)^{2T}}{(2d_1)^{2d_1}\ldots (2d_n)^{2d_n}} \prod_{i=1}^n  \mathbb{E}\left[\left(\frac{Y_ix_i}{k}\right)^{2d_i}\right]\\
&\le C^T \prod_{i=1}^n \sum_{0 \le d_i \le T/2} \frac{(2T)^{2d_i}}{(2d_i)^{2d_i}} \mathbb{E}\left[\left(\frac{Y_ix_i}{k}\right)^{2d_i}\right]\\
&= C^T \prod_{i=1}^n \left(1 +\sum_{1 \le d_i \le T/2}  \left(\frac{T x_i \norm{Y_i}_{2d_i} v}{vd_ik}\right)^{2d_i}\right) 
\end{align*}
Now, we use the fact that $|x_i| \le v$ and the condition on $k$ to obtain that this is bounded by 
\[C^T \prod_{i=1}^n \left(1 + \frac{x_i^2}{v^2} \sum_{1 \le d_i \le T/2}  \left(\frac{T v \norm{Y_i}_{2d_i}}{d_i k}\right)^{2d_i}\right) \le C^T \prod_{i=1}^n \left(1 + Tx_i^2 \right) \le C^T \prod_{i=1}^n e^{Tx_i^2} \le C^T e^T.\]
\end{proof}

Lata{\l}a \cite{LatalaMoments} gives the following nice bound on sums of i.i.d symmetric random variables that we use for combining bounds on rows $Z_r(x_1, \ldots, x_n)$ in Lemma \ref{upper}. 
\begin{lemma}[\cite{LatalaMoments}]
\label{sumiid}
Suppose that $q$ is an even natural number. Suppose that $Y_1, \ldots, Y_n$ are i.i.d symmetric random variables. Then:
\[\norm{\sum_{i=1}^n Y_i}_{q} \lesssim \sup_{2 \le T \le q} \frac{q}{T} \left(\frac{n}{q}\right)^{1/T} \norm{Y_i}_{T}. \]
\end{lemma}

We give a general lower bound on moments of certain (potentially correlated) sums of identically distributed random variables, that we use in proving Lemma \ref{lower}. 
\begin{proposition}
\label{combinerowseasy}
Let $Y_1, \ldots, Y_n$ be identically distributed (but not necessarily independent) random variables, such that the joint distribution is a symmetric function of $Y_1, \ldots, Y_n$ and for any integers $d_1, \ldots d_n \ge 0$, it is true that $\mathbb{E}[\prod_{1 \le i \le n} Y_i^{d_i}] \ge 0$. For any natural number $q$ and natural number $T$ that divides $q$, it is true that
\[\norm{\sum_{i=1}^n Y_i}_q \ge T \left(\frac{n}{q}\right)^{T/q} \norm{Y_1Y_2\ldots Y_T}^{1/T}_{q/T}  \]
\end{proposition}
\begin{proof}[Proof of Proposition \ref{combinerowseasy}]
The proof follows from expanding $\mathbb{E}[\left(\sum_{i=1}^n Y_i \right)^{q}]$ and using the fact that $\mathbb{E}[\prod_{1 \le i \le n} Y_i^{d_i}] \ge 0$ so that we can restrict to a subset of the terms. By the symmetry of the joint distribution, we know that for $1 \le r_1 \neq r_2 \neq r_T \le n$, we know that  $\mathbb{E}[Y_{r_1}^{q/T} \ldots Y_{r_T}^{q/T}] = \mathbb{E}[Y_1^{q/T} \ldots Y_T^{q/T}]$. The number of terms of the form $\mathbb{E}[Y_{r_1}^{q/T} \ldots Y_{r_T}^{q/T}]$ in $\mathbb{E}[\left(\sum_{i=1}^n Y_i \right)^{q}]$ is: 
\[{n \choose T} {q \choose {q/T, q/T, \ldots, q/T}} \geq C^q \left(\frac{n}{T}\right)^T \frac{q!}{\left((q/T)! \right)^{T}} \geq C^q \left(\frac{n}{T}\right)^T T^q  \geq C_2^q \left(\frac{n}{q}\right)^T \left(\frac{q}{T}\right)^T T^q \geq C'^q \left(\frac{n}{q}\right)^T T^q.\]This implies that 
\[\mathbb{E}\left[\left(\sum_{i=1}^n Y_i \right)^{q}\right] \geq C'^q \left(\frac{n}{q}\right)^T T^q \mathbb{E}\left[Y_1^{q/T} \ldots Y_T^{q/T}\right]\] and the statement follows from taking $1/q$th powers.
\end{proof}

We prove a lemma involving the Paley-Zygmund inequality applied to $p$th moments, that we use implicitly in the proof of the upper bound in Theorem \ref{thm:mainresult}. 
\begin{lemma}
\label{ourpz}
Suppose that $K > 0$ and $Z$ is a nonnegative random variable, such that $\norm{Z}_q \ge 2K$ and $\norm{Z}_{2q}$ is finite. Then, 
\[\mathbb{P}[Z > K] \ge 0.25 \left(\frac{\norm{Z}_q}{\norm{Z}_{2q}}\right)^{2q}. \]
\end{lemma} 
We use the Paley-Zygmund inequality, which says the following:
\begin{lemma}[Paley-Zygmund]
\label{pz}
Suppose that $Z$ is a nonnegative random variable with finite variance. Then, 
\[\mathbb{P}[Z > 2^{-1} \mathbb{E}[Z]] \ge \frac{\mathbb{E}[Z]^2}{4\mathbb{E}[Z^2]}. \]
\end{lemma}
\begin{proof}[Proof of Lemma~$\ref{ourpz}$]
We apply Lemma~$\ref{pz}$ to $Z^p$ to obtain that: 
\[ \mathbb{P}[Z^p > 2^{-1} \mathbb{E}[Z^p]] \ge 0.25 \frac{\mathbb{E}[Z^p]^{2}}{\mathbb{E}[Z^{2p}]} = 0.25 \left(\frac{\norm{Z}_p}{\norm{Z}_{2p}}\right)^{2p}. \]
If $\norm{Z}_p \ge 2K$, then we know that 
\[\mathbb{P}[Z > K] = \mathbb{P}[Z^p > K^p] \ge \mathbb{P}[Z^p > 2^{-p} \mathbb{E}[Z^p]] \ge \mathbb{P}[Z^p > 2^{-1} \mathbb{E}[Z^p]]\] and then we can apply the above result. 
\end{proof}

\section{Proofs of Lemma \ref{rowbound} and Lemma \ref{lowerboundrow}}
\label{sec:analyzerow}
We analyze the moments of $Z_r(x_1, \ldots, x_n)$, proving Lemma \ref{lowerboundrow} and Lemma \ref{rowbound}. Our lower bound in Lemma \ref{lowerboundrow} holds for $\norm{Z_r(v, \ldots, v, 0, \ldots, 0}_q$ as well as $\norm{Z_r(v, \ldots, v, 0, \ldots, 0)I_{\sum_{i=1}^{1/v^2} \eta_{r,i} = 2}}_T$ (for technical reasons discussed in \Cref{sec:combinerows}). Our upper bound in Lemma \ref{rowbound} holds for $\norm{Z_r(x_1, \ldots, x_n)}_q$. 
In \Cref{sec:analyzerowlower}, we prove Lemma~$\ref{lowerboundrow}$. In \Cref{sec:analyzerowupper}, we prove Lemma~$\ref{rowbound}$.

\subsection{Proof of Lemma~$\ref{lowerboundrow}$}
\label{sec:analyzerowlower}
The key ingredient of the proof is Lemma~$\ref{usefulquadformbound}$ (for Rademacher quadratic forms). We can view $Z_r(v, \ldots, v, 0, \ldots, 0)$ as the following quadratic form:
\[Z_r(v, \ldots, v, 0, \ldots, 0) = v^2 \sum_{1 \le i \neq j \le N} \eta_{r,i} \eta_{r,j} \sigma_{r,i} \sigma_{r,j},\] where $N = \frac{1}{v^2}$. Since the support of $\eta_{r,i}$ is $\left\{0,1\right\}$ and due to symmetry of this random variable, it is tractable to analyze the expressions in Lemma~$\ref{usefulquadformbound}$.
\begin{proof}[Proof of Lemma~$\ref{lowerboundrow}$]
First, we handle the case of $T = 2$:
\begin{align*}
\mathbb{E}[Z_r(v, \ldots, v, 0, \ldots, 0)]^2 &= v^4 \mathbb{E}\left[\left(\sum_{i \neq j} \eta_{r,i} \eta_{r,j} \sigma_{r,i} \sigma_{r,j}\right)^2\right] \\
&= 2v^4  \mathbb{E}\left[\sum_{i \neq j} \eta_{r,i} \eta_{r,j} \right] = 2v^4 \left(\frac{s}{m}\right)^2 N(N-1) \geq \frac{v^4N^2s^2}{m^2} = \frac{s^2}{m^2}
\end{align*}
as desired.

Now we consider $T > 2$, and we prove a bound on $\norm{Z_1(v, \ldots, v,  0, \ldots, 0)}_T$. We see that $\norm{Z_1(v, \ldots, v,  0, \ldots, 0)}_T  = v^2\norm{\sum_{i \neq j} \eta_{1,i} \eta_{1,j} \sigma_{1,i} \sigma_{1,j}}_T$. Fix $1 \le M \le \min(N,T)$. We use Lemma $\ref{usefulquadformbound}$ with $Y_{i,j} = \eta_{1,i} \eta_{1,j} I_{M = \sum_{k=1}^{N} \eta_{1,k}}$ to compute $\norm{\sum_{i \neq j} \eta_{1,i} \eta_{1,j} \sigma_{1,i} \sigma_{1,j} I_{M = \sum_{k=1}^{N} \eta_{1,k}}}_T$. We will then aggregate over $2 \le M \le T$ and not even count $M = 1$ or $T < M \le N$. We only use the operator-norm-like term in Lemma \ref{usefulquadformbound}. Observe that 
\[I_{M = \sum_{k=1}^{N} \eta_{1,k}}\sup_{\norm{b}_2, \norm{c}_2 \le \sqrt{T}, \norm{b}_{\infty}, \norm{c}_{\infty} \le 1} \sum_{i \neq j} \eta_{1,i} \eta_{1,j} b_i c_j\] is equal to
\[I_{M = \sum_{k=1}^{N} \eta_{1,k}} \sup_{\norm{b}_2, \norm{c}_2 \le \sqrt{T}, \norm{b}_{\infty}, \norm{c}_{\infty} \le 1} \sum_{i, j \mid \eta_{1,i} = 1, \eta_{1,j} = 1} b_i c_j \ge I_{M = \sum_{k=1}^{N} \eta_{1,k}} M(M-1),\] where we set $b_i = 1$ on all $i$ such that $\eta_{1,i} = 1$ and $c_j =1$ on all $j$ such that $\eta_{1,j} = 1$. 

Since the events $M = \sum_{k=1}^{N} \eta_{1,k}$ are disjoint across different $M$ values, we know that:
\begin{align*}
\norm{\sum_{i \neq j} \eta_{1,i} \eta_{1,j} \sigma_{1,i} \sigma_{1,j}}_T &\gtrsim
\left(\sum_{M=2}^{\min(T, N)} \norm{\sum_{i \neq j} \eta_{1,i} \eta_{1,j} \sigma_{1,i} \sigma_{1,j}I_{M = \sum_{k=1}^{N} \eta_{1,k}} }_T^T \right)^{1/T} \\
&\gtrsim \left(\sum_{M=2}^{\min(T,N)} \norm{I_{M = \sum_{k=1}^{N} \eta_{1,k}}\sup_{\norm{b}_2, \norm{c}_2 \le \sqrt{T}, \norm{b}_{\infty}, \norm{c}_{\infty} \le 1} \sum_{i \neq j} \eta_{1,i} \eta_{1,j} b_i c_j}_T^T \right)^{1/T}  \\
&\gtrsim \left(\sum_{M=2}^{\min(T,N)} \norm{I_{M = \sum_{k=1}^{N} \eta_{1,k}} M^2}_T^T \right)^{1/T}  \\
&= \left(\sum_{M=2}^{\min(T,N)} \mathbb{P}[M = \sum_{i=1}^N \eta_{1,i}] M^{2T}\right)^{1/T}  \\
&= \left(\sum_{M=2}^{\min(T,N)} {N \choose M} \left(\frac{s}{m}\right)^M \left(1- \frac{s}{m}\right)^{N-M} M^{2T}\right)^{1/T}  \\
&\gtrsim \left(\sum_{M=2}^{\min(T,N)} \left(\frac{Ns}{mT}\right)^M \left(\frac{T}{M}\right)^M \left(1- \frac{s}{m}\right)^{N-M} M^{2T}\right)^{1/T}  \\
&\gtrsim \left(\sum_{M=2}^{\min(T,N)} \left(\frac{s}{mTv^2}\right)^M \left(1- \frac{s}{m}\right)^{N-M} M^{2T}\right)^{1/T}  \\
&\gtrsim \left(\sum_{M=2}^{\min(T,N)} \left(\frac{s}{mTv^2}\right)^M M^{2T}\right)^{1/T},
\end{align*}
where the last line follows from the fact that since $T \ge \frac{se}{mv^2}$ and $s \le m/e$, we know that:
\[\left(1- \frac{s}{m}\right)^{\frac{N-M}{T}} \ge \left(1- \frac{s}{m}\right)^{\frac{N}{T}} \ge \left(1- \frac{s}{m}\right)^{\frac{Nmv^2}{se}} \ge \left(1- \frac{s}{m}\right)^{\frac{m}{s}} \ge 0.25.\]

Setting $t = T/M$, we obtain, up to constants:
\[\sup_{2 \le M \le \min(T, N)} \left(\frac{s}{mTv^2}\right)^{M/T} M^{2} = \sup_{\max(1, T/N) \le t \le T/2} \left(\frac{T^2}{t^2}\right) \left(\frac{s}{mTv^2}\right)^{1/t}.\] We can take a derivative to obtain the two expressions in the lemma statement at the following regimes of parameters: $\max(1, Tv^2) \le \ln(Tmv^2/s) \le T$ and $\ln(Tmv^2/s) > T$. The second regime aligns with the lemma statement. Thus it suffices to show that when $v \le \frac{\sqrt{\ln(m/s)}}{\sqrt{T}}$, it is true that $Tv^2 \le \ln(Tmv^2/s)$. This is a straightforward calculation\footnote{In fact, $v = \frac{\sqrt{\ln(m/s)}}{\sqrt{T}}$ is very close to the value where $Tv^2 = \ln(Tmv^2/s)$, so this approximation is essentially tight.}.

Now, let's consider the case where we want to bound $\norm{Z_1(v, \ldots, v,  0, \ldots, 0) I_{\sum_{k=1}^N \eta_{1,k} = 2}}_T$. It follows from the above calculations, without taking the sum that we obtain a lower bound of 
\[\left({N \choose 2} \left(\frac{s}{m}\right)^2 \left(1- \frac{s}{m}\right)^{N-2}\right)^{1/T} \gtrsim \left(\frac{s}{mTv^2}\right)^{2/T}. \]
\end{proof}

\subsection{Proof of Lemma~$\ref{rowbound}$} 
\label{sec:analyzerowupper}
In Section 2.1, we discussed the tractability issues with using the general quadratic form moment bound Lemma~$\ref{usefulquadformbound}$ to upper bound $\norm{Z_r(x_1, \ldots, x_n)}_q$. Thus, we require simpler bounds that are easier to analyze. Linear forms naturally arise in the upper bound since $Z_r(x_1, \ldots, x_n) = \left(\sum_{1 \le i \le n} \eta_{r,i} \sigma_{r,i} x_i \right)^2 - \sum_{1 \le i \le n} \eta_{r,i} x_i^2 \le \left(\sum_{1 \le i \le n} \eta_{r,i} \sigma_{r,i} x_i \right)^2$.
However, it turns out that a vanilla linear form bound (e.g. Proposition \ref{simpler}) here is weak due to the loss arising from ignoring the $\sum_{1 \le i \le n} \eta_{r,i} x_i^2$ term. Thus, we use Lemma \ref{generalsymmetric} (our generalized bound tailored to squares of linear forms with a zero diagonal) to obtain: 
\begin{lemma}
\label{supexpr}
If $\norm{x}_{\infty} \le v$ and $\norm{x}_2 \le 1$, then we have that: 
\[\norm{Z_r(x_1, \ldots, x_n)}_T = \norm{\sum_{i \neq j} \eta_{r,i}\eta_{r,j} \sigma_{r,i} \sigma_{r,j} x_i x_j}_T \lesssim v^2 \left(\sup_{1 \le t \le T/2} \frac{T^2}{t^2} \left(\frac{s}{mTv^2}\right)^{1/t}\right).\]
\end{lemma}
\begin{proof}
This can be seen by simply taking $Y_i = \eta_{r,i} \sigma_{r,i}$ in Lemma~$\ref{generalsymmetric}$. 
\end{proof}

It turns out that using only this bound would lose the $m \ge s \cdot e^{\Theta\left(\max(1, \frac{p \epsilon^{-1}}{s})\right)}$ branch in the lower bound on $v(m, \epsilon, \delta, s)$ in Theorem \ref{thm:mainresult}. The lower bound on moments of $\norm{Z_r(v, \ldots, v, 0, \ldots, 0)}_T$ in Lemma~$\ref{lowerboundrow}$ sheds light on where this loss may be arising. We see that the problematic case is when $v \ge \frac{\sqrt{\ln(m/s)}}{\sqrt{T}} =: v_1$, and so we require a new bound for this regime. Since the vector $[v_1, \ldots, v_1, 0, \ldots, 0]$ is in $S_v$ when $v_1 \le v$, we can't hope to beat the bound of $||Z_r(v_1, \ldots, v_1, 0, \ldots, 0)||_T \gtrsim \frac{T^2v_1^2}{\ln^2(Tmv_1^2/s)} \simeq \frac{T}{\ln(m/s)}$ from Lemma~$\ref{lowerboundrow}$. We show that we can match this value: 
\begin{lemma}
\label{easybound}
Suppose that $x = [x_1, \ldots, x_n]$ satisfies $\norm{x}_2 = 1$  and $\norm{x}_{\infty} < v$. If  $s \le m/e$, $T \ge \frac{se}{mv^2}$, $T \ge 3$, $T \ge \ln(mv^2/s)$, then: 
\[\norm{Z_r(x_1, \ldots, x_n)}_T = \norm{\sum_{i \neq j} \eta_{r,i}\eta_{r,j} \sigma_{r,i} \sigma_{r,j} x_i x_j}_T \le \norm{\sum_{i} \eta_{r,i} \sigma_{r,i}  x_i}^2_{2T} \lesssim \frac{T}{\ln(m/s)}.\]
\end{lemma}

The proof of this bound requires a new technique that handles larger $|x_i|$ entries, while still managing the many smaller $|x_i|$ that are still allowed to be present. We separate out $|x_i| \ge \frac{\sqrt{\ln(m/s)}}{\sqrt{T}}$ and $|x_i| \le \frac{\sqrt{\ln(m/s)}}{\sqrt{T}}$. In the quadratic form formulation of $Z_r(x_1, \ldots, x_n)$, this separation cannot be carried out, since there would be cross-terms between $|x_i| \ge \frac{\sqrt{\ln(m/s)}}{\sqrt{T}}$ and $|x_i| \le \frac{\sqrt{\ln(m/s)}}{\sqrt{T}}$. As a result, we require the linear form bound (Proposition~$\ref{simpler}$) for $|x_i| \le \frac{\sqrt{\ln(m/s)}}{\sqrt{T}}$, and it turns out to be sufficiently tight in this regime. 
\begin{proof}[Proof of Lemma~$\ref{easybound}$]
WLOG, assume that $|x_1| \ge |x_2| \ge \ldots \ge |x_n|$. Let $P = \ceil{\frac{T}{\ln(m/s)}}$. We know that
\[\norm{\sum_{i} \eta_{r,i} \sigma_{r,i}  x_i}_{2T} \le \norm{\sum_{1 \le i \le P} \eta_{r,i} \sigma_{r,i}  x_i}_{2T} + \norm{\sum_{i > P} \eta_{r,i} \sigma_{r,i}  x_i}_{2T}.\]
For $1 \le i \le P$, we use the bound $|\sum_{i=1}^p \eta_{r,i} \sigma_{r,i} x_i| \le \sum_{i=1}^p |x_i| \le \sqrt{\ceil{\frac{T}{\ln(m/s)}}} \le 2 \sqrt{\frac{T}{\ln(m/s)}}$. For the remaining terms, we take $Y_i = \eta_{r,i} \sigma_{r,i}$ in Proposition~$\ref{simpler}$ to obtain the following upper bound\footnote{Observe that the upper endpoint of $T$ on the $\sup$ expression does not match with the upper endpoint of $T/2$ on the $\sup$ expression in Lemma~$\ref{supexpr}$, and in fact, it turns out that this bound is not sufficiently strong to recover Theorem~$\ref{thm:mainresult}$. This is sufficiently tight here, since we are focusing on the case where $\ln(Tmv'^2/s)$ is \textit{small}.} for $|x_i| \le v' := \frac{\sqrt{\ln(m/s)}}{\sqrt{T}}$ and $\norm{x}_2 \le 1$:
\[\norm{\sum_{i} \eta_{r,i} \sigma_{r,i}  x_i}_{2T} \lesssim v' \left(\sup_{1 \le t \le T} \frac{T}{t} \left(\frac{s}{mTv'^2}\right)^{\frac{1}{2t}}\right).\] 
Based on the conditions in this lemma statement, we know that $\frac{mTv'^2}{s} = \frac{m T \ln(m/s)}{sT} = \frac{m}{s\ln(m/s)} \ge e$. Thus taking a derivative, we obtain that this can be upper bounded by taking $t = \ln(mTv'^2/s)$ which yields:
\[\frac{Tv'}{\ln(mTv'^2/s)}  = \frac{Tv'}{\ln(m/s) - \ln\ln(m/s)} \le \frac{Tv'}{0.5 \ln(m/s)} = 2 \frac{\sqrt{T}}{\sqrt{\ln(m/s)}}. \]
\end{proof}

Finally, combining Lemma~$\ref{supexpr}$ and Lemma~$\ref{easybound}$ yields Lemma~$\ref{rowbound}$:
\begin{proof}[Proof of Lemma~$\ref{rowbound}$]
We apply Lemma~$\ref{supexpr}$ at $T = 2$ to directly obtain $\frac{Ts}{m}$, and for $T \ge 3$, we apply Lemma~$\ref{supexpr}$ and take a derivative to obtain: 
\[\norm{\sum_{i \neq j} \eta_{r,i}\eta_{r,j} \sigma_{r,i} \sigma_{r,j} x_i x_j}_T \lesssim v^2
 \begin{cases}
    \frac{Ts}{mv^2}, & \text{for } se \ge mTv^2 \\
    \frac{T^2}{\ln(mTv^2/s)^2}, & \text{for } se \le mTv^2, \ln(Tmv^2/s) \le T \\
     \left(\frac{s}{mTv^2}\right)^{2/T}, & \text{for } \ln(Tmv^2/s) \ge T, se \le mTv^2
  \end{cases}
  .\]
To obtain the desired bound, we also include the bound from Lemma~$\ref{easybound}$ in the middle regime.
\end{proof}

\section{Combining rows to bound $\norm{R(x_1, \ldots, x_n)}_q$}
\label{sec:combinerows}
Now, we show to move from bounds on moments of individual rows (i.e. $Z_r(x_1, \ldots, x_n)$) to bounds on moments of $R(x_1, \ldots, x_n)$. In \Cref{sec:proofupper}, we obtain an upper bound on $\norm{R(x_1, \ldots, x_n)}_q$, thus proving Lemma~$\ref{upper}$. In \Cref{sec:prooflower}, we obtain a lower bound on $\norm{R(x_1, \ldots, x_n)}_q$, thus proving Lemma~$\ref{lower}$.
\subsection{Proof of Lemma~$\ref{upper}$}
\label{sec:proofupper}
Since the $\eta_{r,i}$ are negatively correlated, we can always upper bound the moments of $R(x_1, \ldots, x_n)$ by the case of a sum of \textit{independent} random variables when $q$ is even\footnote{This can easily be seen by expanding.} $Z'_r(x_1, \ldots, x_n) \sim  Z_r(x_1, \ldots, x_n)$. 
\begin{equation}
\label{momentupperboundeq}
 s \cdot\norm{R(x_1, \ldots, x_n)}_q \le \norm{\sum_{r=1}^m Z_r(x_1, \ldots, x_n)}_q \le  \norm{\sum_{r=1}^m Z'_r(x_1, \ldots, x_n)}_q \lesssim \sup_{2 \le T \le q} \frac{q}{T} \left(\frac{m}{q}\right)^{1/T} \norm{Z_1(x_1, \ldots, x_n)}_T,
\end{equation}
where the last inequality follows from Lemma~$\ref{sumiid}$. Thus, it remains to analyze the $\sup$ expression. It turns out that each regime of bounds in Lemma~$\ref{rowbound}$ collapses to one value, so the different regimes in Lemma~$\ref{rowbound}$ correspond to different parts of the $\max$ expressions in Lemma~$\ref{upper}$. Depending on the parameters, some of these regimes may not exist, as is reflected by branches of the $\max$ expression sometimes vanishing in Lemma~$\ref{rowbound}$. We defer the computation to \Cref{sec:auxupper}. 

\subsection{Proof of Lemma~$\ref{lower}$}
\label{sec:prooflower}
Moving from a lower bound on the moments of individual rows given by Lemma~$\ref{lowerboundrow}$ to moments of \\
$R(v, \ldots, v, 0, \ldots, 0)$ is more delicate. Unlike in the upper bound, the negative correlations between random variables require some care to handle, even with the simplification that the $s$ nonzero entries in a column are chosen uniformly at random. For example, the conditional distribution of $\eta_{s+1,1} \mid \eta_{1,1} = \eta_{2,1} = \ldots = \eta_{s, 1} = 1$ is $0$, while the marginal distribution of $\eta_{s+1,1}$ has expectation $s/m$. One aspect that simplifies our analysis is that we \textit{know} from our proof of Lemma~$\ref{upper}$ which moments of $Z_r(x_1, \ldots, x_n)$ are critical in the $\sup$ expression in $\eqref{momentupperboundeq}$. We only need to account for these particular moments in our lower bound approach. It turns out that the three critical values are $q/T = 2$, $q/T = q$, and $q/T = \ln(qmv^4/s^2)$.

For $q/T = q$, where rows are isolated, we can directly obtain a bound from Lemma~$\ref{combinerowseasy}$ and Lemma~$\ref{lowerboundrow}$ to obtain.
\begin{lemma}
\label{alpha3}
Suppose $\mathcal{A}_{s,m,n}$ is a uniform sparse JL distribution. Suppose that $q$ is even, $s \le m/e$, $q \ge \frac{se}{mv^2}$, $1 \le \ln(qmv^2/s) \le q$, $v \le \frac{\sqrt{\ln(m/s)}}{\sqrt{q}}$ and $\frac{1}{v^2}$ is an even integer. Then it is true that:
\[\norm{R(v, \ldots, v,  0, \ldots, 0)}_q  \gtrsim \frac{q^2v^2}{s \ln^2(\frac{qmv^2}{s})}.\]
\end{lemma}
\begin{proof}
By Lemma~$\ref{combinerowseasy}$ with $T = 1$, we have that:
\[\norm{R(v, \ldots, v,  0, \ldots, 0)}_q \ge \frac{m^{1/q}}{s} \norm{Z_1(v, \ldots, v,  0, \ldots, 0)}_q \ge \frac{1}{s} \norm{Z_1(v, \ldots, v,  0, \ldots, 0)}_q .\]
Now, we apply Lemma~$\ref{lowerboundrow}$ to obtain the desired expression.
\end{proof}

For $q/T = 2$ and $q/T = \ln(qmv^4/s^2)$, we make use of the Lemma \ref{genterm} that relates moments of products of rows to products of moments of rows by taking advantage of either $s$ and $\frac{1}{v^2}$ being sufficiently large. The method essentially uses a counting argument to show that not too many terms vanish as a result of negative correlations, and requires adding in an indicator for the number of nonzero entries in a row being $2$ for some cases (which is sufficient to prove Lemma~$\ref{lower}$). 
\begin{lemma}
\label{genterm}
Suppose $\mathcal{A}_{s,m,n}$ is a uniform sparse JL distribution. If $1 \le T \le q/2$ is an integer, $q/T$ is an even integer, $\frac{1}{v^2}$ is an even integer, and $2Tv^2 \le s$, then:
\[ \norm{\prod_{i=1}^T Z_i(v, \ldots, v, 0, \ldots, 0)}^{1/T}_{q/T} \gtrsim
\begin{cases}
\norm{Z_1(v, \ldots, v, 0, \ldots, 0)}_{2} &\text{ if } T = q / 2 \\
 \norm{Z_1(v, \ldots, v, 0, \ldots, 0) I_{\sum_{i=1}^N \eta_{1,i} = 2}}_{q/T} &\text{ if } 1 \le T \le q/2
\end{cases}.
\]
\end{lemma}
\noindent We defer the proof to \Cref{sec:auxlower}. 

Now we can use Lemma~$\ref{combinerowseasy}$ coupled with Lemma~$\ref{genterm}$ and Lemma~$\ref{lowerboundrow}$ to handle the cases of $q/T = 2, \ln(qmv^4/s^2)$ and obtain the following bounds. For $q / T = 2$, we obtain:
\begin{lemma}
\label{alpha1}
Suppose $\mathcal{A}_{s,m,n}$ is a uniform sparse JL distribution. If $q$ is an even integer, $\frac{qv^2}{s} \le 1$, and $\frac{1}{v^2}$ is an even integer, then it is true that: 
\[\norm{R(v, \ldots, v, 0, \ldots, 0)}_q \gtrsim \left( \frac{q}{m}\right)^{1/2}.\]
\end{lemma}
\begin{proof}[Proof of Lemma~$\ref{alpha1}$]
Take $T = \frac{q}{2}$ and $qv^2 \le s$. By Lemma \ref{genterm} and Lemma \ref{combinerowseasy}, we have that:
\[\norm{R(v, \ldots, v, 0, \ldots, 0)}_q \gtrsim \frac{q}{s} \left( \frac{m}{q}\right)^{1/2} \norm{Z_1(v, \ldots, v, 0, \ldots, 0)}_{2}.\] Now, by Lemma~$\ref{lowerboundrow}$, we can see that $\norm{Z_1(v, \ldots, v, 0, \ldots, 0)}_{2} \gtrsim \frac{s}{m}$. Thus, our bound becomes:
\[\frac{q}{s} \left( \frac{m}{q}\right)^{1/2}\frac{s}{m} = \left( \frac{q}{m}\right)^{1/2}.\] 
\end{proof}
For $q/T = \ln(qmv^4/s^2)$, we similarly obtain the following bound using Lemma~$\ref{combinerowseasy}$ coupled with Lemma~$\ref{genterm}$.
\begin{lemma}
\label{alpha2}
Suppose $\mathcal{A}_{s,m,n}$ is a uniform sparse JL distribution. Suppose that $q$ is a power of $2$, $s \le m/e$,  $2qv^2 \le 0.5 s \ln(qmv^4/s^2)$, $\frac{1}{v^2}$ is even, $2 \le \ln(qmv^4/s^2) \le q$,  and $m \ge q$. Then it is true that:
\[\norm{R(v, \ldots, v,  0, \ldots, 0)}_q  \gtrsim \frac{qv^2}{s \ln(\frac{qmv^4}{s^2})}.\]
\end{lemma}
\begin{proof}
Let's let $f(x)$ be the function that rounds $x$ to the nearest power of $2$. By the conditions, we know that $2 \le f(\ln(qmv^4/s^2)) \le q$. Now, we want the condition $2qv^2 \le s f(\ln(qmv^4/s^2))$ to be satisfied. If $f(\ln(qmv^4/s^2)) \ge \ln(qmv^4/s^2)$, then this is implied by $2qv^2 \le s \ln(qmv^4/s^2) = s \max(\ln(qmv^4/s^2), 2)$, which is a strictly weaker condition than the one given in the lemma statement. If $f(\ln(qmv^4/s^2)) \le \ln(qmv^4/s^2)$, then $f(\ln(qmv^4/s^2)) \ge 0.5\ln(qmv^4/s^2)$ and so $2qv^2 \le 0.5 s \ln(qmv^4/s^2) \le s f(\ln(qmv^4/s^2))$ gives the desired condition.

We use the fact that $\ln(qmv^4/s^2) / 2 \le f(\ln(qmv^4/s^2)) \le 2 \ln(qmv^4/s^2)$. We apply Lemma \ref{genterm} and Lemma \ref{combinerowseasy}, with $T = \frac{q}{f(\ln(qmv^4/s^2))}$ and Lemma~$\ref{lowerboundrow}$ to see that if we have the additional condition that $f(\ln(qmv^4/s^2)) \ge \frac{se}{mv^2}$, then we know that: 
\begin{align*}
\norm{R(v, \ldots, v,  0, \ldots, 0)}_q &\gtrsim \frac{q}{s f(\ln(\frac{qmv^4}{s^2}))} \left( \frac{m}{q}\right)^{1 / f(\ln(\frac{qmv^4}{s^2}))} \norm{Z_1(v, \ldots, v,  0, \ldots, 0) I_{M = 2}}_{f(\ln(\frac{qmv^4}{s^2}))}\\
& \gtrsim  \frac{qv^2}{2 s \ln(\frac{qmv^4}{s^2})}\left( \frac{m}{q}\right)^{1 / f(\ln(\frac{qmv^4}{s^2}))} \left(\frac{s}{m  f(\ln(\frac{qmv^4}{s^2})) v^2}\right)^{2/ f(\ln(\frac{qmv^4}{s^2}))}\\
&= \frac{qv^2}{2 s \ln(\frac{qmv^4}{s^2})}  \left(\frac{s^2}{qmv^4}\right)^{1/ f(\ln(\frac{qmv^4}{s^2}))} \left(\frac{1}{f(\ln(qmv^4/s^2))^2}\right)^{1/f(\ln(qmv^4/s^2))}\\
&\gtrsim \frac{qv^2}{s \ln(\frac{qmv^4}{s^2})}  \left(\frac{s^2}{qmv^4}\right)^{\frac{2}{\ln(\frac{qmv^4}{s^2})}} .\\
&\gtrsim \frac{qv^2}{s \ln(\frac{qmv^4}{s^2})}.\\
\end{align*}

Now, we see that 
\[\frac{mv^2}{se} = \sqrt{\frac{qmv^4}{s^2}} \frac{1}{e} \frac{\sqrt{m}}{\sqrt{q}} \ge \frac{\sqrt{m}}{\sqrt{q}} \ge 1.\] This implies that $\frac{se}{mv^2} \le 1$, so the condition of $f(\ln(qmv^4/s^2)) \ge 2 \ge \frac{se}{mv^2}$ is automatically satisfied. 
\end{proof}
With these bounds, Lemma~$\ref{lower}$ follows.
\begin{proof}[Proof of Lemma~$\ref{lower}$]
We combine Lemma~$\ref{alpha3}$, Lemma~$\ref{alpha1}$, and Lemma~$\ref{alpha2}$.
\end{proof}

\section{Proofs of Auxiliary Lemmas for Lemma~$\ref{upper}$}
\label{sec:auxupper}
First, we use Lemma~$\ref{sumiid}$ and Lemma~$\ref{rowbound}$ to prove a upper bound $\norm{R(x_1, \ldots, x_q)}_q$ that is not quite in the desired form for Lemma~$\ref{upper}$. 
\begin{lemma}
\label{momentupperbound}
Let $2 \le q \le m$ be an even integer and $|x_i| \le v$ and $\norm{x}_2 = 1$. If $\frac{se}{mv^2} \ge q$, then: 
\[\norm{R(x_1, \ldots, x_n)}_q \lesssim \alpha_1(q,v,s,m).\] If $\ln(qmv^2/s) > q$ then we have
\[\norm{R(x_1, \ldots, x_n)}_q \lesssim \max(\alpha_1(q, v, s, m), \alpha_2(q, v, s, m)).\] In all other cases, we have that 
\[\norm{R(x_1, \ldots, x_n)}_q \lesssim \max(\alpha_1(q, v, s, m), \alpha_2(q, v, s, m), \min(\alpha_3(q,v,s,m), \alpha_4(q,v,s,m))).\]
The functions are defined as follows.
\begin{align*}
\alpha_1(q,v, s, m) &= \frac{\sqrt{q}}{\sqrt{m}} \\
\alpha_2(q,v,s,m) &= 
\begin{cases}
\frac{eqv^2}{s\ln(qmv^4/s^2)} & \text{ for } \ln(qmv^4/s^2) \ge 2 \\
\frac{\sqrt{q}}{\sqrt{m}} & \text{ for } \ln(qmv^4/s^2) \le 2 \\
\end{cases} \\
\alpha_3(q,v,s,m) &= 
\frac{qv^2e}{s} \sup_{T \le q, T \ge \max(\frac{se}{mv^2}, 3, \ln(mv^2T/s))} \frac{T}{\ln^2(mv^2T/s)} \left(\frac{s}{qv^2}\right)^{1/T}\\
\alpha_4(q,v,s,m) &= \frac{qe^2}{s \ln(m/s)}
\begin{cases}
1 & \text{ for } \ln(qmv^4/s^2) \ge 2 \\
\left(\frac{s}{qv^2}\right)^{1/\ln(mv^2/s)} & \text{ else } \\
\end{cases} 
\end{align*}
\end{lemma}
\begin{proof}[Proof of Lemma~$\ref{momentupperbound}$]
As we discussed in \Cref{sec:combinerows}, it suffices to bound 
\[\frac{1}{s} \sup_{2 \le T \le q} \frac{q}{T} \left(\frac{m}{q}\right)^{1/t} \norm{Z_1(x_1, \ldots, x_n)}_t.\] Our bounds on $\norm{Z_1(x_1, \ldots, x_n)}_t$ are based on Lemma~$\ref{rowbound}$. We split into cases based on the $T$ value, and how it separates into different cases in Lemma~$\ref{rowbound}$. Let 
\begin{align*}
\beta_1(q,v,s,m) &= \frac{1}{s} \sup_{T = 2, 3 \le T \le \frac{se}{mv^2}}\frac{q}{T} \left(\frac{m}{q}\right)^{1/t} \norm{Z_1(x_1, \ldots, x_n)}_t. \\
\beta_2(q,v,s,m) &= \frac{1}{s}  \sup_{\max(3, \frac{se}{mv^2}) \le T \le \ln(mv^2T/s)}\frac{q}{T} \left(\frac{m}{q}\right)^{1/t} \norm{Z_1(x_1, \ldots, x_n)}_t. \\
\beta_{34}(q,v,s,m) &= \frac{1}{s}  \sup_{T \ge \max(3, \frac{se}{mv^2}, \ln(mv^2T/s))} \frac{q}{T} \left(\frac{m}{q}\right)^{1/t} \norm{Z_1(x_1, \ldots, x_n)}_t. \\
\end{align*}
Let $\beta_3$ branch arise when we use the $\frac{T^2v^2}{\ln(Tmv^2/s)^2}$ for the $\norm{Z_1(x_1, \ldots, x_n)}_t$ bound, and let the $\beta_4$ branch arise when we use $\frac{Tv^2}{s \ln(m/s)}$ for the $\norm{Z_1(x_1, \ldots, x_n)}_t$ bound. Thus, we know that 
\[\beta_{34}(q,v,s,m) \le \min(\beta_3(q,v,s,m), \beta_4(q,v,s,m)).\] 

Let's first consider $\frac{se}{mv^2} \ge q$. In this case, only the $\beta_1$ branch arises. Now, suppose that $\frac{se}{mv^2} < q$. 

Suppose that $\ln(qmv^2 / s) > q$. Then we show that the $\beta_{34}$ branch does not arise. It suffices to show that $\ln(Tmv^2/s) > T$ for all $T \ge \frac{se}{mv^2}$. Let $x = Tmv^2 / s$. It suffices to show that $\frac{s}{mv^2} \frac{x}{\ln x} < 1$ for all $e \le x \le \frac{qmv^2}{s}$. Since $\frac{s}{mv^2} \frac{x}{\ln x} < 1$ at $x = \frac{qmv^2}{s}$ and this is an increasing function of $x$, we know that the condition is true.   

We now produce bounds $\alpha_1(q,v,s,m), \ldots, \alpha_4(q,v,s,m)$ such that $\beta_i(q,v,sm) \lesssim \alpha_i(q,v,s,m)$, which is what we do for the remainder of the analysis.

% When $\frac{se}{mv^2} \ge q \ge T$, we see that only the $\alpha_1(q,v,s,m)$ term arises. When $\ln(Tmv^2/s) \ge T$ (which can be written as $\frac{T}{\ln(Tmv^2/s)} \ge 1$) for all $\frac{se}{mv^2} \le T \le q$, we see that $\alpha_{34}$ does not arise. Let $x = Tmv^2/s$. Then we know that $x \ge e$ based on the condition $T \ge \frac{se}{mv^2}$. The condition $\frac{T}{\ln(Tmv^2/s)} \ge 1$ can be written as $\frac{s}{mv^2} \frac{x}{\ln x}$. This is an increasing function for $x \ge e$, and so if this condition is not met at $T = q$, then it is never going to be met. We write this condition as $\ln(qmv^2/s) \ge q$. In all other cases, we include all three terms.  

First, we handle the $\beta_1(q,v,s,m)$ term. We see that 
\[\beta_1 (q,v,s,m) = \frac{1}{s} \sup_{2 \le T \le \frac{s}{mv^2}} \frac{q}{T} \left(\frac{m}{q}\right)^{1/T} \frac{Ts}{m} = \frac{1}{s} \frac{qs}{m} \left(\frac{m}{q}\right)^{1/T} \le \frac{q}{m} \left(\frac{m}{q}\right)^{1/2} = \frac{\sqrt{q}}{\sqrt{m}}. \]

Now, we handle the $\beta_2(q,v,s,m)$ term. We obtain a bound for $\norm{Z_r}_T \lesssim v^2 \left(\frac{s}{mTv^2}\right)^{2/T}$. The expression becomes:
\begin{align*}
\beta_2(q,v,s,m) &= \frac{1}{s} \sup_{T \ge \max(\frac{se}{mv^2}, 3), T \le \ln(mv^2T/s)} \frac{qv^2}{T} \left(\frac{m}{q}\right)^{1/T} \left(\frac{s}{mTv^2}\right)^{2/T} \\
&= \frac{1}{s} \sup_{T \ge\max(\frac{se}{mv^2}, 3), T \le \ln(mv^2T/s)} \frac{qv^2}{T} \left(\frac{s}{\sqrt{qm}Tv^2}\right)^{2/T} \\
&\le \frac{1}{s} \sup_{T \ge \max(\frac{se}{mv^2}, 3), T \le \ln(mv^2T/s)} \frac{qv^2}{T} \left(\frac{s^2}{qmv^4}\right)^{1/T}.
\end{align*}
Suppose that $\ln(qmv^4/s^2) \ge 2$. In this case, we have that this expression is upper bounded by $T = \ln(qmv^4/s^2)$. When we plug this into the expression, we obtain $\frac{qv^2}{s\ln(qmv^4/s^2)}$. Otherwise, if $\ln(qmv^4 / s^2) \le 2$, then this expression is upper bounded by $T = 3$: 
\[\frac{qv^2}{s} \left(\frac{s^2}{qmv^4}\right)^{1/3} = \frac{C_1 C_5 q^{2/3}v^{2/3}}{s^{1/3}m^{1/3}}.\] We know that that $\frac{q^{2/3} v^{2/3}}{s^{1/3}m^{1/3}} \le \frac{\sqrt{q}}{\sqrt{m}}$ because this reduces to 
\[\frac{q^{1/6} v^{2/3}m^{1/6}}{s^{1/3}} = \left(\frac{qmv^4}{s^2}\right)^{1/6} \le e^{1/3}.\]

Now, we handle the $\beta_4(q,v,s,m)$ term when $\ln(qmv^2 / s) \le q$. 
\begin{align*}
\beta_4(q,v,sm) &= \frac{1}{s} \sup_{T \ge \max(\frac{se}{mv^2}, 3, \ln(mv^2T/s))} \frac{q}{T} \left(\frac{m}{q}\right)^{1/T} \frac{T}{\ln(m/s)} \\
&\le \sup_{T \ge \max(\frac{se}{mv^2}, 3, \ln(mv^2T/s))} \frac{q}{s\ln(m/s)} \left(\frac{mv^2}{s}\right)^{1/T}  \left(\frac{s}{qv^2}\right)^{1/T} \\
&\le \frac{qe}{s\ln(m/s)}  \sup_{T \ge \max(\frac{se}{mv^2}, 3, \ln(mv^2T/s))} \left(\frac{s}{qv^2}\right)^{1/T} \\
\end{align*}
If $s \le qv^2$, this is bounded by $1$, and if $s \ge qv^2$, this is bounded by $\left(\frac{s}{qv^2}\right)^{1/\ln(mv^2/s)}$. We see that $\frac{s}{qv^2} \le \frac{mv^2}{s}$, so $\left(\frac{s}{qv^2}\right)^{1/\ln(mv^2/s)} \le \left(\frac{mv^2}{s}\right)^{1/\ln(mv^2/s)} \le e$.    Thus this is bounded by $\frac{qe^2}{s\ln(m/s)}$. 

Now, we handle the $\beta_3(q,v,s,m)$ term. In this case, the expression becomes:
\begin{align*}
\beta_3(q,v,s,m) &=\frac{1}{s} \sup_{T \ge \max(\frac{se}{mv^2}, 3, \ln(mv^2T/s))} \frac{qv^2}{T} \left(\frac{m}{q}\right)^{1/T} \frac{T^2}{\ln^2(mv^2T/s)} \\
&\le \sup_{T \ge \max(\frac{se}{mv^2}, 3, \ln(mv^2T/s))} \frac{qv^2T}{s\ln^2(mv^2T/s)} \left(\frac{mv^2}{s}\right)^{1/T}  \left(\frac{s}{qv^2}\right)^{1/T} \\
&\le \frac{qv^2e}{s} \sup_{T \ge \max(\frac{se}{mv^2}, 3, \ln(mv^2T/s))} \frac{T}{\ln^2(mv^2T/s)} \left(\frac{s}{qv^2}\right)^{1/T}
\end{align*} 
\end{proof}

We use some function bounding arguments to come with a simpler bound for $\alpha_3$ for sufficiently large $v$. 
\begin{lemma}
\label{assumptionupperbound}
Assume that $C_2q^3mv^4 \ge s^2$ for some $C_2 \ge 1$. Then it is true that
\[\frac{qv^2e}{s} \sup_{T \le q, T \ge \frac{se}{mv^2}, 3, \ln(mv^2T/s)} \frac{T}{\ln^2(mv^2T/s)} \left(\frac{s}{qv^2}\right)^{1/T} \le \frac{C_2^{1/3}q^2v^2e^{5}}{s\ln^2(mv^2q/s)}.\]
\end{lemma}
\begin{proof}[Proof of Lemma~$\ref{assumptionupperbound}$]
With the assumptions that we made we know that $\frac{s}{q^3v^2C_2^2} \le \frac{mv^2}{s}$. This implies that our expression becomes:
\begin{align*}
E &= \frac{qv^2e}{s} \sup_{T \le q, T\ \ge \max(\frac{se}{mv^2}, 3, \ln(mv^2T/s))} \frac{T}{\ln^2(mv^2T/s)} \left(\frac{s}{qv^2}\right)^{1/T} \\
 &= \frac{qv^2e}{s} \sup_{T \le q, T \ge \max(\frac{se}{mv^2}, 3, \ln(mv^2T/s))} C_2^{1/T} \frac{T}{\ln^2(mv^2T/s)} \left(\frac{s}{C_2q^3v^2}\right)^{1/T} q^{2/T} \\
&\le \frac{qv^2e^2}{s} C_2^{1/3}\sup_{T \le q, T \ge \max(\frac{se}{mv^2}, 3, \ln(mv^2T/s))} \frac{T}{\ln^2(mv^2T/s)} q^{2/T} \\.
\end{align*}
It suffices to show that $\sup_{T \le q, T \ge \max(\frac{se}{mv^2}, 3, \ln(mv^2T/s))} \frac{T}{\ln^2(mv^2T/s)} q^{2/T} \le \frac{qe^3}{\ln^2(mv^2q/s)}$. 

Let $T_{min}$ be the minimum $T$ such that $T \ge \max(\frac{se}{mv^2}, 3, \ln(mv^2T/s))$. We just need to bound 
\begin{align*}
\sup_{T_{min} \le T \le q} \frac{T q^{2/T}}{\ln^2(mv^2T/s)} &\le \max\left(\sup_{T_{min} \le T \le \ln q} \frac{T q^{2/T}}{\ln^2(mv^2T/s)}, \sup_{\max(T_{min}, \ln q) \le T \le q} \frac{T q^{2/T}}{\ln^2(mv^2T/s)}  \right) \\
&\le \max\left(\sup_{T_{min} \le T \le \ln q} \frac{T q^{2/T}}{\ln^2(mv^2T/s)}, e^{2} \sup_{\max(T_{min}, \ln q) \le T \le q} \frac{T}{\ln^2(mv^2T/s)}  \right)
\end{align*}
First, we handle the second term. Let $Q = \frac{mv^2T}{s}$. We use that $T_{min} \ge \frac{se}{mv^2}$, so $\frac{mv^2T_{min}}{s} \ge e$ to conclude $Q \ge e$. We see that 
\[e^{2} \sup_{\max(T_{min}, \ln q) \le T \le q} \frac{T}{\ln^2(mv^2T/s)} \le e^{2} \frac{s}{mv^2} \sup_{e \le Q \le  \frac{qmv^2}{s}} \frac{Q}{\ln^2(Q)}.\] We see that setting $Q$ to its maximum value achieves within a factor of $e$ of the maximum value of $\frac{Q}{\ln^2(Q)}$. Thus, we obtain that this is upper bounded by $e^{3} \frac{q}{\ln^2(mv^2q/s)}$. 

Now, we just need to handle the first term. If $T_{min} \ge \ln q$, then this term doesn't exist. Let's take a log of the expression to obtain:
\[\ln\left(\frac{T}{\ln^2(mv^2T/s)}  \right) = \ln T - 2 \ln \ln (mv^2T/s) + \frac{2}{T} \ln(q)).\]
The derivative is:
\[\frac{1}{T} - \frac{2}{T\ln(mv^2T/s)} - \frac{2}{T^2} \ln(q).\] The sign of the derivative is the same as:
\[1 - \frac{2}{\ln(mv^2T/s)} - \frac{2 \ln q}{T}.\]Since $T_{min} \ge \frac{se}{mv^2}$, we know that $\ln(mv^2T/s) \ge 0$. Thus, we know that $1 - \frac{2}{\ln(mv^2T/s)} \le 1$. Since $T \le \ln q$, we know that $\frac{\ln q}{T} \ge 1$, so $- \frac{2 \ln q}{T} \le - 2$. Thus, the derivative is negative, so the $\sup$ is attained at $T_{min} = T$, where the expression is: 
\[e^{3} \frac{T_{min} q^{2/T_{min}}}{\ln^2(mv^2T_{min}/s)} \le e^3 \frac{(\ln q) q^{2/3}}{\ln^2(mv^2T_{min}/s)} \le e^3 \frac{q^{3/4}}{\ln^2(mv^2T_{min}/s)}.\] Thus, to upper bound by $ \frac{q}{\ln^2(mv^2q/s)}$, it suffices to show:
\[\frac{\ln^2(mv^2q/s)}{\ln^2(mv^2T_{min}/s)} \le q^{0.25}.\]
If  $\frac{s}{mv^2} \le 1$, the ratio is at most 
\[\frac{\ln(mv^2q/s)}{\ln(mv^2T_{min}/s) } \le \frac{\ln(mv^2/s) + \ln q}{\ln(mv^2/s) + \ln T_{min}} \le \frac{\ln q }{\ln e} = \ln q \le q^{0.25}.\]
If $\frac{s}{mv^2} \ge 1$, then $qmv^2/s \le q$. Using this and $\frac{mv^2T_{min}}{s} \ge e$, we know: 
\[\frac{\ln(mv^2q/s)}{\ln(mv^2T_{min}/s) } \le \frac{\ln(q)}{\ln(e)} = \ln q \le q^{0.25}.\] 
\end{proof}

Now, we combine Lemma~$\ref{momentupperbound}$ and Lemma~$\ref{assumptionupperbound}$ to prove Lemma~$\ref{upper}$.
\begin{proof}[Proof of Lemma~$\ref{upper}$]
First, we compute the second moment by hand: 
\begin{align*}
\mathbb{E}[R(x_1, \ldots, x_n)]^2 &= \frac{1}{s^2} \mathbb{E}\left[\left(\sum_{r=1}^m \sum_{i \neq j} \eta_{r,i} \eta_{r,j} \sigma_{r,i} \sigma_{r,j} x_i x_j\right)^2\right] \\
&= \frac{2}{s^2} \mathbb{E}\left[\sum_{r=1}^m \sum_{i \neq j} \eta_{r,i} \eta_{r,j} x^2_i x^2_j\right] \\
&\le \frac{2}{m} \left(\sum_{i=1}^n x_i^2\right)^2 \\
&= \frac{2}{m}.
\end{align*}
For $2 < q \le m$, we apply Lemma~$\ref{momentupperbound}$ and Lemma~$\ref{assumptionupperbound}$. We only include $\alpha_4$ when $\ln(qmv^4/s^2) \ge 2$ to simplify the bound. The bound follows. 
\end{proof}

\section{Proof of Auxiliary Lemma for Lemma~$\ref{lower}$}
\label{sec:auxlower}
We prove Lemma~$\ref{genterm}$.
\begin{proof}[Proof of Lemma~$\ref{genterm}$]
First, we show the following fact: Suppose that there are $T$ distinguishable buckets and we want to a assign an ordered pair of $2$ unequal elements in $[N]$ to each bucket so that the total number of times that any element $i \in [N]$ shows up is $\le s$. We show that the number of such assignments is at least $C^T N^{2T}$ for some constant $C$. To prove this, we first consider the case where $N \ge 2T$. In this case, we have that the number of such assignments is at least:
\[N(N-1)(N-2) \ldots (N-2T + 1) \geq C_1^{2T} N^{2T}.\] Now, if $N < 2T$, then we define:
\[\beta = \ceil{\frac{2T}{N}} = \ceil{2Tv^2} \le s.\] We partition $2T$ into $\beta$ blocks, each of size $N$, until potentially the last block, which may be smaller. We can read off ordered pairs assigned to each bucket from this formulation. Let's assume that each block is a permutation of $1, \ldots, N$, and the last block is $2T - (\beta - 1)(N)$ non-equal numbers drawn from $1, \ldots, N$. (this satisfies the unequal ordered pair condition). Then the number of assignments is $(N!)^{\beta - 1} \cdot (N)(N-1) \ldots (N- (2T - (\beta - 1)(N)) + 1)$. This is at least as big as $C_1^{2T} N^{2T}$ for some constant $C_1$. 

First, we handle the case where $q / T = 2$. Since we have a uniform sparse JL distribution, we know that for $1 \le x \le s$:
\[\mathbb{E}[\eta_{1,1} \ldots \eta_{x,1}] \geq \frac{s(s-1) \ldots (s-x + 1)}{(m)(m-1) \ldots (m-x+1)} \geq C_2^{-x} \left(\frac{s}{m}\right)^x.\] We know that
\[Z_r^{2} = 2 \left(\sum_{i \neq j} \eta_{r,i} \eta_{r, j}\right) + Y_r,\] where $Y_r$ has expectation $0$.  In this case we have that 
\[Z_1^2 \ldots Z_T^2 = 2^T \left(\sum_{i_1 \neq j_1, \ldots, i_T \neq j_T} \prod_{k=1}^T \eta_{k,i_k} \eta_{k, j_k} \right) + Q.\] where $Q$ consists of terms that contain a factor of some $Y_r$. Due to the independence of the Rademachers, the expectation of any term that contains a factor of $Y_r$ has expectation $0$, which implies that:
\[\mathbb{E}[Z_1^2 \ldots Z_T^2] = v^{2T} 2^T \mathbb{E}\left[\left(\sum_{i_1 \neq j_1, \ldots, i_T \neq j_T} \prod_{k=1}^T \eta_{k,i_k} \eta_{k, j_k} \right)\right].\] 
Let $\eta'_{r,i} \sim \eta_{r,i}$ be \textit{independent} random variables. Suppose that 
\[Z'_r := v^{2T} \sum_{i \neq j} \eta'_{r,i} \eta'_{r,j} \sigma_{r,i} \sigma_{r,j}.\]  We know that
\[Z_r^{'2} = 2 \left(\sum_{i \neq j} \eta'_{r,i} \eta'_{r, j}\right) + Y'_r,\] where $Y'_r$ has expectation $0$. This means that:
\[Z_1^{'2} \ldots Z_T^{'2} = v^{2T} 2^T \left(\sum_{i_1 \neq j_1, \ldots, i_T \neq j_T} \prod_{k=1}^T \eta'_{k,i_k} \eta'_{k, j_k} \right) + Q' \] where $Q'$ consists of terms that contain a factor of some $Y'_r$. For similar reasons, this implies that 
\[\mathbb{E}[Z_1^{'2} \ldots Z_T^{'2}] = v^{2T} 2^T \mathbb{E}\left[\left(\sum_{i_1 \neq j_1, \ldots, i_T \neq j_T} \prod_{k=1}^T \eta'_{k,i_k} \eta'_{k, j_k} \right)\right].\]
Let's view $\prod_{k=1}^T \eta'_{k,i_k} \eta'_{k, j_k}$ and $\prod_{k=1}^T \eta_{k,i_k} \eta_{k, j_k}$ as terms in a sum. In the second expression, every term has expectation $\left(\frac{s}{m}\right)^{2T}$, and there are at most $N^{2T}$ terms. In the first expression, if there are $> s$ copies of any $i_k$ value, then the expectation is $0$. Otherwise, the expectation varies between $C_2^{-2T} \left(\frac{s}{m}\right)^{2T}$ and $\left(\frac{s}{m}\right)^{2T}$. By the counting argument at the beginning of the proof, we know that there are at least $C_1^{2T} N^{2T}$ terms. This implies that 
\[\norm{Z_1 \ldots Z_T}_2 \gtrsim C^T \norm{Z'_1 \ldots Z'_T}_2 = C^T \norm{Z'_1}^{T}_2 = C^T \norm{Z_1}^{T}_2\] as desired. 

Now, we handle the case of the general $q / T$. Since we have a uniform sparse JL distribution, we know that for $1 \le x \le s$:
\[\mathbb{E}[\eta_{1,1} \ldots \eta_{x,1}] \geq \frac{s(s-1) \ldots (s-x + 1)}{(m)(m-1) \ldots (m-x+1)} \geq C_2^{-x} \left(\frac{s}{m}\right)^x.\]  We know that
\[(Z_r)^{q/T} =  2^{q/T - 1} \sum_{i \neq j} \left(\eta_{r,i} \eta_{r, j}\right)^{q/T} + Y_r,\] where $Y_r$ has expectation $\ge 0$. In this case we have that 
\[(Z_1\ldots Z_T)^{q/T} = 2^{q - T} \left(\sum_{i_1 \neq j_1, \ldots, i_T \neq j_T} \prod_{k=1}^T \left(\eta_{k,i_k} \eta_{k, j_k}\right)^{q/T}  \right) + Q.\] where $Q$ has expectation $\ge 0$. This implies that:
\[\mathbb{E}[Z_1^{q/T} \ldots Z_T^{q/T}] \ge v^{2T} 2^{q - T} \mathbb{E}\left[\left(\sum_{i_1 \neq j_1, \ldots, i_T \neq j_T} \prod_{k=1}^T \left(\eta_{k,i_k} \eta_{k, j_k}\right)^{q/T}  \right)\right].\] 
Let $\eta'_{r,i} \sim \eta_{r,i}$ be \textit{independent} random variables, and let $M'_r = \sum_{i=1}^N \eta'_{r,i}$. Suppose:
\[Z'_r := v^{2T} \sum_{i \neq j} \eta'_{r,i} \eta'_{r,j} \sigma_{r,i} \sigma_{r,j}.\] We know that
\[(Z'_rI_{M'_r = 2})^{q/T} =  2^{q/T - 1} \sum_{i \neq j} \left(\eta'_{r,i} \eta'_{r, j} I_{M'_r = 2}\right)^{q/T} + Y'_r,\] where $Y'_r$ has expectation $0$. In this case we have that 
\[(Z'_1I_{M'_1 =2} \ldots Z'_TI_{M'_T =2})^{q/T} = 2^{q - T} \left(\sum_{i_1 \neq j_1, \ldots, i_T \neq j_T} \prod_{k=1}^T \left(\eta'_{k,i_k} \eta'_{k, j_k} I_{M'_k = 2}\right)^{q/T}  \right) + Q'.\] where $Q'$ consists of terms that contain a factor of some $Y'_r$. For similar reasons to the above, we have that:
\[\mathbb{E}[Z_1^{'q} \ldots Z_T^{'q}] = v^{2T} 2^{q - T} \mathbb{E}\left[\left(\sum_{i_1 \neq j_1, \ldots, i_T \neq j_T} \prod_{k=1}^T \left(\eta'_{k,i_k} \eta'_{k, j_k} I_{M'_k = 2}\right)^{q/T}  \right)\right].\] 
Let's view  $\prod_{k=1}^T \left(\eta_{k,i_k} \eta_{k, j_k} \right)^{q/T}$ and $\prod_{k=1}^T \left(\eta'_{k,i_k} \eta'_{k, j_k} I_{M'_k = 2}\right)^{q/T}$ as terms in a sum. In the second expression, every term has expectation $\le \left(\frac{s}{m}\right)^{2T}$ (the indicator can only \textit{reduce} the expectation), and there are at most $N^{2T}$ terms. In the first expression, if there are $> s$ copies of any $i_k$ value, then the expectation is $0$. Otherwise, the expectation varies between $C_2^{-2T} \left(\frac{s}{m}\right)^{2T}$ and $\left(\frac{s}{m}\right)^{2T}$. By the counting argument, we know that there are at least $C_1^{-2T} N^{2T}$ terms. This implies that 
\[\norm{Z_1 \ldots Z_T}_{q/T} \gtrsim C^T \norm{Z'_1 I_{M'_1 = 2} \ldots Z'_T I_{M'_T = 2}}_{q/T} = C^T \norm{Z'_1 I_{M'_1} = 2}^{T}_{q/T} = C^T \norm{Z_1 I_{M_1 = 2}}^{T}_{q/T}\] as desired. 
\end{proof}

% \section{Proof of Theorem \ref{lowerbound}, Theorem \ref{upperbound} and Corollary \ref{dimsparsitylower}}\label{sec:mainresultsfullproofs}

\section{Proof of Lemma \ref{upperfinal} and Lemma \ref{lowerupperfinal}}\label{sec:functionbounding}
Recall that our proof of Theorem \ref{thm:mainresult} requires cleaner bounds on moments of $\norm{R(x_1, \ldots, x_n)}_q$ that follow simplifying the bounds in Lemma~$\ref{upper}$ and Lemma~$\ref{lower}$ at the target values of $v$. The proofs of these lemmas boil down to function bounding and simplification.

% \section{Proofs of Lemma~$\ref{upperfinal}$ and Lemma~$\ref{lowerupperfinal}$}
% \label{sec:lowerupperfinal}

\subsection{Proof of Lemma \ref{upperfinal}}
First, we show how Lemma~$\ref{upper}$ implies Lemma~$\ref{upperfinal}$. The proof involves simplifying and bounding the function at the target $v$ value. 
\begin{proof}[Proof of Lemma~$\ref{upperfinal}$]
We plug $q = p$ into Lemma~$\ref{upper}$. We use this relaxed version of the bound: If $\frac{se}{mv^2} \ge q$, then $\norm{R(x_1, \ldots, x_n)}_q \lesssim \frac{\sqrt{q}}{\sqrt{m}}$. Otherwise, if  there exists $C_2 q^3mv^4 \ge s^2$, then
\[\norm{R(x_1, \ldots, x_n)}_q \lesssim 
\begin{cases} 
\max\left(\frac{\sqrt{q}}{\sqrt{m}}, \frac{C_2^{1/3}q^2v^2}{s \ln^2(qmv^2/s)}\right) & \text{ if } \ln(qmv^4/s^2) \le 2\\ 
\max\left(\frac{\sqrt{q}}{\sqrt{m}}, \frac{qv^2}{s\ln(qmv^4/s^2)}, \min\left( \frac{C_2^{1/3}q^2v^2}{s \ln^2(qmv^2/s)}, \frac{q}{s\ln(m/s)}\right)\right) & \text{ if } \ln(qmv^4/s^2) > 2\\
\end{cases}
\] Suppose that the absolute constant on the upper bounds is $\le C'$. Let $C = \max(C', 1)$ (we take $C$ to be the constant on the upper bounds). Let's take $C_{v,2} = \frac{0.25}{\sqrt{C}}$, $C_{v,1} = \min\left( \frac{0.1}{C^{3/2}}, C_{v,2}\right)$, $C_S = 4C$, $C_M = \max\left(e^{\frac{1}{C_{v,1}^2}}, 16C^2, e^{\frac{1}{C_{v,2}^2}}, e^2 \right)$. For the remainder of the analysis, we assume that $m \ge C_M \epsilon^{-2} p$ and $m < 2 \epsilon^{-2} \delta$. 

First, observe $m \ge 16C^2 \epsilon^{-2}p$ gives us that $C \frac{\sqrt{p}}{\sqrt{m}} \le 0.25 \epsilon$ regardless of $v$. 

Now, let $f_1 =  C_{v,1} \sqrt{\epsilon s}\frac{\ln(\frac{m \epsilon}{p})}{p}$, and $f_2 = \frac{C_{v,2} \sqrt{\epsilon s} \sqrt{\ln \frac{m \epsilon^{2}}{p}}}{\sqrt{p}}$. 

First, let's analyze $v = f_2$. We show that $\ln(pmf_2^4/s^2) \ge 2$. Observe that $\ln(pmf_2^4/s^2) = \ln(C_{v,2}^4 \ln^2(m \epsilon^{2}/p)) + \ln(m\epsilon^2/p)$. Using the fact that $m \ge e^{\frac{1}{C_{v,2}^2}} \epsilon^{-2} p$, we see that 
\[C_{v,2}^4 \ln^2(m \epsilon^{2}/p) \ge C_{v,2}^4 \frac{1}{C_{v,2}^4} = 1.\] Now, since $m \ge e^2 \epsilon^{-2} p$, this implies that 
\[\ln(pmf_2^4/s^2) = \ln(C_{v,2}^4 \ln^2(m \epsilon^{2}/p)) + \ln(m\epsilon^2/p) \ge \ln(m \epsilon^{2} /p) \ge 2. \] Moreover, we know that $p^3 mf_2^2 \ge s^2$, since $pmv^4 \ge e^2 s^2$. Now, we show that $C \frac{pf_2^2}{s \ln(pmf_2^4/s^2)} \le 0.25 \epsilon$.  Let's observe that 
\[\frac{Cpv^2}{s \ln(pmf_2^4/s^2)} \le \frac{C C_{v,2}^2 \epsilon \ln(\frac{m \epsilon^{2}}{p})}{\ln(\frac{m \epsilon^{2}}{p})} = C C_{v,2}^2 \epsilon  \] Since $C_{v,2} = \frac{0.25}{\sqrt{C}}$, we get a bound of $0.25 \epsilon$. 

Now, we handle the case where $m \ge s \cdot e^{\frac{C_S p \epsilon^{-1}}{s}}$. We first show that $C \frac{p}{s \ln(m/s)} \le 0.25\epsilon$. If $s \ge \Theta(\epsilon^{-1} \ln(1/\delta))$, using that $m \ge se$, this immediately follows from $\frac{p}{s \ln(m/s)} \le \frac{p}{s} \le 0.25 \epsilon$. Otherwise, we need it to be true that $s \ln(m/s) \ge 4 C p \epsilon^{-1}$. This can be written as 
$\ln(m/s) \ge \frac{4Cp\epsilon^{-1}}{s}$. Since $C_S = 4C$, this can be written as: $m \ge s \cdot e^{\frac{C_S p \epsilon^{-1}}{s}}$, as desired. This, combined with the above analysis, implies that when $m \ge s \cdot e^{\frac{C_S p \epsilon^{-1}}{s}}$, taking $v = f_2$:
\begin{align*}
\norm{R(x_1, \ldots, x_n)}_q &\le C \max\left(\frac{\sqrt{q}}{\sqrt{m}}, \frac{qv^2}{s\ln(qmv^4/s^2)}, \min\left( \frac{C_2^{1/3}q^2v^2}{s \ln^2(qmv^2/s)}, \frac{q}{s\ln(m/s)}\right)\right) \\
&\le C \max\left(\frac{\sqrt{q}}{\sqrt{m}}, \frac{qv^2}{s\ln(qmv^4/s^2)}, \frac{q}{s\ln(m/s)}\right)  \\
&\le 0.25\epsilon. 
\end{align*}

Now, we just need to handle the case where $m \le s \cdot e^{C_S p \epsilon^{-1}/s}$, $m \ge \Theta(\epsilon^{-1} \ln(1/\delta))$, $m \ge se$. Such values only exist if $s \le \Theta(\epsilon^{-1} \ln(1/\delta))$. 
Observe that we can set $C_2 = \frac{1}{C_{v,1}^4}$ and using the fact that $C_{v,1} \le C_{v,2}$, we obtain that \[\frac{C_2p^3mv^4}{s^2} \ge \frac{C_2p^3m}{s^2} \min\left(C_{v,1}, C_{v,2}\right)^4 \left(\frac{\sqrt{\epsilon s}}{p} \right)^4 = C_2 C_{v,1}^4 \frac{m \epsilon^2}{p} \ge C_2 C_{v,1}^4.\] Thus, this is lower bounded by $1$ when $C_2 = \frac{1}{C_{v,1}^4}$. 

First, we analyze the case of $v = f_1$. We show that $\frac{C C_2^{1/3}p^2v^2}{s\ln^2(pmv^2/s)} \le 0.1 \epsilon$. Observe that 
\[\frac{C C_2^{1/3}p^2v^2}{s\ln^2(pmv^2/s)} = \frac{\epsilon C C_2^{1/3} C_{v,1}^2 \ln^2(\frac{m \epsilon}{p})}{\ln^2(\frac{C_{v,1}^2 m \epsilon\ln^2(\frac{m \epsilon}{p})}{p})} = \frac{\epsilon C C_2^{1/3} C_{v,1}^2 \ln^2(\frac{m \epsilon}{p})}{\left(\ln(\frac{m \epsilon}{p}) + \ln(C_{v,1}^2 \ln^2(\frac{m \epsilon}{p})\right)^2}.\] Now, since $m \ge e^{1/C_{v,1}^2} \epsilon^{-2} p$, we know that $\ln(C_{v,1}^2 \ln^2(\frac{m \epsilon}{p})) \ge 0$. Thus we can bound the above expression by: 
\[\frac{\epsilon C C_{v,1}^{2/3} \ln^2(\frac{m \epsilon}{p})}{\ln^2(\frac{m \epsilon}{p})} = \epsilon C C_{v,1}^{2/3} \epsilon \le 0.1 \epsilon,\] where the last inequality uses the fact that $C_{v,1} \le \frac{0.1}{C^{3/2}}$.  

Let's now consider how the term $\frac{pv^2}{s \ln(pmv^4/s^2)}$ how changes as a function of $v$. This term only arises in the bound if $\ln(pmv^4/s^2) \ge 2$. First, we show this is an increasing function of $v$. Let $w = pmv^4/s^2$. We see that $\frac{pv^2}{s \ln(pmv^4/s^2)} = \frac{s}{\sqrt{pm}} \frac{ \sqrt{w}}{\ln w}$. We observe that this is an increasing function of $w$ as long as $w \ge e^2$, which is exactly our restriction on $w$. Thus, $\frac{pv^2}{s \ln(pmv^4/s^2)}$ is an increasing function of $v$ in this range. 

Now, we consider how the $\frac{C_2^{1/3}p^2v^2}{s \ln^2(pmv^2/s)}$ term changes a function of $v$. This term only arises in the bound if $\ln(pmv^2/s) \ge 1$. First, we show that $f(v) \le 2 f(v')$ if $v \le v'$. Let $w = pmv^2/s$. We see that $\frac{p^2v^2}{s \ln(pmv^2/s)} = \frac{p}{m} \frac{w}{\ln^2 w}$. We observe that this is an increasing function of $w$ as long as $w \ge e^2$. When $e \le w \le e^2$, observe that this is bounded by at most a factor of $2$ above any other $w$ value.

Now, for the remainder of the analysis, let $v =  \min(f_1, f_2)$. We show that $\norm{R(x_1, \ldots, x_n)}_q \le 0.25\epsilon$.

If $\ln(pmv^2/s) \le 1$ (i.e. $\frac{se}{mv^2} \ge p$), then we know that the bound is actually $\frac{\sqrt{p}}{\sqrt{m}}$, and we've already shown that $\norm{R(x_1, \ldots, x_n)}_q \le 0.25\epsilon$.

For the remainder of the analysis, we assume that $\ln(pmv^2/s) > 1$. 
 
First, suppose that $v = f_1$. If $\ln(pmv^4/s^2) \le 2$, then we know that 
\[\norm{R(x_1, \ldots, x_n)}_q \le C \max\left(\frac{\sqrt{q}}{\sqrt{m}}, \frac{C_2^{1/3}q^2v^2}{s \ln^2(qmv^2/s)}\right)\le 0.25\epsilon.\] Otherwise, we know that $\ln(pmv^4/s^2) > 2$. First let's show that that $C \frac{pv^2}{s \ln(pmv^4/s^2)} \le 0.25 \epsilon$. We know that $v \le f_2$. At $v = f_2$, we know that the expression is upper bounded by $0.25 \epsilon$. Since the $\frac{pv^2}{s \ln(pmv^4/s^2)}$ term is an increasing function of $v$ in this regime, this means that we get a bound of $0.25 \epsilon$ in this case too. Thus, we know that:
\begin{align*}
\norm{R(x_1, \ldots, x_n)}_q &\le C \max\left(\frac{\sqrt{q}}{\sqrt{m}}, \frac{qv^2}{s\ln(qmv^4/s^2)}, \min\left( \frac{C_2^{1/3}q^2v^2}{s \ln^2(qmv^2/s)}, \frac{q}{s\ln(m/s)}\right)\right) \\
&\le C \max\left(\frac{\sqrt{q}}{\sqrt{m}}, \frac{qv^2}{s\ln(qmv^4/s^2)}, \frac{C_2^{1/3}q^2v^2}{s \ln^2(qmv^2/s)}\right)  \\
&\le 0.25\epsilon.
\end{align*}

Now, suppose that $v = f_2$. We've already shown that $\ln(pmv^4/s^2) \ge 2$ here (near the beginning of the proof). Since $v \le  f_1$, we obtain a bound of $2 \cdot 0.1 \epsilon = 0.2 \epsilon$. This means: 
\begin{align*}
\norm{R(x_1, \ldots, x_n)}_q &\le C \max\left(\frac{\sqrt{q}}{\sqrt{m}}, \frac{qv^2}{s\ln(qmv^4/s^2)}, \min\left( \frac{C_2^{1/3}q^2v^2}{s \ln^2(qmv^2/s)}, \frac{q}{s\ln(m/s)}\right)\right) \\
&\le C \max\left(\frac{\sqrt{q}}{\sqrt{m}}, \frac{qv^2}{s\ln(qmv^4/s^2)}, \frac{C_2^{1/3}q^2v^2}{s \ln^2(qmv^2/s)}\right) \\
&\le 0.25\epsilon.
\end{align*}

\end{proof}

\subsection{Proof of Lemma \ref{lowerupperfinal}}
Now, we show how Lemma~$\ref{upper}$ and Lemma~$\ref{lower}$ imply Lemma~$\ref{lowerupperfinal}$. The proof simply involves bounding and simplifying the functions in the original lemmas at the target $v$ value.  
\begin{proof}[Proof of Lemma~$\ref{lowerupperfinal}$]
We use Lemma~$\ref{upper}$ but put in an absolute constant. Let $D_2 > 0$ be such that: if $\frac{se}{mv^2} \ge q$, then
\[\norm{R(x_1, \ldots, x_n)}_q \le D_2 \frac{\sqrt{q}}{\sqrt{m}}.\] 
Otherwise, if $q^3mv^4 \ge s^2$, then $\norm{R(x_1, \ldots, x_n)}_q$ is upper bounded by:
\small{ 
\[D_2
\begin{cases} 
\max\left(\frac{\sqrt{q}}{\sqrt{m}}, \frac{q^2v^2}{s \ln^2(qmv^2/s)}\right) & \text{ if } \ln(qmv^4/s^2) \le 2, \ln(qmv^2/s) \le q\\ 
\frac{\sqrt{q}}{\sqrt{m}} & \text{ if } \ln(qmv^4/s^2) \le 2, \ln(qmv^2/s) > q\\ 
\max\left(\frac{\sqrt{q}}{\sqrt{m}}, \frac{4096qv^2}{s\ln(qmv^4/s^2)}, \min\left( \frac{q^2v^2}{s \ln^2(qmv^2/s)}, \frac{q}{s\ln(m/s)}\right)\right) & \text{ if } \ln(qmv^4/s^2) > 2, \ln(qmv^2/s) \le q\\
\max\left(\frac{\sqrt{q}}{\sqrt{m}}, \frac{4096qv^2}{s\ln(qmv^4/s^2)}\right) & \text{ if } \ln(qmv^4/s^2) > 2, \ln(qmv^2/s) > q. \\
\end{cases}
\]} \normalsize{}

We use Lemma~$\ref{lower}$ but put in an absolute constant $D_1 > 0$ (which we take to be $\le 1$). Let $2 \le q \le m$ be an even integer, and suppose that $0 < v \le 0.5$ and $\frac{1}{v^2}$ is an even integer. If $qv^2 \le s$, then
\[\norm{R(x_1, \ldots, x_n)}_q \geq D_1 \frac{\sqrt{q}}{\sqrt{m}}.\] If $m \ge q$,  $2 \le \ln(qmv^4/s^2) \le q$, $2qv^2 \le 0.5 s \ln(qmv^4/s^2)$, and $s \le m/e$ then:
\[\norm{R(x_1, \ldots, x_n)}_q \geq D_1 \frac{4096 qv^2}{s \ln(qmv^4/s^2)}.\] If $v \le \frac{\sqrt{\ln(m/s)}}{\sqrt{q}}$ and $1 \le \ln(qmv^2/s) \le q/2$, and $s \le m/e$, then: 
\[\norm{R(x_1, \ldots, x_n)}_q \geq D_1 \frac{q^2v^2}{s \ln^2(qmv^2/s)}.\]

Let $D = \frac{D_1}{2048D_2}$. It suffices to show that for $v$ defined in the lemma statement, 
$\norm{R(v, v, \ldots, v, 0, \ldots, 0)}_q \ge 2\epsilon$ and 
\[ \frac{\norm{R(v, v, \ldots, v, 0, \ldots, 0)}_q}{\norm{R(v, v, \ldots, v, 0, \ldots, 0)}_{2q}} \ge D.\]

First, we handle the case where $m \le \Theta(\epsilon^{-2}\ln(1/\delta))$. Let's take $v = \psi$ for any sufficiently small $\psi$. By sufficiently small, we mean $v^2 \le \frac{se}{2mq}$ and $0 < v \le 0.5$. This implies that $\frac{se}{mv^2} \ge 2q$ and $qv^2 \le s$. Thus we know (using that $q \le m$) that $\norm{R(x_1, \ldots, x_n)}_{2q} \le D_2 \frac{\sqrt{2q}}{\sqrt{m}}$ and $\norm{R(x_1, \ldots, x_n)}_q \ge D_1 \frac{\sqrt{q}}{\sqrt{m}}$. This means that:
\[ \frac{\norm{R(v, v, \ldots, v, 0, \ldots, 0)}_q}{\norm{R(v, v, \ldots, v, 0, \ldots, 0)}_{2q}} \ge D \] as desired. Suppose that $m \le \Theta(\epsilon^{-2} \ln(1/\delta))$. Based on the setting $q$, this means that $\norm{R(v, \ldots, v, 0 \ldots, 0)}_q \ge D_1 \frac{\sqrt{q}}{\sqrt{m}} \ge 2 \epsilon$ as desired. 

Now, we handle the cases where $m \ge \Theta(\epsilon^{-2}\ln(1/\delta))$. Notice that the condition $f'(m, \epsilon, \delta, s) \le 0.5$ allows us to assume that $s \le \Theta(\epsilon^{-1} \ln(1/\delta))$ and $m \le \epsilon^{-2} e^p$. Let $f_1 = 4 \sqrt{\epsilon s} \frac{\ln(\frac{m\epsilon}{q})}{q}$ and let $f_2 = \sqrt{\epsilon s} \frac{\sqrt{\ln(\frac{m \epsilon^2}{q}}}{\sqrt{q}}$. We will consider $v = C_{v,1} f_1 = : v_1$ and $v = C_{v,2} f_2 =: v_2$. First, we handle the condition of $q^3mv^4 \ge s^2$. We enforce the condition $C_{v,1}, C_{v,2} \ge 1$. Assuming that $v \ge \frac{\sqrt{\epsilon s}}{q}$ (which is true at the two values of $v$ that we consider), we know $\frac{q^3mv^4}{s^2} \ge \frac{m \epsilon^2}{q} \ge 1$. Also, we make $m \ge 2C^2 \epsilon^{-2} q$, so that $\sqrt{\frac{2q}{m}} \le \sqrt{\frac{2q}{2C^2 \epsilon^{-2} q}} = \frac{\epsilon}{C}$. 

Consider $v = v_2$. We first check that the conditions for the upper bound are satisfied. We have that $\frac{qmv^4}{s^2} = C_{v,2}^4 \frac{m \epsilon^2}{q} \ln^2(\frac{m\epsilon^2}{q})$. Observe that when $m \ge e^2 \epsilon^{-2} q$ and $C_{v,2} \ge 1$, this is lower bounded by $e^2$, so $\ln(qmv^4/s^2) \ge 2$. Also, we have that $\frac{qmv^2}{se} = \sqrt{qm} \sqrt{\frac{qmv^4}{s^2}} \frac{1}{e} \ge 1$.  Now, we check the additional conditions needed for the lower bound. Observe that 
\[\frac{2qv^2}{s} = 2 \epsilon C_{v,2}^2 \ln(\frac{m \epsilon^2}{q}) \le  0.5 C_{v,2}^4 \frac{m \epsilon^2}{q} \ln^2(\frac{m\epsilon^2}{q}) = 0.5 \ln(qmv^4/s^2)\] as desired. We check that $\ln(qmv^4/s^2) \le q$. It suffices to show that
\[\frac{m \epsilon^2}{q} \ln^2\left(\frac{m\epsilon^2}{q}\right) \le \frac{e^q}{C_{v,2}^4}. \] Using the condition that $m \le \epsilon^{-2} \frac{e^q}{q C_{v,2}^4}$ where we obtain that
\[\frac{m\epsilon^2}{q} \ln^2(\frac{m \epsilon^2}{q}) \le \frac{e^q}{q^2C_{v,2}^4} \ln^2(\frac{e^q}{q^2C_{v,2}^4}) \le \frac{e^q}{q^2C_{v,2}^4} \ln^2(e^q) \le \frac{e^q}{C_{v,2}^4}\] as desired. Now, we compute the value of $\frac{qv^2}{s \ln(qmv^4/s^2)}$ at $v = C_{v,2} f_2$. We obtain: 
\[\frac{qv^2}{s \ln(qmv^4/s^2)} = C_{v,2}^2 \epsilon \frac{\ln(m\epsilon^2/q)}{\ln\left(\frac{m \epsilon^2}{q} \right) + \ln\left(C_{v,2}^4 \ln^2\left(\frac{m \epsilon^2}{q}\right)\right)}. \]

Consider $v = v_1$. We first check that the conditions for the upper bound are satisfied. In this case, we have that $\frac{qmv^2}{s} = 16 C_{v,1}^2 \frac{m \epsilon}{q} \ln^2(\frac{m\epsilon}{q})$. Observe that when $C_{v,1} \ge 1$ and $m \ge e^2 \epsilon^{-2} q \ge e^2 \epsilon^{-1} q$, this is lower bounded by $e^2$, so $\ln(qmv^2/s) \ge 2$. Now, we claim that when $f_1 \le f_2$, we show that $\ln(qmv^2/s) \le q/2$. In this case, using that $m \le \epsilon^{-2} q e^q$, we have: $\frac{4 \ln (m \epsilon/q)}{q} \le \frac{\sqrt{\ln (m \epsilon^2 / q)}}{\sqrt{q}}$. This means that $\ln (m \epsilon/q) \le \sqrt{q} \sqrt{\ln (m \epsilon^2 / q)} / 4 \le q / 4$. 
Observe that
\begin{align*}
\ln(qmv^2/s) &= \ln(16C_{v,1}^2) + \ln(m \epsilon /q) + 2 \ln \ln(m \epsilon /q)\\
&\le \ln(16C_v^2) + \frac{q}{4} + 2 \ln \ln q\\
&\le \frac{q}{2}.
\end{align*} At this value, observe that: 
\[\frac{q^2v^2}{s \ln^2(qmv^2/s)} = 16 C_{v,1}^2 \epsilon \left(\frac{\ln(m\epsilon/q)}{\ln\left(\frac{m \epsilon}{q} \right) + \ln\left(16 C_{v,1}^2 \ln^2\left(\frac{m \epsilon}{q}\right)\right)}\right)^2.  \] 
 
Let $C = D_1$. Let's set $\sqrt{\frac{1}{C}} \le C_{v,2} = C_{v,1} = C_v \le 4 \sqrt{\frac{1}{C}}$. Using the fact that $v^2 \le 0.5$ (so $\frac{1}{v^2} \ge 2$), this means that $\frac{1}{v^2}$ has can take on at least $3$ different powers of $2$. Let's observe that when $16C_{v}^2 \ln^2(\frac{m \epsilon}{q}) \le \frac{m \epsilon}{q}$ (we can get this condition by saying that $m \ge C_{M,2} \epsilon^{-2} q$ for a sufficiently large $C_{M,2}$) and $16C_v^2 \ln^2(m \epsilon / q) \ge 1$ (we can get this condition by saying that $m \ge C_{M,2} \epsilon^{-2} q$ for a sufficiently large $C_{M,2}$), we know that 
\[\frac{4 \epsilon}{C} \le 4C_{v}^2 \epsilon \le \frac{q^2v_1^2}{s \ln^2(qmv_1^2/s)} \le 16 C_v^2 \epsilon \le \frac{256 \epsilon}{C}. \]
Suppose that $C_{v}^4 \ln^2(m \epsilon^2/q) \le \frac{m \epsilon^2}{q}$ (we can get this condition by saying that $m \ge C_{M,2} \epsilon^{-2} q$ for a sufficiently large $C_{M,2}$) and $C_v^4 \ln^2(m\epsilon^2/q) \ge 1$ (we can get this condition by saying that $m \ge C_{M,2} \epsilon^{-2} q$ for a sufficiently large $C_{M,2}$). 
Let's observe that 
\[4096 C_{v}^2 \epsilon \ge \frac{4096 qv_2^2}{s \ln^2(qmv_2^4/s^2)} \ge 2048 C_v^2 \epsilon \ge \frac{2048 \epsilon}{C}. \]
 
Let $m' = s \cdot e^{\frac{C\epsilon^{-1}q}{1024s}}$. When $m \ge m'$, we know that $\frac{q}{s \ln(m/s)} \le \frac{1024\epsilon}{C}$ and when $m \le m'$, we know that
$\frac{q}{s \ln(m/s)} \ge \frac{1024\epsilon}{C}$. 

In order to plug in $v = v_1$ and use the $\frac{q^2v^2}{s \ln^2(qmv^2/s)}$ lower bound, we need to show that $v \le \frac{\sqrt{\ln(m/s)}}{\sqrt{q}}$. At $v = \frac{\sqrt{\ln(m/s)}}{\sqrt{q}}$, we have that $\frac{qmv^2}{s} = \frac{m}{s} \ln\left(\frac{m}{s}\right)$. Observe that when $m \ge e^2 s$, this is lower bounded by $e^2$, so $\ln(qmv^2/s) \ge 2$. At this value, observe that: 
\[\frac{q^2v^2}{s \ln^2(qmv^2/s)} = \frac{q \ln(m/s)}{s \ln^2\left(\frac{m}{s} \ln\left(\frac{m}{s}\right) \right)} \ge  \frac{q \ln(m/s)}{4 s \ln^2\left(\frac{m}{s} \right)} = \frac{q}{4 s \ln\left(\frac{m}{s} \right)}.\] We can write $\frac{q^2v^2}{s \ln^2(qmv^2/s)} = \frac{q}{m} \frac{w}{\ln^2 w}$, where $w = qmv^2/s$. We observe that this is an increasing function of $w$ as long as $w \ge e^2$. Thus, it suffices to show that $\frac{q^2v_1^2}{s \ln^2(qmv_1^2/s)} \le \frac{q^2v^2}{s \ln^2(qmv^2/s)}$. When $m \le m'$, we know that  
\[\frac{q^2v_1^2}{s \ln^2(qmv_1^2/s)} \le \frac{256 \epsilon}{C} \le \frac{q}{4 s \ln\left(\frac{m}{s} \right)} \le \frac{q^2v^2}{s \ln^2(qmv^2/s)}.\] Thus, we have that $v_1 \le v = \frac{\sqrt{\ln(m/s)}}{\sqrt{q}}$ as desired. 

The first case is $m \le m'$ and $f_2 \le f_1$. We set $v = C_{v} f_2$. 
\[\norm{R(v, \ldots, v,  0, \ldots, 0)}_q \ge D_1 \frac{4096qv^2}{s \ln(qmv^4/s)}.\]
For the upper bound, we see that $\ln(2qmv^4/s^2) > \ln(qmv^4/s^2) \ge 2$ and $\sqrt{\frac{2q}{m}} \le \frac{\epsilon}{C}$. Here, we have that
\[\norm{R(v, \ldots, v,  0, \ldots, 0)}_{2q} \le 
\begin{cases}
D_2 \max\left(\frac{\sqrt{2q}}{\sqrt{m}}, \frac{8192qv^2}{s \ln(2qmv^4/s)}, \frac{4q^2v^2}{s \ln^2(2qmv^2/s)}\right) &\text{ if } \ln(2qmv^2/s) \le 2q \\
D_2 \max\left(\frac{\sqrt{2q}}{\sqrt{m}}, \frac{8192qv^2}{s \ln(qmv^4/s)}\right) &\text{ if } \ln(2qmv^2/s) > 2q \\
\end{cases}.
 \] Now, we use the fact that $v \le C_v  f_1 := v_1$ to see that: 
 \[\frac{4q^2v^2}{s \ln(2qmv^2/s)} \le \frac{4q^2v^2}{s \ln(qmv^2/s)} \le \frac{8q^2v_1^2}{s \ln(qmv_1^2/s)} \le \frac{2048\epsilon}{C}. \] We also observe that since $2qmv^4/s \le (qmv^4/s)^2$, we know:
 \[\frac{8192qv^2}{s \ln(2qmv^4/s)} \ge \frac{8192qv^2}{2 s \ln(qmv^4/s)} \ge \frac{2048 \epsilon}{C}. \] This, coupled with the guarantee on $\frac{\sqrt{2q}}{\sqrt{m}}$, implies we have an upper bound of: 
\[\norm{R(v, \ldots, v,  0, \ldots, 0)}_{2q} \le D_2 \frac{8192qv^2}{s \ln(2qmv^4/s)}.\] Thus, we have that
\[\frac{\norm{R(v, \ldots, v, 0 \ldots, 0)}_q}{\norm{R(v, \ldots, v, 0, \ldots, 0)}_{2q}} \ge \frac{D_1}{2D_2} \ge D. \] Moreover, we have that
\[\norm{R(v, \ldots, v, 0, \ldots, 0)}_q \ge D_1 \cdot \frac{4096qv^2}{s \ln(qmv^4/s)} \ge D_1 \frac{2048 \epsilon}{C} = 2048 \epsilon \]
 
The next case is $f_1 \le f_2$ and $m \le m'$. We set $v = v_1$. Since $f_1 \le f_2$, we know that $\ln(qmv^4/s^2) \le q$. Thus we know:
 \[\norm{R(v, \ldots, v,  0, \ldots, 0)}_q \ge 
\begin{cases}
D_1 \max\left(\frac{4096qv^2}{s \ln(qmv^4/s)}, \frac{q^2v^2}{s \ln^2(qmv^2/s)}\right)  &\text{ if } \ln(qmv^4/s^2) \ge 2, qv^2 \le s \\
D_1 \frac{q^2v^2}{s \ln^2(qmv^2/s)} &\text{ else } \\
\end{cases}. 
 \]  For the upper bound, we know that:
\[\norm{R(v, \ldots, v,  0, \ldots, 0)}_{2q} \le 
\begin{cases}
D_2 \max\left(\frac{\sqrt{2q}}{\sqrt{m}}, \frac{8192qv^2}{s \ln(2qmv^4/s)}, \frac{4q^2v^2}{s \ln^2(2qmv^2/s)}\right) &\text{ if } \ln(2qmv^4/s) > 2  \\
D_2 \max\left(\frac{\sqrt{2q}}{\sqrt{m}}, \frac{4q^2v^2}{s \ln^2(2qmv^2/s)}\right) &\text{ if } \ln(2qmv^4/s) \le 2 \\
\end{cases}.
 \] To make these bounds compatible, we need to handle the case where $\ln(qmv^4/s) \ge 2$, $qv^2 \ge s$ better. Let $v' = C_v  f_2$. Assuming that $\ln(qmv^4/s) \ge 2$, we know that $\frac{8192qv^2}{s \ln(2qmv^4/s)}$ can be upper bounded by:  
  \[\frac{8192qv^2}{s \ln(qmv^4/s)} \le \frac{8192qv'^2}{s \ln(qmv'^4/s)}  = \frac{8192 C_{v}^2 \epsilon \ln(m\epsilon^2/q)}{\ln(m \epsilon^2/q) + \ln(C_{v}^4 \ln^2(m\epsilon^2/q))} \le 8192 C_{v}^2 \epsilon \le \frac{8192 \epsilon}{C}\] as long as $\ln^2(m\epsilon/q) C_v^4 \ge 1$ (which we can make true by appropriately setting the constants on the bound for $m$). Observe also that:
  \[\frac{4q^2v^2}{s \ln^2(2qmv^2/s)} \ge \frac{q^2v^2}{s \ln^2(qmv^2/s)} \ge \frac{4 \epsilon}{C}.\] Thus:
  \[\frac{8192 qv^2}{s \ln(qmv^4/s)} \le  \frac{8192q^2v^2}{s \ln^2(2qmv^2/s)} .\] This, coupled with the guarantee on $\frac{\sqrt{2q}}{\sqrt{m}}$, implies that our upper bound becomes:
  \[\norm{R(v, \ldots, v,  0, \ldots, 0)}_{2q} \le 
\begin{cases}
D_2 \frac{8192q^2v^2}{s \ln^2(2qmv^2/s)} &\text{ if } \ln(2qmv^4/s) \le 2 \text{ or} \ln(qmv^4/s) \ge 2, qv^2 \ge s \\
D_2 \max\left(\frac{8192qv^2}{s \ln(2qmv^4/s)}, \frac{4q^2v^2}{s \ln^2(2qmv^2/s)}\right) &\text{ else }. \\
\end{cases}.
 \] We now show that we can tweak $C_v$ within the factor of $2^{1/4}$ range permitted to show that we can ensure that it is not true that $2 - \ln 2 < \ln(qmv^4/s) \le 2$. Observe that multiplying by a factor of $2^{1/4}$ in this case yields $\ln(2qmv^4/s) > 2$ and dividing by a factor of $2^{1/4}$ yields $\ln(qmv^4/s) \le 2 - \ln 2$. Thus, at least one of the $C_v$ values that yields a power of $2$ for $\frac{1}{v^2}$ will work. Thus, we have that
\[\frac{\norm{R(v, \ldots, v, 0 \ldots, 0)}_q}{\norm{R(v, \ldots, v, 0, \ldots, 0)}_{2q}} \ge \frac{D_1}{8192D_2} = \frac{D}{2048}. \] Moreover, we have that:
\[\norm{R(v, \ldots, v, 0, \ldots, 0)}_q \ge D_1 \cdot \frac{q^2v^2}{s \ln^2(qmv^2/s)} \ge D_1 \frac{4 \epsilon}{C} = 4 \epsilon \]
 
The next case is that $m > m'$. We set $v = C_{v} \sqrt{\epsilon s} \frac{\sqrt{\ln\left(\frac{m\epsilon^2}{q}\right)}}{\sqrt{q}}$. We know: 
\[\norm{R(v, \ldots, v,  0, \ldots, 0)}_q \ge D_1 \frac{4096 qv^2}{s \ln(qmv^4/s)}.
 \] For the upper bound, we see that $\ln(2qmv^4/s^2)> \ln(qmv^4/s^2) > 2$. We know: 
\[\norm{R(v, \ldots, v,  0, \ldots, 0)}_{2q} \le 
\begin{cases}
D_2 \max\left(\frac{\sqrt{2q}}{\sqrt{m}}, \frac{8192qv^2}{s \ln(2qmv^4/s)}, \frac{2q}{s \ln(m/s)} \right) &\text{ if } \ln(2qmv^2/s) \le 2q \\
D_2 \max\left(\frac{\sqrt{2q}}{\sqrt{m}}, \frac{8192qv^2}{s \ln(2qmv^4/s)}\right) &\text{ if } \ln(2qmv^2/s) > 2q \\
\end{cases}.
 \] This can be relaxed to: 
\[\norm{R(v, \ldots, v,  0, \ldots, 0)}_{2q} \le D_2 \max\left(\frac{\sqrt{2q}}{\sqrt{m}}, \frac{8192qv^2}{s \ln(2qmv^4/s)}, \frac{2q}{s \ln(m/s)} \right).\] Now, we know that
\[\frac{2q}{s \ln(m/s)} \le \frac{2048 \epsilon}{C} \le \frac{4096qv^2}{s \ln(qmv^4/s)} = \frac{8192qv^2}{2s \ln(qmv^4/s)} \le \frac{8192qv^2}{s \ln(2qmv^4/s)}.  \] This coupled with what we know about $\frac{\sqrt{2q}}{\sqrt{m}}$ means that:
\[\norm{R(v, \ldots, v,  0, \ldots, 0)}_{2q} \le D_2 \frac{8192qv^2}{s \ln(2qmv^4/s)}.\]  Thus, we have that
\[\frac{\norm{R(v, \ldots, v, 0 \ldots, 0)}_q}{\norm{R(v, \ldots, v, 0, \ldots, 0)}_{2q}} \ge \frac{D_1}{2D_2} \ge D. \]
Moreover, we have that
\[\norm{R(v, \ldots, v, 0, \ldots, 0)}_q \ge D_1 \cdot \frac{4096qv^2}{s \ln(qmv^4/s)} \ge D_1 \frac{2048 \epsilon}{C} = 2048 \epsilon. \]

We use the condition on $q$ not being more than a constant factor away from $p = \ln(1/\delta)$, to conclude that $\epsilon^{-2} q = \Theta(\epsilon^{-2} p)$, $f_2 = \Theta\left(\sqrt{\epsilon s} \frac{\sqrt{\ln\left(\frac{m\epsilon^2}{p}\right)}}{\sqrt{p}} \right)$, and $f_1 = \Theta\left(\sqrt{\epsilon s} \frac{ \ln\left(\frac{m\epsilon}{p}\right)}{p}\right)$, and to conclude that the boundaries move within the $\Theta$ notation as well. 
\end{proof}

\section{Additional Experimental Results and Discussion}\label{sec:additionalexperiments}
All of the experiments (in Section 4 and in this section) were run on the default hardware on a Google Colab notebook. The code is available at \url{https://github.com/mjagadeesan/sparsejl-featurehashing}. 

First, we give the results of additional experimental results on real-world and synthetic datasets, using the same experimental setup as Section 4. 
\begin{figure}[!htb]
\centering
\begin{minipage}[b]{.5\textwidth}
  \centering
    \includegraphics[scale=0.48]{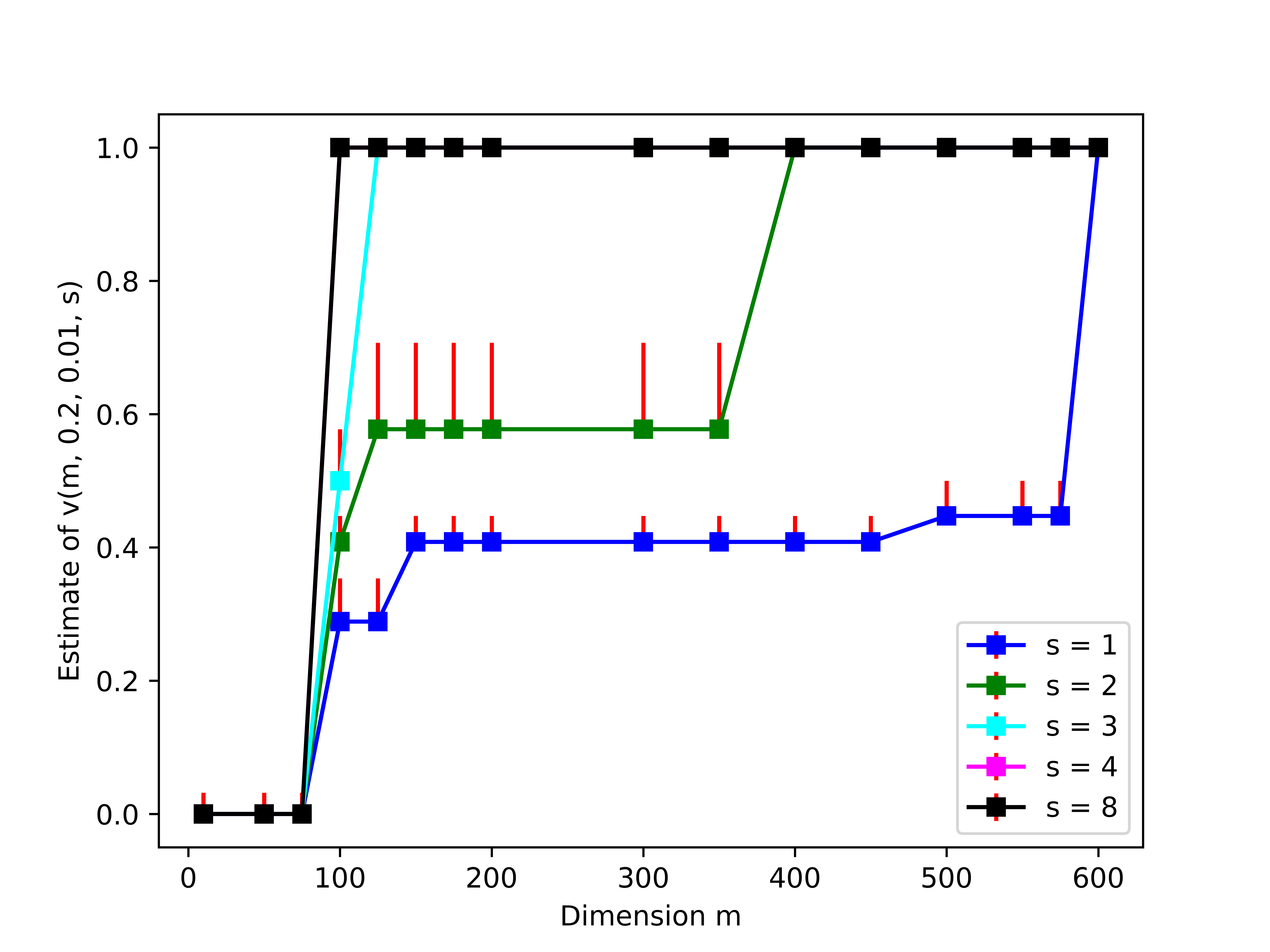}
  \caption{Phase transitions of $\hat{v}(m, 0.2, 0.01,s)$}
  \label{fig:syntheticadd1}
\end{minipage}%
\begin{minipage}[b]{.5\textwidth}
  \centering
    \includegraphics[scale=0.48]{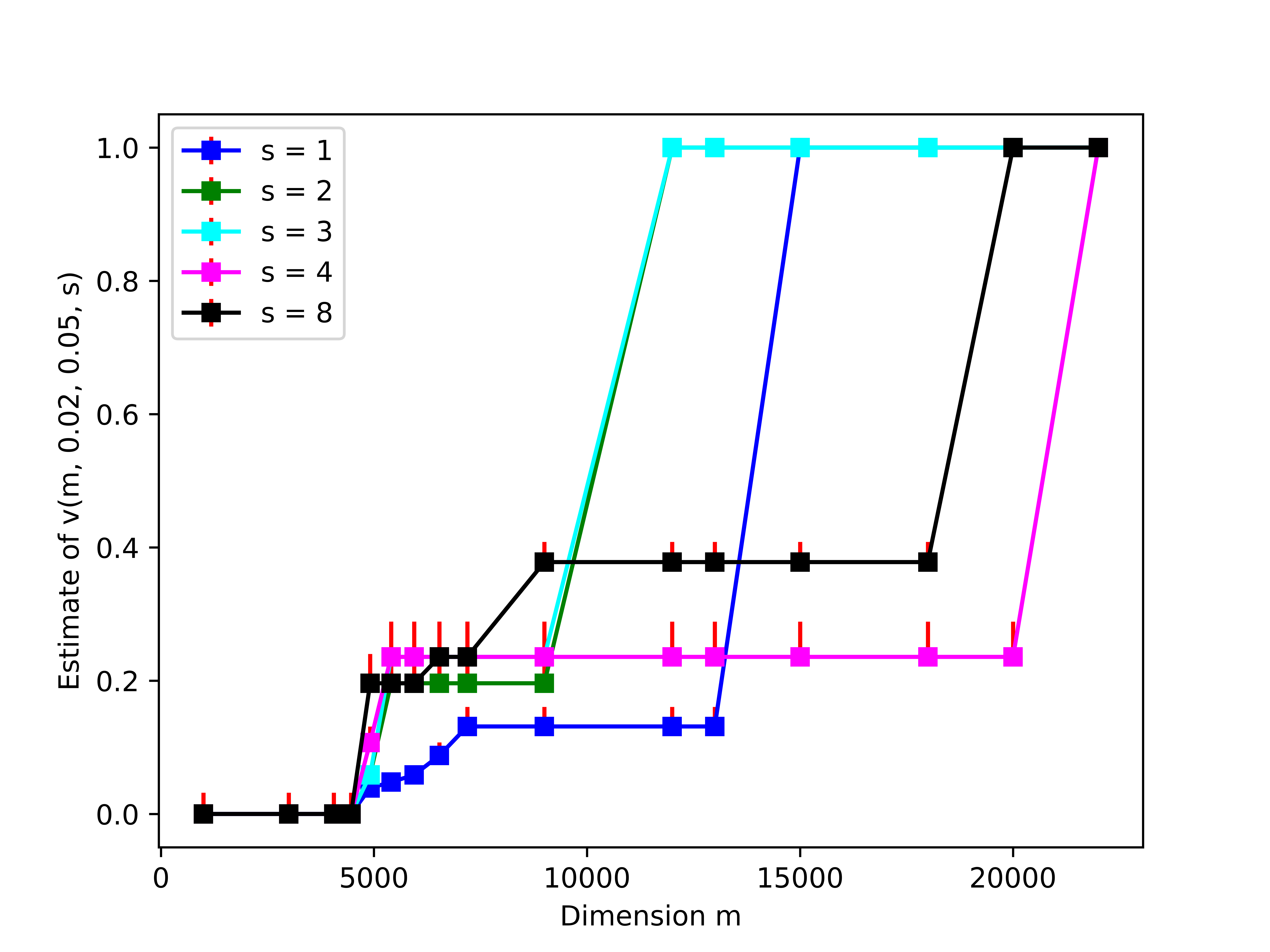}
  \caption{Phase transitions of $\hat{v}(m, 0.02, 0.05,s)$}
  \label{fig:syntheticadd2}
\end{minipage}
\end{figure}

\begin{figure}[!htb]
\centering
\begin{minipage}[b]{.5\textwidth}
  \centering
  \includegraphics[scale=0.48]{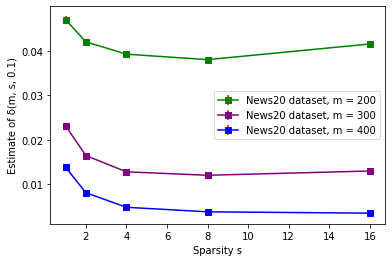}
  \caption{$\hat{\delta}(m, s, 0.1)$ on News20}
  \label{fig:realworldadd1}
\end{minipage}%
\begin{minipage}[b]{.5\textwidth}
  \centering
 \includegraphics[scale=0.48]{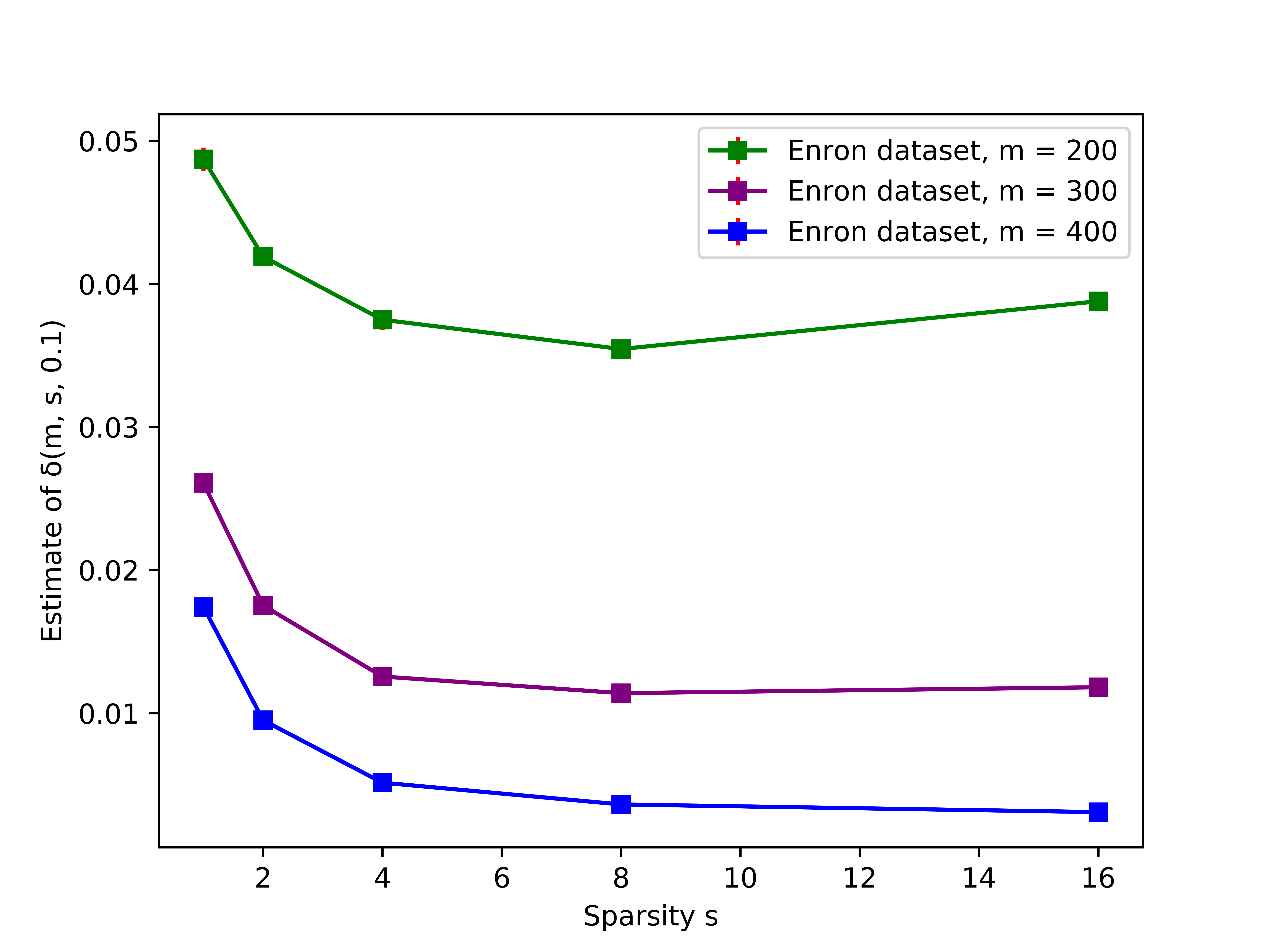}
  \caption{$\hat{\delta}(m, s, 0.1)$ on Enron}
  \label{fig:realworldadd2}
\end{minipage}
\begin{minipage}[b]{.5\linewidth}
 \centering
    \includegraphics[scale=0.5]{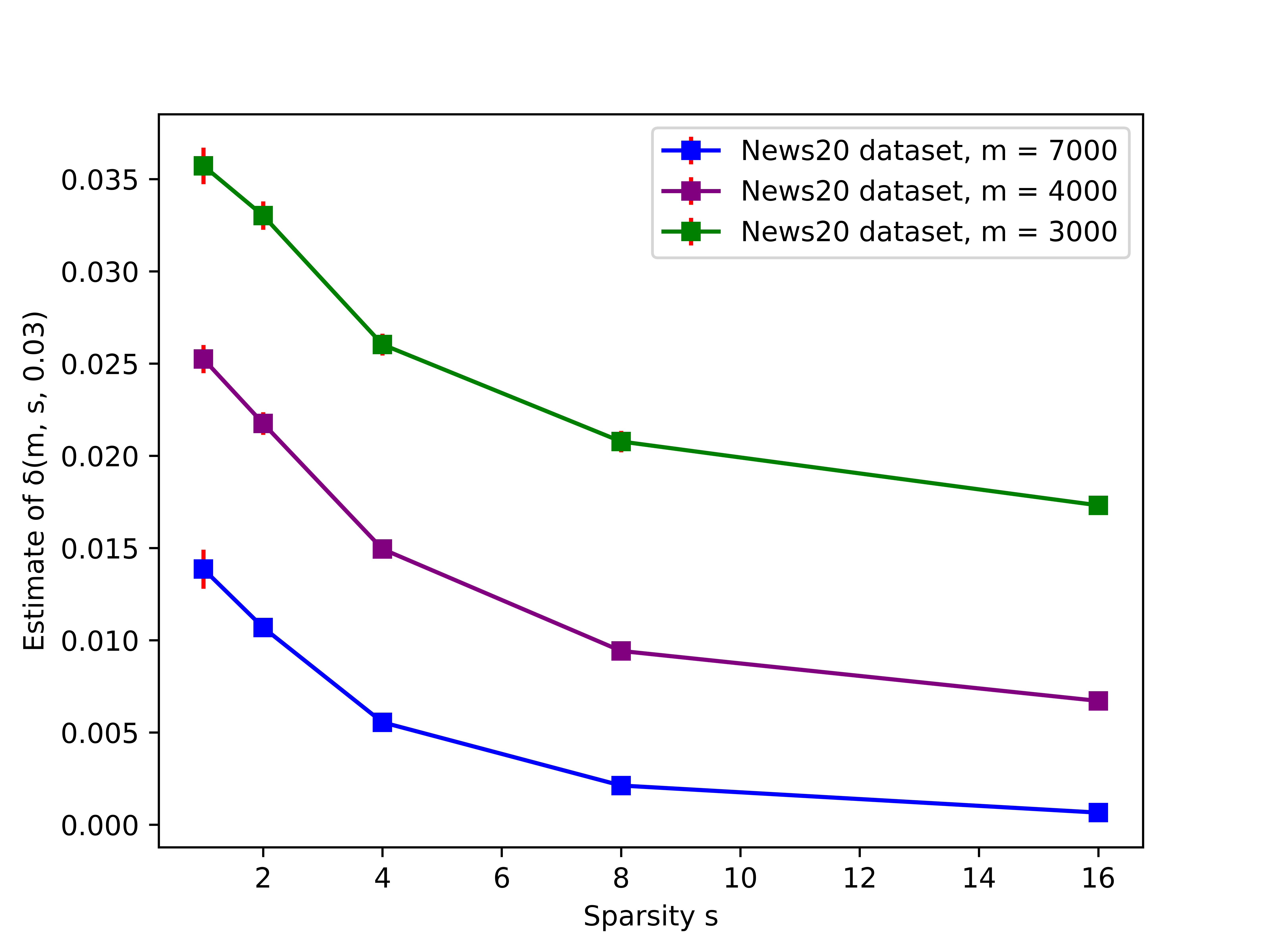}
    \caption{$\hat{\delta}(m, s, 0.03)$ on News20}
    \label{fig:realworldadd3}
\end{minipage}%
\end{figure}

For the synthetic datasets, the trends in Figure \ref{fig:syntheticadd1} and Figure \ref{fig:syntheticadd2} look quite similar to the figures in Section 4. We see, though, that Figure \ref{fig:syntheticadd2} experiences more severe non-monotonic behavior as a function of $s$ in the second phase transition. Consider, for example, in Figure \ref{fig:syntheticadd2}, the behavior at $m = 12000$: we see that $\hat{v}(m, \epsilon, \delta, 4) < \hat{v}(m, \epsilon, \delta, 3)$. In fact, the order of the phase transitions in Figure \ref{fig:syntheticadd2} is far from decreasing. Nonetheless, the general patterns and trends in the theoretical result still hold (e.g. the ``flat'' part occurs at a lower y-coordinate for lower $s$ values.) 

For the real-world datasets, the trends in Figure \ref{fig:realworldadd1}, Figure \ref{fig:realworldadd2}, and Figure \ref{fig:realworldadd3} look quite similar to the figures in Section 4. One slight difference is that the failure probability noticeably increases in Figure \ref{fig:realworldadd1} and Figure \ref{fig:realworldadd2} between $s = 8$ and $s = 16$. It turns out that the failure probability actually increases to a local maximum somewhere in $12 \le s \le 16$, and then decreases when $s \ge 16$, reaching lower than the value at $s = 8$ by the time $s = 20$. There turns out to be a similar local maximum phenomenon when $\epsilon = 0.07$ and $m = 500$, though the local maximum occurs in $24 \le s \le 32$ and thus is not as visible in the graph. 

As a general comment on non-monotonicity as a function of $s$, we emphasize that our asymptotic theoretical results characterize the \textit{macroscopic} behavior of $v(m, \epsilon, \delta, s)$, and do not preclude the existence of constant factor fluctuations for small changes in parameters. An interesting direction for future work would be to look further into this non-mononocity and try to characterize when it arises.

\end{document}